\documentclass[11pt]{article} 
\usepackage{smile}
\usepackage[margin=1in]{geometry}
\usepackage{natbib}

\usepackage[utf8]{inputenc} 
\usepackage[T1]{fontenc}    
\usepackage[colorlinks, linkcolor=blue, anchorcolor=blue, citecolor=blue]{hyperref}       
\usepackage{url}            
\usepackage{booktabs}       
\usepackage{amsfonts}       
\usepackage{nicefrac}       
\usepackage{microtype}      
\usepackage{enumerate}
\usepackage{mathtools}
\usepackage{wrapfig}
\usepackage{amssymb}

\usepackage{kpfonts}
\DeclareMathAlphabet{\mathsf}{OT1}{cmss}{m}{n}

\SetMathAlphabet{\mathsf}{bold}{OT1}{cmss}{bx}{n}

\makeatletter
\newcommand{\specificthanks}[1]{\@fnsymbol{#1}}
\makeatother

\DeclareMathSymbol{\mlq}{\mathrel}{operators}{``}
\DeclareMathSymbol{\mrq}{\mathrel}{operators}{`'}
\DeclareMathSymbol{\mlqq}{\mathrel}{operators}{"5C}
\DeclareMathSymbol{\mrqq}{\mathrel}{operators}{`"}

\newcommand{\ReLU}{\mathrm{ReLU}}
\newcommand{\Softmax}{\mathrm{Softmax}}
\newcommand{\tem}{H}

\def\shownotes{0}  
\ifnum\shownotes=1
\newcommand{\authnote}[2]{$\ll$\textsf{\small #1 notes: #2}$\gg$}
\else
\newcommand{\authnote}[2]{}
\fi
\newcommand{\cms}[1]{{\color{red}\authnote{Minshuo}{#1}}}

\newcommand{\tuo}[1]{{\color{blue}\authnote{Tuo}{#1}}}

\numberwithin{equation}{section}
\begin{document}

\title{Doubly Robust Off-Policy Learning on Low-Dimensional Manifolds by Deep Neural Networks \thanks{Work in progress.}}
\author{Minshuo Chen $\quad$ Hao Liu $\quad$ Wenjing Liao $\quad$ Tuo Zhao \thanks{Minshuo Chen and Hao Liu contribute equally; Minshuo Chen and Tuo Zhao are affiliated with the ISYE department at Georgia Tech; Hao Liu and Wenjing Liao are affiliated with the Math department at Georgia Tech; Email: \text{$\{$mchen393, wliao60, tzhao80$\}$@gatech.edu, hao.liu@math.gatech.edu}.}}
\date{}

\maketitle
\begin{abstract}


Causal inference explores the causation between actions and the consequent rewards on a covariate set. Recently deep learning has achieved a remarkable performance in causal inference, but existing statistical theories cannot well explain such an empirical success, especially when the covariates are high-dimensional. Most theoretical results in causal inference are asymptotic, suffer from the curse of dimensionality, and only work for the finite-action scenario.
To bridge such a gap between theory and practice, this paper studies doubly robust off-policy learning by deep neural networks. When the covariates lie on a low-dimensional manifold, we prove nonasymptotic regret bounds, which converge at a fast rate depending on the intrinsic dimension of the manifold. Our results cover both the finite- and continuous-action scenarios. Our theory shows that deep neural networks are adaptive to the low-dimensional geometric structures of the covariates, and partially explains the success of deep learning for causal inference.

\end{abstract}

\section{Introduction}

Causal inference studies the causal connection between actions and rewards, which has wide applications in healthcare \citep{kim2011battle,lunceford2004stratification}, digital advertising \citep{farias2019learning}, product recommendation \citep{sharma2015estimating}, and policy formulation \citep{heckman2007econometric}. For example in healthcare, each patient can be characterized by a set of covariates (also called features), and the actions are a set of treatments. Each patient has the corresponding reactions, or rewards, to different treatments. Causal inference enables one to personalize the treatment to each patient to maximize the total rewards. Such a personalized decision-making rule is referred to as a policy, which is a map from the covariate set to the action set. In off-policy learning, a batch of observational data is given, which typically consists of a covariate (also called a feature), the action taken (e.g. medical treatments and recommendations), and the observed reward. In this paper, we are interested in learning an optimal policy that targets personalized treatments or services to different individuals based on the logged data. This is also known as the optimal treatment assignment in literature \citep{rubin1974estimating,heckman1977sample}.


Conventional causal inference methods often rely on parametric models \citep{lunceford2004stratification,cao2009improving,robins1994estimation,kitagawa2018should}, which can introduce a large bias when the real model is not in the assumed parametric form. Many nonparametric methods are proposed \citep{hill2011bayesian,kitagawa2018should,zhao2015doubly,kennedy2017non, richardson2014nonparametric,chan2016globally,frolich2017finite,benkeser2020nonparametric,kennedy2020optimal,lee2020doubly, crump2008nonparametric,benkeser2017doubly}, while the statistical theories often suffer from the curse of dimensionality. Recently 
neural networks became a popular modeling tool for causal inference. Many results have shown that neural networks outperform conventional nonparametric approaches, especially when the learning task involves high-dimensional complex data. 
For example, \citet{lopez2017discovering} proposed to discover causal and anticausal features in images from ImageNet using a $20$-layer residual network. \citet{pham2017deep} used recurrent neural networks to study the causality between group forming loans and the funding time on an online non-profit financial platform.
Other examples can be found in diverse areas, including climate analysis, medical diagnosis, cognitive science, and online recommendations \citep{chalupka2014visual, van2019eliminating,johansson2016learning,hartford2017deep,zhang2019deep,lim2018forecasting}. \tuo{Add more example on LSTMs}


Despite the great progress of causal inference, there is still a huge gap between theory and practice. In casual inference, many existing theories on nonparametric or neural networks approaches 
are asymptotic, and suffer from the curse of dimensionality. Specifically, to achieve an $\epsilon$ accuracy, the sample complexity needs to grow in the order of $\epsilon^{-D}$, where $D$ is the covariate dimension. Such theories can not explain the empirical success when $D$ is large.
For example, in \citet{lopez2017discovering}, the RGB images in ImageNet are of resolution $3\times 224 \times 224$. To obtain a $0.1$ error, the sample complexity needs to scale like $10^{-3\times 224 \times 224}$, which well exceeds the training size of $99,309$.
Besides, the curse of dimensionality is inevitable unless additional data structures are considered.
\citet{gao2020minimax} proved that, for binary policy learning problems, the sample complexity obtained by the optimal algorithm still grows exponentially in the covariate dimension $D$ in the order of $\epsilon^{-D}$. \cms{explicit rate}

To bridge this gap, we take the low-dimensional geometric structures of the covariates into consideration. This is motivated by the fact that real-world data often exhibit low-dimensional structures, due to rich local regularities, global symmetries, or repetitive patterns \citep{tenenbaum2000global,roweis2000nonlinear,peyre2009manifold}. For example, many images describe the same object with different transformations, like translation, rotation, projection and skeletonization. These transformations are often represented by a small number of parameters, including the translation position, the rotation, and the projection angle, etc. 
Similar low-dimensional structures exist in medical data \citep{choi2016learning,mahoney2009cur} and financial data \citep{baptista2000low}. 
\cms{give an example} 
To incorporate such low-dimensional structures of data, we assume that the input covariates are concentrated on a $d$-dimensional Riemannian manifold embedded in $\RR^D$ with $d \ll D$.

In many off-policy learning methods, policy evaluation plays an important role by evaluating the expected reward of a given policy. Most of the existing works dedicate to policy evaluation with finite actions. Only few works addressed the continuous action scenario \citep{kallus2018policy,demirer2019semi,kallus2019kernel}. Among the policy evaluation methods with finite actions, the doubly robust method \citep{cassel1976some,robins1994estimation,dudik2011doubly} has the advantage of being consistent, if either the reward function or the propensity score (the probability of choosing a certain action given the covariate) is correctly specified. 
The statistical theory for policy evaluation has been intensively studied in the past with many asymptotic results \citep{swaminathan2015batch,zhou2018offline,zhao2012estimating,kallus2018confounding,kitagawa2018should}.



This paper establishes statistical guarantees of policy learning in causal inference using neural networks. We consider the doubly robust method (see Section \ref{sec:setup} for details), and use deep ReLU neural networks to parameterize the policy class, the propensity score, and the conditional expected reward. We prove nonasymptotic regret bounds which converge at a fast rate depending on the intrinsic dimension $d$, instead of the covariate dimension $D$.
Furthermore, our theory applies to both the finite-action and continuous-action scenarios.

This paper has three main contributions:
{\bf 1)} By taking the low-dimensional geometric structures of the covariates into consideration, we prove a fast convergence rate of the learned policy, depending on the intrinsic dimension of the covariates. {\bf 2)} Our statistical theory is nonasymptotic for policy learning, while most existing works established asymptotic theories for policy evaluation using the doubly robust method. While policy evaluation gives rise to the performance of any specific policy, policy learning furhter returns an optimal policy. 
{\bf 3)} To our best knowledge, we prove the first regret bound of policy learning in a continuous action space.

\paragraph{Related work}
In off-policy learning, one line of research learns the optimal policy by evaluating the expected reward of candidate policies and then finding the policy with the largest expected reward. The procedure of evaluating a target policy from the given data is called off-policy evaluation, which has been intensively studied in literature. The simplest way to evaluate a policy is the direct method which estimates the empirical reward of the target policy from collected data \citep{beygelzimer2009offset}. The direct method is unbiased if one specifies the reward model correctly. However, model specification is a difficult task in practice. Another method is the inverse propensity weighting \citep{horvitz1952generalization,robins1994estimation}, which uses the importance weighting to correct the mismatch between the propensity scores of the target policy and the data collection policy. This method is unbiased if the data collection policy can be exactly estimated, yet it has a large variance especially when some actions are rarely observed. A more robust method is the doubly robust method \citep{cassel1976some,cao2009improving,dudik2011doubly}, which integrates the direct method and the inverse propensity weighting. This method is unbiased if the reward model is correctly specified or the data collection policy is known.

The aforementioned methods have been used in \citet{kitagawa2018should,zhao2015doubly,athey2017efficient,zhou2018offline} for off-policy learning. \citet{kitagawa2018should} used the inverse propensity weighting, and \citet{athey2017efficient} and \citet{zhao2015doubly} used the doubly robust method to learn the optimal policy with binary actions.
In \citet{zhou2018offline}, an algorithm based on decision trees was proposed to learn the optimal policy with multiple actions using the doubly robust method. \citet{kallus2018balanced} proposed a balanced method which minimizes the worst-case conditional mean squared error to evaluate and learn the optimal policy with multiple actions.

Another line of research learns the optimal policy without evaluating policies. In \citet{zhang2012estimating,zhao2012estimating}, the authors transformed the policy learning task with binary actions into a classification problem. Other works on off-policy learning include \citet{kallus2020more,ward2019anesthesiologist}, and \citet{bennett2020efficient}.

Most of the aforementioned works provide asymptotic regret bounds with finite actions, which are valid when the number of samples goes to infinity. A nonasymptotic bound was derived in \citet{kitagawa2018should}, but this work requires that the propensity score is known and the algorithm only works for policy learning with binary actions. Meanwhile, off-policy learning with continuous actions has not been addressed until recently \citep{kallus2018policy,demirer2019semi,kallus2019kernel}. \citet{demirer2019semi} developed a semi-parametric off-policy learning algorithm, which requires the reward function in a specific class. \citet{kallus2018policy} applied a kernel method to extend the inverse propensity weighting and the doubly robust method to the continuous-action setting. These works on continuous actions did not provide a nonasymptotic regret bound with an explicit dependency on the number of samples.

The rest of the paper is organized as follows: Section \ref{sec:preliminary} introduces manifold and neural networks; Section \ref{sec:setup} presents the doubly robust estimation framework; Section \ref{sec:theory} states our regret bounds of the learned policy; Section \ref{sec:proofsketch} gives a proof sketch of our theory; Section \ref{sec:discussion} discusses several related topics.


{\bf Notations}: We use bold lowercase letters to denote vectors, i.e., $\bx\in \RR^D$. We use $x_i$ to denote the $i$-th entry of $\bx$, and define $|\bx| = \sum_{i=1}^D |x_i|$ and $\norm{\bx}_2^2 = \sum_{i=1}^D x_i^2$. For a function $f:\RR^D\rightarrow \RR$ and a multi-index $\bs=[s_1,\dots,s_D]^\top$, $\partial^{\bs} f$ denotes $\partial^{|\bs|} f /\partial x_1^{s_1}\cdots \partial x_D^{s_D}$. Let $\Omega$ be the support of a probability distribution $\PP$. The $L_2$ norm of $f$ with respect to $\PP$ is denoted as $\|f\|^2_{L_2}=\int_{\Omega} f^2(\bx) d \PP(\bx)$. We use $\circ$ to denote function composition. For a set $\cA$, $|\cA|$ denotes its cardinality. For a scalar $a>0$, $\lfloor a \rfloor$ denotes the largest integer which is no larger than $a$, $\lceil a \rceil$ denotes the smallest integer which is no smaller than $a$. For $a,b\in\RR$, we denote $a\vee b=\max\{a,b\}$ and $a\wedge b=\min\{a,b\}$. We refer to a one-hot vector as a canonical basis, i.e. $\bv_j = [0, \dots, 0, 1, 0, \dots, 0]^\top \in \RR^d$ with $j$-th element being $1$. We use $\coloneqq$ to define important quantities.

\section{Preliminaries on Manifold and Neural Networks}\label{sec:preliminary}

We briefly review smooth manifolds (see \citet{lee2003introduction} and \citet{tu2010introduction} for more details), H\"{o}lder space on a smooth manifold, and define the neural network class considered throughout this paper.


\subsection{Low-Dimensional Manifold}
Let $\cM$ be a $d$-dimensional Riemannian manifold isometrically embedded in $\RR^D$. A chart for $\cM$ is a pair $(U,\phi)$ such that $U\subset \cM$ is open and $\phi: U \rightarrow \RR^d$ is a homeomorphism, i.e., $\phi$ is a bijection, its inverse and itself are continuous. Two charts $(U,\phi)$ and $(V,\psi)$ are called $C^k$ compatible if and only if the transition functions
$$
\phi\circ\psi^{-1}: \psi(U\cap V)\rightarrow \phi(U\cap V) \quad \mbox{ and } \quad \psi\circ\phi^{-1}: \phi(U\cap V)\rightarrow \psi(U\cap V)
$$
are both $C^k$ functions. A $C^k$ atlas of $\cM$ is a collection of $C^k$ compatible charts $\{(U_i,\phi_i)\}$ such that $\bigcup_i U_i=\cM$. An atlas of $\cM$ contains an open cover of $\cM$ and the mappings from each open cover to $\RR^d$.

\begin{definition}[Smooth manifold]
  A manifold $\cM$ is smooth if it has a $C^{\infty}$ atlas.
\end{definition}
Through the concept of atlas, we are able to define $C^s$ functions and H\"{o}lder space on a smooth manifold.
\begin{definition}[$C^s$ functions on $\cM$]
Let $\cM$ be a smooth manifold and fix a $C^{\infty}$ atlas of it. For a function $f:\cM\rightarrow \RR$, we say $f$ is a $C^s$ function on $\cM$  if for any $(U,\phi)$ in the atlas, $f\circ\phi^{-1}$ is a $C^s$ function in $\RR^d$.
\end{definition}
\begin{definition}[H\"{o}lder space on $\cM$] Let $\cM$ be a compact manifold.
A function $f: \cM \mapsto \RR$ belongs to the H\"{o}lder space $\cH^{\alpha}(\cM)$ with a H\"{o}lder index $\alpha>0$, if for any chart $(U, \phi)$, 
we have
\begin{align*}
\|f\|_{\cH^{\alpha}(U)} &= \max_{|\bs|<\lceil \alpha-1\rceil}\sup\limits_{ \bx \in \phi(U)} |\partial^{\bs}f \circ \phi^{-1}(\bx)| + \max\limits_{|\bs|=\lceil \alpha-1\rceil}
\sup\limits_{ \bx\neq \by \in \phi(U)}\frac{|\partial^{\bs}f\circ \phi^{-1}(\bx)- \partial^{\bs}f \circ \phi^{-1}(\by)|}{\|\bx-\by\|_2^{\alpha - \lceil \alpha-1\rceil}} < \infty.
\end{align*}
For a fixed atlas $\{(U_i, \phi_i)\}$, the H\"{o}lder norm of $f$ is defined as $\norm{f}_{\cH^\alpha(\cM)} = \sup_i \|f\|_{\cH^{\alpha}(U_i)}$. We occasionally omit $\cM$ in the H\"{o}lder norm when it is clear from the context.
\end{definition}

We introduce the reach \citep{federer1959curvature,niyogi2008finding} of a manifold to characterize the local curvature of $\cM$.
\begin{definition}[Reach]
The medial axis of $\cM$ is defined as $$\cT(\cM) = \{\bx \in \RR^D ~|~ \exists~ \bx_1\neq \bx_2\in\cM,\textrm{~such that~}\|\bx-\bx_1\|_2=\|\bx-\bx_2\|_2 = \inf_{\by \in \cM} \norm{\bx - \by}_2\}.$$ The reach $\tau$ of $\cM$ is the minimum distance between $\cM$ and $\cT(\cM)$, i.e.
$$\tau = \inf_{\bx\in\cT(\cM),\by\in\cM} \|\bx-\by\|_2.$$
\end{definition}
Roughly speaking, reach measures how fast a manifold ``bends'' --- a manifold with a large reach ``bends'' relatively slowly.

\subsection{Neural Network}
 We focus on feedforward neural networks with the ReLU activation function: $\ReLU(x)=\max\{0,x\}$. When the argument is a vector or matrix, ReLU is applied entrywise. Given an input $\bx$, an $L$-layer network computes an output as
\begin{align}
f(\bx)=W_L\cdot\ReLU\left( W_{L-1}\cdots \ReLU(W_1\bx+\bb_1)+ \cdots +\bb_{L-1}\right)+\bb_L,\label{eq.ReLU}
\end{align}
where the $W_i$'s are weight matrices and $\bb_i$'s are intercepts. We define a class of neural networks as
\begin{align*}
\cF(L,p,K,\kappa,R) = &\{f~|~ f \mbox{ has the form of (\ref{eq.ReLU}) with } L \mbox{ layers and width bounded by } p, \|f\|_{\infty}\leq R,\\
&\qquad \sum_{i=1}^L \|W_i\|_0+\|\bb_i\|_0\leq K,\|W_i\|_{\infty, \infty}\leq \kappa, \|\bb_i\|_{\infty}\leq \kappa \mbox{ for } i=1,...,L\},
\end{align*}
where $\norm{H}_{\infty, \infty} = \max_{i, j} |H_{ij}|$ for a matrix $H$ and $\|\cdot\|_0$ denotes the number of non-zero elements of its argument.


\section{Off-Policy Learning with Low-Dimensional Covariates}\label{sec:setup}

We introduce a two-stage policy learning scheme using neural networks. Suppose we receive $n$ i.i.d. triples $\{(\bx_i, \ba_i, y_i)\}_{i=1}^{n}$, where $\bx_i \in \cM$ denotes a covariate independently sampled from an unknown distribution on $\cM$, $\ba_i \in \cA$ denotes the action taken, and $y_i \in \RR$ is the observed reward. To incorporate the low-dimensional geometric structures of the covariates, we assume $\cM$ is a $d$-dimensional Riemannian manifold isometrically embedded in $\RR^D$. The action space $\cA$ can be either finite or continuous. For each covariate and action pair $(\bx, \ba)$, there is an associated random reward. We adopt the unconfoundedness assumption to simplify the model, which is commonly used in existing literature on causal inference \citep{wasserman2013all, zhou2018offline}.
\begin{assumption}[Unconfoundedness]\label{assum.unconf}
The reward is independent of $\ba$ conditioned on $\bx$.
\end{assumption}
To better interpret Assumption \ref{assum.unconf}, we first consider a finite action space $\cA = \{A_1, \dots, A_{|\cA|}\}$, where $A_j$ is a one-hot vector, i.e. $A_j = [0, \dots, 0, 1, 0, \dots, 0]^\top$ with $1$ appearing at the $j$-th position. Given the covariate $\bx$, there is a reward $\{Y_1(\bx), \dots, Y_{|\cA|}(\bx)\}$ for each action, where the randomness of $Y_j(\bx)$ only depends on $\bx$. The observed reward $y_i$ is a realization of $Y_j(\bx_i)$ with $\ba_i = A_j$.

\subsection{Policy Learning with Finite Actions}
When the action space is finite, a policy $\pi: \cM \rightarrow \Delta^{|\cA|}$ maps a covariate on $\cM$ to a vector on the $|\cA|$-dimensional simplex
$$\Delta^{|\cA|} = \left\{\bz \in \RR^{|\cA|}: z_i\ge 0 \text{ and } \sum_i z_i = 1\right\}.$$
The $j$-th entry of $\pi(\bx)$ denotes the probability of choosing the action $A_j$ given $\bx$. {A policy in the interior of the simplex is called a randomized policy. If $\pi(\bx)$ is a one-hot vector, 
it is called a deterministic policy.} 
The expected reward of deploying a policy $\pi$ is
\begin{align}\label{eq.Q}
Q(\pi)& = \EE[Y(\pi(\bx))] = \EE\left[\left\langle [Y_1(\bx), \dots, Y_{|\cA|}(\bx)]^\top, \pi(\bx)\right\rangle\right].
\end{align}

We investigate the doubly robust approach \citep{cassel1976some,robins1994estimation,dudik2011doubly} for policy learning, which consists of two stages.
After receiving the training data, we split them into two groups \begin{equation}
\cS_1=\{(\bx_i,\ba_i,y_i)\}_{i=1}^{n_1}\quad\textrm{and}\quad\cS_2=\{(\bx_i,\ba_i,y_i)\}_{i=n_1+1}^{n}.
\label{eqs1}
\end{equation}
We denote $n_2=n-n_1$ and choose $n_1,n_2$ to be proportional to $n$ such that $n_1/n$ is a constant.
In the first stage, we solve nonparametric regression problems using $\cS_1$ to estimate two important functions --- the propensity score and the conditional expected reward. For any action $A_j$, the propensity score
$$e_{A_j}(\bx) \coloneqq \PP(\ba = A_j ~|~ \bx)$$
 quantifies the probability of choosing $A_j$ given the covariate $\bx$, and the expected reward of choosing $A_j$ is
 $$\mu_{A_j}(\bx) \coloneqq \EE[Y_j(\bx) ~|~ \bx].$$
 Substituting the definition above into (\ref{eq.Q}), we can write
 \begin{align}
   Q(\pi)=\EE\left\langle [\mu_{A_1}(\bx),...,\mu_{A_{|\cA|}}(\bx)]^\top,\pi(\bx)\right\rangle=\int_{\cM} \left\langle[\mu_{A_1}(\bx),...,\mu_{A_{|\cA|}}(\bx)]^\top,\pi(\bx)\right\rangle  d\bx.
   \label{eq.Qmu}
 \end{align}
In the second stage, we learn a policy using $\cS_2$ based on our estimated $e_{A_j}$'s and $\mu_{A_j}$'s, which only requires that either the $e_{A_j}$'s or the $\mu_{A_j}$'s are accurately estimated.


\vspace{0.1in}
\noindent {\bf $\bullet$ Stage 1: Estimating $\mu_{A_j}$ and $e_{A_j}$}.
For each action $A_j$, we use a neural network to estimate the reward function $\mu_{A_j}$ by minimizing the following empirical quadratic loss
\begin{align}\label{eq:hat_mu}
\hat{\mu}_{A_j}(\bx) & = \argmin_{f \in \cF_{\rm NN}}~ \frac{1}{n_{A_j}} \sum_{i=1}^{n_1} (y_i - f(\bx_i))^2 \mathds{1}\{\ba_i = A_j\} \quad \textrm{with}\quad n_{A_j} = \sum_{i=1}^{n_1} \mathds{1}\{\ba_i = A_j\},
\end{align}
where $\cF_{\rm NN}: \cM \rightarrow \RR$ is a properly chosen network class defined in Lemma \ref{lemma:stage1rate}.

An estimator of the propensity score $e_{A_j}$ is obtained by minimizing the multinomial logistic loss. Let $\cG_{\rm NN}: \cM \mapsto \RR^{|\cA|-1}$
be a properly chosen network class defined in Lemma \ref{lemma:stage1rate}. We obtain $\hat{e}_{A_j}(\bx)$ via
\begin{align}
\hat{g}(\bx) & = \argmin_{g \in \cG_{\rm NN}}~ \frac{1}{n_1} \sum_{i=1}^{n_1} - [g(\bx_i)^\top, 1] \ba_i + \log \Big(1 + \sum_{j=1}^{|\cA|-1} \exp([g(\bx_i)]_j)\Big), \label{eq.hatg} \\
\hat{e}_{A_j}(\bx) & = \frac{\exp([\hat{g}(\bx)]_j)}{1 + \sum_{j=1}^{|\cA|-1} \exp([\hat{g}(\bx)]_j)}~ \textrm{for~} j \leq |\cA|-1, ~ \textrm{and} ~~ \hat{e}_{A_{|\cA|}}(\bx) = \frac{1}{1 + \sum_{j=1}^{|\cA|-1} \exp([\hat{g}(\bx)]_j)}. \label{eq:hat_e}
\end{align}
Here $[g]_j$ denotes the $j$-th entry, and $[g^\top, 1] \in \RR^{|\cA|}$ is obtained by augmenting $g$ by $1$.

\vspace{0.1in}
\noindent {\bf $\bullet$ Stage 2: Policy Learning}.
Given $\hat{\mu}_{A_j}$ and $\hat{e}_{A_j}$, we learn an optimal policy by maximizing a doubly robust empirical reward:
\begin{align}
\hat{Q}(\pi) := \frac{1}{n_2} \sum_{i=n_1+1}^{n} \pi(\bx_i)^\top \hat{\Gamma}_i \quad \textrm{with}\quad \hat{\Gamma}_i = \frac{y_i-\hat{\mu}_{\ba_i}(\bx_i)}{\hat{e}_{\ba_i}(\bx_i)}\cdot\ba_i+[\hat{\mu}_{A_1}(\bx_i),\dots,\hat{\mu}_{A_{|\cA|}}(\bx_i)]^{\top} \in \RR^{|\cA|}. \label{eq.hatGamma}
\end{align}
A doubly robust optimal policy is learned by
\begin{align}
\hat{\pi}_{\rm DR} = \argmax_{\pi \in \Pi_{\rm NN}}~ \hat{Q}(\pi), \label{eq:pidr}
\end{align}
where $\Pi_{\rm NN}$ is a properly chosen network class (see Section \ref{sec:theory} for the configurations of $\Pi_{\rm NN}$, e.g., \eqref{eq.PiNN} and \eqref{eq.thm1.para}).
The doubly robust reward $\hat{Q}$ can tolerate a relatively large estimation error in either $\hat{\mu}_{A_j}$ or $\hat{e}_{A_j}$ (see the discussion after Theorem \ref{thm.holder}).


\subsection{Policy Learning with Continuous Actions}\label{sec:cts-action}
Continuous actions, e.g. doses of drugs, often arise in applications, but there are limited studies on policy learning with continuous actions. In this paper,
we consider the continuous action space $\cA = [0, 1]$ and use $a \in \cA$ to denote an action. When the random action $a$ takes the value $A \in [0, 1]$, we denote $Y(\bx, A)$ as its random reward. The propensity score and conditional expected reward are defined analogously to the finite action case:
\begin{align*}
e(\bx, A) \coloneqq \frac{d}{dA} \PP(a \leq A, A \in \cA ~|~ \bx) \quad \textrm{and} \quad \mu(\bx, A) \coloneqq \EE[Y(\bx, A) ~|~ \bx].
\end{align*}
Note that $e(\bx, A)$ is a probability density function.

In this scenario, we can learn an optimal policy by replicating the two-stage scheme with a discretization technique on the continuous action space. Specifically, we uniformly partition the action space $\cA$ into $V$ sub-intervals and denote $I_j = [(j-1)/V, j/V]$ for $j = 1, \dots V$. Accordingly, we define the discretized propensity score and conditional expected reward for the sub-interval $I_j$ as
\begin{align}\label{eq:discretized_score_reward}
e_{I_j}(\bx) \coloneqq \PP(a \in I_j ~|~ \bx) \quad \textrm{and} \quad \mu_{I_j}(\bx) \coloneqq \EE[Y(\bx, a) \mathds{1}\{a \in I_j\} ~|~ \bx]/e_{I_j}.
\end{align}
After the discretization on the action space, we identify all the actions $a$ belonging to a single sub-interval $I_j$ as the midpoint $A_j = (2j-1)/2V$ of $I_j$ and equips $A_j$ with the average expected reward $\mu_{I_j}$. 
After discretization, we resemble the setup in the finite-action scenario, and then apply the aforementioned two-stage doubly robust approach to learn a discretized policy concentrated on the $A_j$'s.
In the first stage, we obtain $\hat{\mu}_{I_j}$ and $\hat{e}_{I_j}$ as estimators of $\mu_{I_j}$ and $e_{I_j}$, respectively. In the second stage, we use neural networks for policy learning by maximizing the discretized doubly robust empirical reward. Specifically, we define $I(a_i) = I_j$ for $a_i \in I_j$ which maps the continuous action to the corresponding discretized sub-interval. For $a_i\in I_j$, we denote $\ba_i \in \{0, 1\}^V$ as the one-hot vector with the $j$-th element being $1$, which encodes the action $a_i$. The discretized doubly robust empirical reward is defined as
\begin{align}\label{eq:doublyrobustreward}
\hat{Q}^{\rm (D)}(\pi) \coloneqq \frac{1}{n_2} \sum_{i=n_1+1}^n \left\langle \hat{\Gamma}_i^{\rm (D)}, \pi(\bx_i) \right\rangle ~~ \textrm{with} ~~ \hat{\Gamma}_i^{\rm (D)} = \frac{y_i-\hat{\mu}_{I(a_i)}(\bx_i)}{\hat{e}_{I(a_i)}(\bx_i)}\cdot \ba_{i}+[\hat{\mu}_{I_1}(\bx_i),\dots,\hat{\mu}_{I_V}(\bx_i)]^{\top},
\end{align}
where the superscript $({\rm D})$ denotes the discretized quantities. We learn an optimal policy by solving the following maximization problem:
\begin{align}\label{eq:picdr}
\hat{\pi}_{\textrm{ C-DR}}  = \argmax_{\pi \in \Pi_{{\rm NN}}}~ \hat{Q}^{\rm (D)}(\pi)
\end{align}
where $\Pi_{\rm NN}$ is a properly chosen neural network. See Section \ref{subseccontinuousaction} for more details of the learning procedure, a proper choice of $V$, and the statistical guarantees of the learned policy.

\vspace{-0.15in}
\section{Main Results}\label{sec:theory}

Our main results are nonasymptotic regret bounds (see Definition \ref{def.regret}) on the policy learned by the two-stage scheme in Section \ref{sec:setup}, when the covariates are concentrated on a low-dimensional manifold.

The regret of a policy $\pi$ against a reference policy $\bar{\pi}$ is defined as the difference between their respective expected rewards. The formal definition is given as follows.
\begin{definition}\label{def.regret}
Let $\bar{\pi}$ be a fixed reference policy. For any policy $\pi$, the regret of $\pi$ against $\bar{\pi}$ is
  \begin{align}
    R(\bar{\pi},\pi) = Q(\bar{\pi})-Q(\pi).\notag
  \end{align}
\end{definition}
Here $Q(\pi)$ is the expected reward either in the finite-action scenario defined in \eqref{eq.Q} or the continuous-action scenario which is defined later in \eqref{eq:continuousQ}. We consider two reference policies: 1) the optimal H\"{o}lder policy that maximizes the expected reward; 2) the unconstrained optimal policy that maximizes the expected reward. We establish high probability bounds on the regret of the learned policy for both discrete actions (Section \ref{subsecfiniteaction}) and continuous actions (Section \ref{subseccontinuousaction}).

\vspace{-0.1in}
\subsection{Policy Learning with Finite Actions}

\label{subsecfiniteaction}

Our theory is based on the following assumptions, including a manifold model for covariates, some standard assumptions on the smoothness of the propensity score and the reward.

\begin{assumption}\label{assum.cM}
$\cM$ is a $d$-dimensional compact smooth manifold isometrically embedded in $\RR^D$. There exists $B>0$ such that $\|\bx\|_\infty \leq B$ whenever $\bx\in \cM$. The reach $\tau$ of $\cM$ satisfies $\tau>0$.
\end{assumption}

\renewcommand{\theassumption}{A.\arabic{assumption}}
\begin{assumption} \label{assum.causal}
The propensity score and random reward satisfy:

\begin{itemize}
\item[(i)] \emph{Overlap:} $e_{A_j}(\bx) \geq \eta$ for $j = 1, \dots, |\cA|$, where $\eta> 0$ is a constant;

\item[(ii)]  \emph{Bounded Reward:} $Y_j(\bx)$ is bounded and has a bounded variance, i.e., $\sup_{\bx \in \cM} |Y_j(\bx)| \leq M_1$ and $\Var[Y_j] \leq \sigma^2$ for any $j = 1, \dots, |\cA|$, where $M_1 > 0$ and $\sigma > 0$ are constants.
\end{itemize}
\end{assumption}
Assumption \ref{assum.causal} is a standard assumption for statistical guarantees of all learning approaches using the inverse propensity score \citep{wasserman2013all, farrell2018deep, zhou2018offline}. Assumption \ref{assum.causal} implies that expected reward $\mu_{A_j}$ is bounded since $|\mu_{A_j}(\bx)|\leq \EE[|Y_j(\bx)| ~|~ \bx]\leq M_1$ for every $\bx \in \cM$.

\begin{assumption}\label{assum.regularity}
Given a H\"{o}lder index $\alpha \ge 1$, we assume
$\mu_{A_j}(\bx) \in \cH^{\alpha}(\cM)$ and $e_{A_j}(\bx) \in \cH^{\alpha}(\cM)$ for $j = 1, \dots, |\cA|$. Moreover, for a fixed $C^{\infty}$ atlas of $\cM$, there exists $M_2>0$ such that $$\max_j \norm{\mu_{A_j}}_{\cH^\alpha} \leq M_2\quad\textrm{and}\quad\max_j \|\log e_{A_j}\|_{\cH^{\alpha}} \leq M_2.$$
\end{assumption}


Thanks to Assumption \ref{assum.causal} {\it (i)}, $e_{A_j} \in \cH^\alpha$  implies $\log e_{A_j} \in \cH^\alpha$ (see Lemma \ref{lem.holder.log} in Appendix \ref{appendix:helper}). Now we are ready to derive the following estimation bounds for $\mu_{A_j}$ and $e_{A_j}$ using nonparametric regression techniques \citep{tsybakov2008introduction}. To simplify the notation, we denote
\begin{align}
M = \max\{1, M_1, 2M_2,-\log \eta\}.
\label{eq.M}
\end{align}


\vspace{-0.1in}
\subsubsection{Estimation Bounds of $\mu_{A_j}(\bx)$ and $e_{A_j}(\bx)$}

By choosing networks
\begin{equation}
\cF_{\rm NN} = \cF(L_1,p_1,K_1,\kappa_1,R_1)\quad\textrm{and}\quad\cG_{\rm NN} = \cF(L_2,p_2,K_2,\kappa_2,R_2)
\label{eqfnngnn}
\end{equation} to estimate $\mu_{A_j}$ and $e_{A_j}$ in \eqref{eq:hat_mu} and \eqref{eq:hat_e}, respectively, we prove the following estimation error bounds for the estimators $\hat{\mu}_{A_j}$ and $\hat{e}_{A_j}$ (Lemma \ref{lemma:stage1rate} is proved in Appendix \ref{appendix.stage1rate}). We use $O(\cdot)$ to hide absolute constants and polynomial factors of $\alpha$, H\"{o}lder norm, $\log D$, $d$, $\tau$, $|\cA|$, and the surface area of $\cM$.
\begin{lemma}\label{lemma:stage1rate}
Suppose Assumptions \ref{assum.unconf} -- \ref{assum.cM} and \ref{assum.causal} -- \ref{assum.regularity} hold. We choose
\begin{equation}
\begin{aligned}
  &L_1=O(\log \eta n_1),\quad p_1=O\big((\eta n_1)^{\frac{d}{2\alpha+d}}\big), \quad K_1=O\big( (\eta n_1)^{\frac{d}{2\alpha+d}}\log \eta n_1\big),\\ &\hspace{1in} \kappa_1=\max\{B,M,\sqrt{d},\tau^2\},\quad R_1=M
\end{aligned}
 \label{eq.mu.parameter}
\end{equation}
for $\cF_{\rm NN}$ and
\begin{align}
L_2=O(\log n_1),\ p_2=O\big(n_1^{\frac{d}{2\alpha+d}}\big),\ K_2=O\big( n_1^{\frac{d}{2\alpha+d}}\log n_1\big), \ \kappa_2=\max\{B,M,\sqrt{d},\tau^2\},\ R_2=M,
\label{eq.e.parameter}
\end{align}
for $\cG_{\rm NN}$ in \eqref{eqfnngnn}.
Then for any $j = 1, \dots, |\cA|$, we have
\begin{align}
\EE_{\cS_1}\left[\norm{\hat{\mu}_{A_j} - \mu_{A_j}}_{L^2}^2\right] & \leq C_1(M^2+\sigma^2) (\eta n_1)^{-\frac{2\alpha}{2\alpha+d}}\log^3 (\eta n_1), \label{eq.mu.bound} \\
\EE_{\cS_1}\left[\norm{\hat{e}_{A_j} - e_{A_j}}_{L^2}^2\right] & \leq C_2M^2|\cA|^{\frac{4\alpha}{2\alpha+d}}n_1^{-\frac{2\alpha}{2\alpha+d}}\log^3n_1,\label{eq.e.bound}
\end{align}
where $C_1, C_2$ depend on $\log D$, B, $\tau$ and the surface area of $\cM$.
\end{lemma}
In (\ref{eq.mu.bound}) and (\ref{eq.e.bound}), the expectation is taken with respect to $\cS_1$ defined in \eqref{eqs1}. Lemma \ref{lemma:stage1rate} provides performance guarantees of neural networks to solve regression problems (\ref{eq:hat_mu}) and (\ref{eq:hat_e}) in order to estimate $\mu_{A_j}$ and $e_{A_j}$. When the covariates $\bx$ are on a manifold, we prove that the estimation errors converge at a fast rate in which the exponent only depends on the intrinsic dimension $d$ instead of the ambient dimension $D$. \cms{we can comment more, e.g., the procedure for estimating $\mu$ and $e$}


\vspace{-0.05in}

\subsubsection{Regret Bound of Learned Policy versus Constrained Oracle Policy}
Our first main result is a regret bound of $\hat{\pi}_{\rm DR}$ obtained in \eqref{eq:pidr} against the oracle policy in a H\"older policy class:
$$  \pi_\beta^* = \argmax_{\pi \in \Pi_{\cH^\beta}} \EE[Q(\pi(\bx))],$$
where the H\"older policy class $\Pi_{\cH^\beta}$ is defined as
\begin{align}
\Pi_{\cH^\beta} \coloneqq \left\{\textrm{Softmax}[\nu_1(\bx), \dots, \nu_{|\cA|}(\bx)]^\top : \nu_j \in \cH^\beta(\cM) \textrm{~ and $\norm{\nu_j}_{\cH^{\beta}} \leq M$ for~} j = 1, \dots, |\cA| \right\}.
\label{eq.PiHolder}
\end{align}
Accordingly, we pick the neural network policy class as
\begin{align}
\Pi_{\rm NN}^{|\cA|} = \{\textrm{Softmax}(f) \text{ with } f: \cM \rightarrow \RR^{|\cA|} \in \cF(L_{\Pi}, p_{\Pi}, K_{\Pi}, \kappa_{\Pi}, R_{\Pi}) \}.
\label{eq.PiNN}
\end{align}

Our first theorem shows that $\hat{\pi}_{\rm DR}$ is a consistent estimator of the oracle H\"{o}lder policy $\pi^*_\beta$ as long as the network parameters $L_{\Pi}, p_{\Pi}, K_{\Pi}, \kappa_{\Pi}, R_{\Pi}$ are properly chosen.

\begin{theorem}\label{thm.holder}
Suppose Assumptions \ref{assum.unconf} -- \ref{assum.cM} and \ref{assum.causal} -- \ref{assum.regularity} hold. Under the setup in Lemma \ref{lemma:stage1rate}, if the network parameters of $\Pi_{\rm NN}^{|\cA|}$ are chosen with
\begin{equation}
\begin{aligned}
& L_{\Pi}=O(\log n),\quad p_{\Pi}=O\big(|\cA|n^{\frac{d}{2\beta+d}}\big),\quad K_{\Pi}=O\big(|\cA|n^{\frac{d}{2\beta+d}}\log n\big),\\
& \hspace{1in} \kappa_{\Pi}=\max\{B,M,\sqrt{d},\tau^2\},\quad R_{\Pi}=M,
\end{aligned}
\label{eq.thm1.para}
\end{equation}
then with probability no less than $1-C_1|\cA|n^{-\frac{\beta}{2\beta+d}}$ over the randomness of data $\cS_1$ and $\cS_2$, the following bound holds
\begin{align}
  &R(\pi_\beta^*, \hat{\pi}_{\rm DR})\leq C|\cA|^{2} n^{-\frac{\beta}{2\beta+d}}\log^{2}n \nonumber \\
  &\qquad + \eta^{-1}|\cA|\sqrt{\frac{1}{n_2} \sum_{i=n_1+1}^n\left(  \hat{\mu}_{A_j}(\bx_i) - \mu_{A_j}(\bx_i)\right)^2} \sqrt{\frac{1}{n_2} \sum_{i=n_1+1}^n\left(\hat{e}_{A_j}(\bx_i)- e_{A_j}(\bx_i)\right)^2},
  \label{eq.thm1.holderbound}
\end{align}
where $C_1>0$ is an absolute constant and $C$ depends on $\log D$, $d$, $B$, $M$, $\tau$, $\eta$, $\beta$, and the surface area of $\cM$.
\end{theorem}
Theorem \ref{thm.holder} is proved in Section \ref{sec.proof.thm1}.
Theorem \ref{thm.holder} corroborates the doubly robust property of $\hat{\pi}_{\rm DR}$. The regret of $\hat{\pi}_{\rm DR}$ is not sensitive to the individual estimation error of either $\hat{\mu}_{A_j}$ or $\hat{e}_{A_j}$, since the bound depends on the product of the estimation errors. Combining Theorem \ref{thm.holder} and Lemma \ref{lemma:stage1rate}
yields the following corollary (see proof in Section \ref{sec.proof.coral1}).
\begin{corollary}\label{coral.holder}
 Suppose Assumptions \ref{assum.unconf} -- \ref{assum.cM} and \ref{assum.causal} -- \ref{assum.regularity} hold. If the network structures are chosen as in Lemma \ref{lemma:stage1rate} and Theorem \ref{thm.holder},
the following regret bound holds with probability no less than $1-C_1n^{-\frac{\alpha\wedge \beta}{2(\alpha\wedge \beta)+d}}\log^3 n$
  \begin{align}
R(\pi_\beta^*, \hat{\pi}_{\rm DR})\leq  C|\cA|^{\frac{8\alpha+2d}{2\alpha+d}}n^{-\frac{\alpha \wedge\beta}{2(\alpha \wedge\beta)+d}}\log^{2}n
\end{align}
where $C_1$ is an absolute constant, and $C$ depends on $\log D$, $d$, $B$, $M$, $\sigma$, $\tau$, $\eta$, $\alpha$, $\beta$, and the surface area of $\cM$.
\end{corollary}

In comparison with existing works, our theory has several advantages:
\begin{itemize}
\item By considering the low-dimensional geometric structures of the covariates, we obtain a fast rate depending on the intrinsic dimension $d$.
Our theory partially justifies the success of off-policy learning by neural networks for high-dimensional data with low-dimensional structures.
\item Our assumptions on the propensity score and expected reward are weak in the sense that the H\"older index $\alpha \ge 1$ can be arbitrary. In \citet{farrell2018deep} and \citet{zhou2018offline}, the H\"{o}lder index $\alpha$ of the propensity score and expected reward needs to satisfy $2\alpha > D$. This condition is hard to satisfy when the covariates are high-dimensional, unless the $\mu_{A_j}$'s and $e_{A_j}$'s are super smooth with bounded high-order derivatives.
\item Our theory is nonasymptotic, 
while most existing works focus on an asymptotic analysis \citep{zhou2018offline,athey2017efficient,farrell2018deep}.
\end{itemize}

\subsubsection{Regret Bound of Learned Policy versus Unconstrained Optimal Policy}

We have shown that neural networks can accurately learn an oracle H\"{o}lder policy in Corollary \ref{coral.holder}. 
\begin{wrapfigure}{r}{0.34\columnwidth}
\vspace{-0.0in}
\centering
\includegraphics[width = 0.28\textwidth]{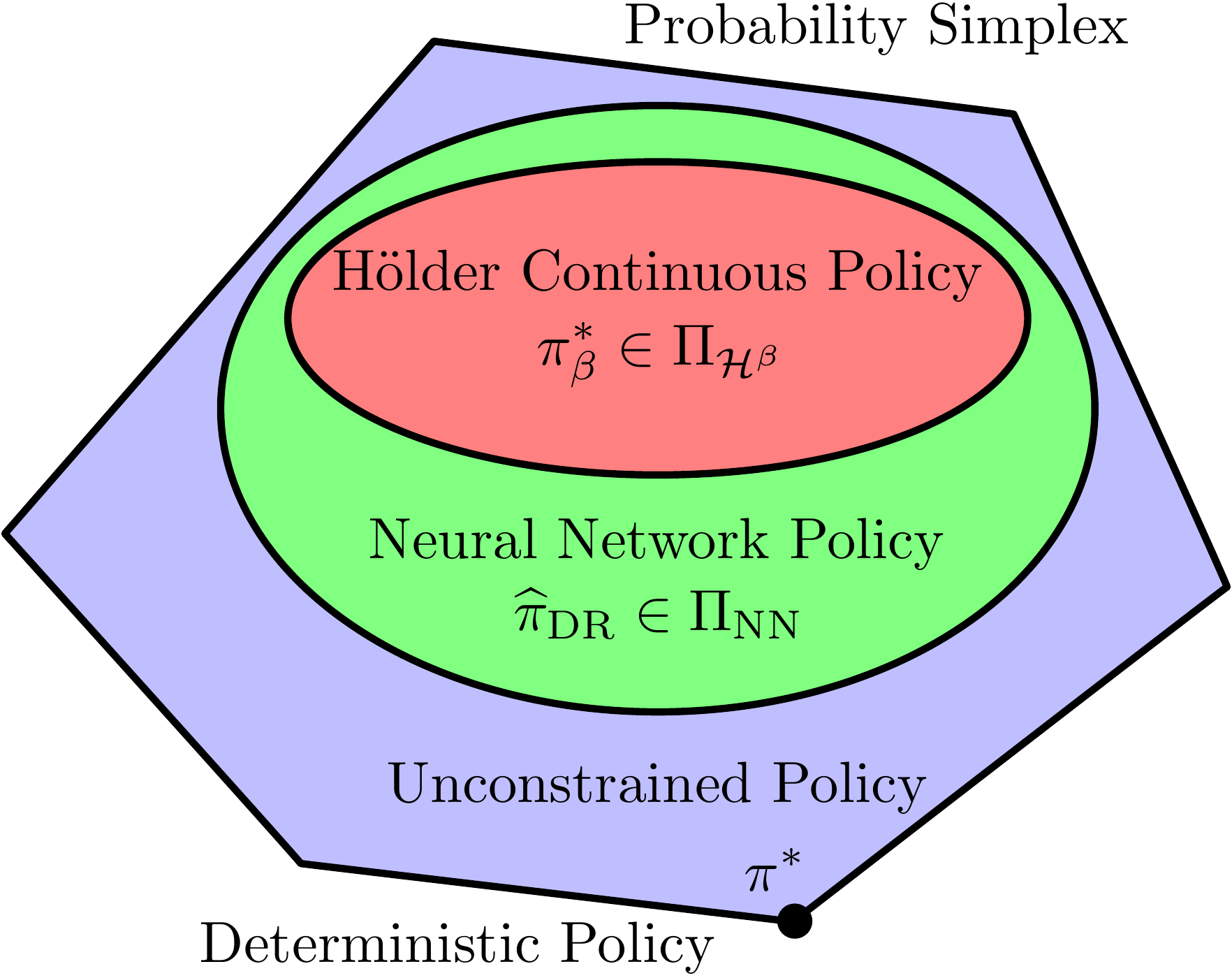}
\vspace{-0.1in}
\caption{The unconstrained policy class is the whole probability simplex with vertices being deterministic polices. The inclusion relation of the neural network policy and the H\"{o}lder policy classes indicates that for any H\"{o}lder continuous policy, there is an approximation given by a neural network policy. }
\label{fig:policy}
\vspace{-0.3in}
\end{wrapfigure}
\noindent In this section, we enlarge the oracle policy class to capture all possible policies, including highly nonsmooth polices, e.g., deterministic policies. We show that neural networks can still achieve a small regret, due to their strong expressive power. The relationship between the H\"{o}lder policy class, neural network policy class, and unconstrained policy class is depicted in Figure \ref{fig:policy}.

The unconstrained optimal policy is defined as
$$  \pi^* = \argmax_{\pi}~ Q(\pi).$$
To establish the regret bound of $\hat{\pi}_{\rm DR}$ in \eqref{eq:pidr} against $\pi^*$, we need the following assumption on the $\mu_{A_j}$'s.


\begin{assumption}[Noise Condition] \label{assum.noise}
Let $q \geq 1$ and denote $j^*(\bx) = \argmax_{j} \mu_{A_j}(\bx)$. There exists $c>0$, such that 
\begin{align*}
  \PP\Big[\big|\mu_{A_{j^*(\bx)}}(\bx) - \max_{j\neq j^*(\bx)} \mu_{A_j}(\bx)\big| \le Mt \Big] \le c t^{q}, \quad \text{ for any } t \in (0, 1).
\end{align*}
\end{assumption}

Assumption \ref{assum.noise} implies that, with high probability, there exists an optimal action whose expected reward is larger than those of others by a positive margin. This is an analogue of Tsybakov
\begin{wrapfigure}{r}{0.34\textwidth}
\vspace{-0.1in}
\centering
\includegraphics[width = 0.3\textwidth]{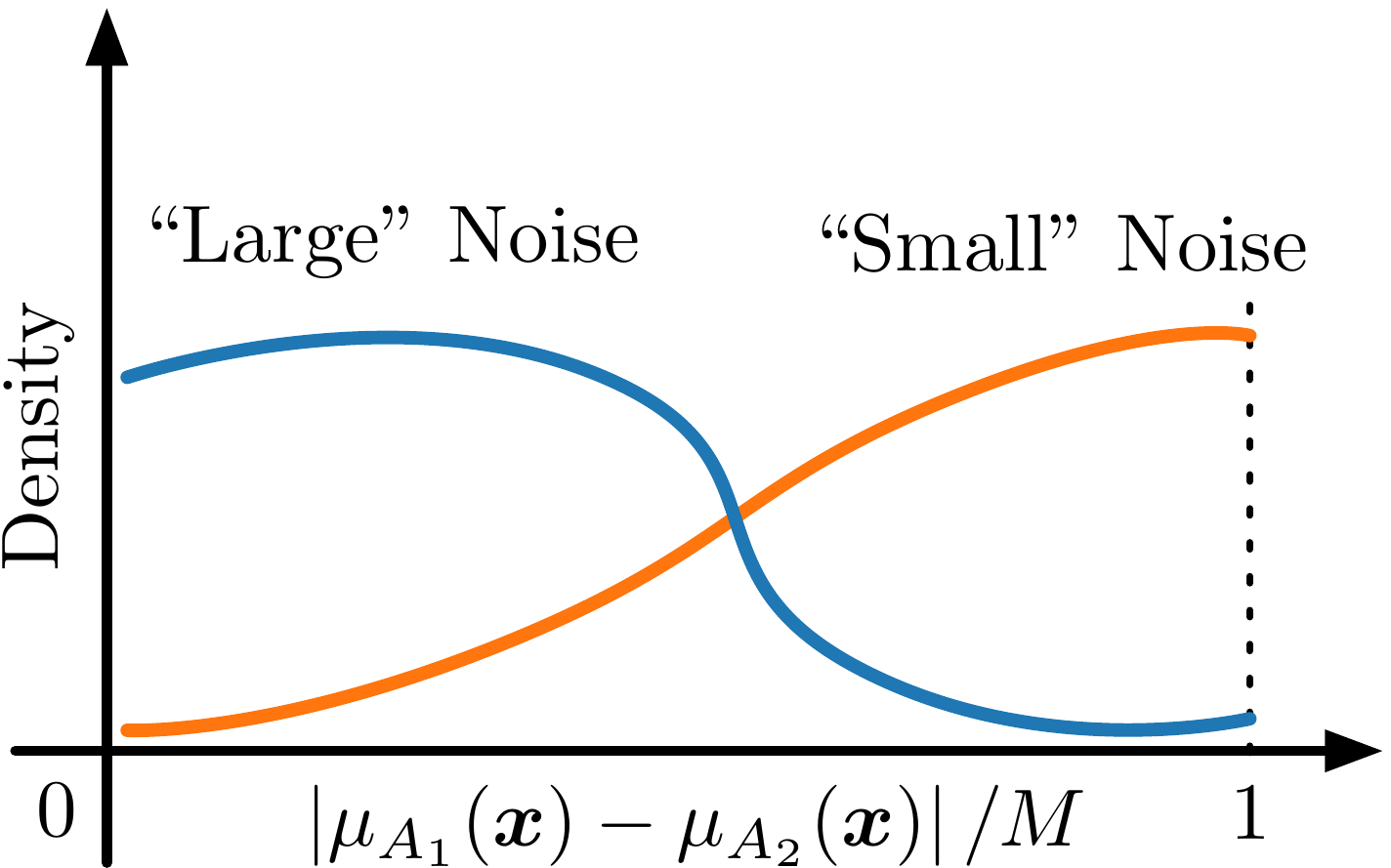}
\vspace{-0.1in}
\caption{Noise condition in a binary-action scenario. ``Large'' noise corresponds to high densities when $\left|\mu_{A_1}(\bx) - \mu_{A_2}(\bx)\right|$ is small; ``Small'' noise corresponds to low densities near the origin.}
\label{fig:noise}
\vspace{-0.1in}
\end{wrapfigure}
\noindent low-noise condition (\citet{tsybakov2004optimal}) in multi-class classification problems, which appears similarly in \citet{wang2016noise}. We illustrate the noise condition in a binary-action scenario in Figure \ref{fig:noise}.

We utilize a temperature parameter $\tem$ in the Softmax layer of the neural network to better learn the unconstrained optimal policy. Under the H\"{o}lder continuity in Assumption \ref{assum.regularity}, there exists a deterministic optimal policy $\pi^*$, i.e., $\pi^*(\bx)=A_{j^*(\bx)}$, which is a one-hot vector. In contrast, the output of the Softmax function is a randomized policy (i.e., a vector in the interior of the simplex), unless the output of the neural network is positive infinity. Accordingly, we adopt the Softmax function with a tunable temperature parameter $\tem$ to push the learned policy to a one-hot vector. This idea has given many empirical successes in reinforcement learning \citep{koulouriotis2008reinforcement,kuleshov2014algorithms}. Specifically, we set
\begin{align}
\Pi_{{\rm NN}(\tem)}^{|\cA|} = \{\textrm{Softmax}_\tem(f) ~\textrm{with } f : \cM \mapsto \RR^{|\cA|} \in \cF(L_\Pi, p_\Pi, K_\Pi, \kappa_\Pi, R_\Pi)\},
\label{eq.PiH}
\end{align}
where $[\textrm{Softmax}_\tem(f)]_i = \frac{\exp(f_i/\tem)}{\sum_j \exp(f_j/ \tem)}$.
A small temperature $\tem$ will push the output of $\Pi^{|\cA|}_{{\rm NN}(\tem)}$ towards a one-hot vector, which can better approximate the deterministic policy $\pi^*$.

Our main result is the following regret bound of $\hat{\pi}_{\rm DR}$ (see proof in Section \ref{sec.proof.thm2}).
\begin{theorem}\label{thm.optimal}
Suppose Assumptions \ref{assum.unconf} -- \ref{assum.cM} and \ref{assum.causal} -- \ref{assum.noise} hold. Assume the network structures defined in Lemma \ref{lemma:stage1rate} are used to estimate the $\mu_{A_j}$'s and the $e_{A_j}$'s. If the network parameters of $\Pi_{{\rm NN}(\tem)}^{|\cA|}$ are chosen with
$$ L_{\Pi}=O\left(\log n\right),\  p_{\Pi}=O\big(|\cA|n^{\frac{d}{2\alpha+d}}\big),\  K_{\Pi}=O\big(|\cA|n^{\frac{d}{2\alpha+d}}\log n\big),\ \kappa_{\Pi}=\max\{B,M,\sqrt{d},\tau^2,1/\tem\},\ R_{\Pi}=M,$$
then the following bound holds with probability no less than $1-C_1n^{-\frac{\alpha}{2\alpha+d}}\log^3 n $
\begin{align}
R(\pi^*,\hat{\pi}_{\rm DR})& \leq \underbrace{C|\cA|^{\frac{8\alpha+2d}{2\alpha+d}} n^{-\frac{\alpha }{2\alpha+d}}\log^{2}n \log^{1/2}\left(1/\tem\right)}_{\cT_1} \nonumber \\
&\hspace{3cm}+  \underbrace{\min_{t\in(0,1)}2cMt^q+M|\cA|^2 \exp\left[\left(-Mt + 2n^{-\frac{\alpha}{2\alpha+d}}\right)/\tem\right]}_{\cT_2},
\label{eq.thm2.bound}
\end{align}
where
$C_1$ is an absolute constant, and $C$ depends on $\log D$, $d$, $B$, $M$, $\sigma$, $\tau$, $\eta$, $\alpha$, and the surface area of $\cM$.
\end{theorem}

The regret $R(\pi^*, \hat{\pi}_{\rm DR})$ consists of two parts: a variance term $\cT_1$ and a bias term $\cT_2$. When the temperature $\tem$ is fixed, the variance $\cT_1$ converges at the rate $n^{-\frac{\alpha}{2\alpha + d}}$, while the bias $\cT_2$ does not vanish. This is because $\hat{\pi}_{\rm DR}$ is a random policy as the output of a softmax function, 
while $\pi^*$ is deterministic as a one-hot vector under Assumption \ref{assum.noise}. Furthermore, $\hat{\pi}_{\rm DR}$ is asymptotically consistent with $\pi^*$ when $\tem \rightarrow 0$. If we choose $\tem = n^{-\frac{2\alpha}{2\alpha + d}}$ and $t = 2n^{-\frac{\alpha}{2\alpha + d}}$, then $\cT_2$ converges at the rate $n^{-\frac{q\alpha}{2\alpha + d}}$ and $\cT_1$ converges at the rate $n^{-\frac{\alpha}{2\alpha + d}}$. We have the following corollary.
\begin{corollary}
\label{corollaryd2}
  Suppose Assumptions \ref{assum.unconf} -- \ref{assum.cM} and \ref{assum.causal} -- \ref{assum.noise} hold. In the setup of Theorem \ref{thm.optimal}, setting $\tem = n^{-\frac{2\alpha}{2\alpha + d}}$ and $t = 2n^{-\frac{\alpha}{2\alpha + d}}$ gives rise to
  \begin{align}
R(\pi^*,\hat{\pi}_{\rm DR}) \leq C|\cA|^{\frac{8\alpha+2d}{2\alpha+d}} n^{-\frac{(1\wedge q)\alpha }{2\alpha+d}}\log^{5/2}n
\label{eq.coro2.bound}
\end{align}
with probability no less than $1-C_1n^{-\frac{\alpha}{2\alpha+d}}\log^3 n $, where
$C_1$ is an absolute constant, and $C$ depends on $\log D$, $d$, $B$, $M$, $\sigma$, $\tau$, $\eta$, $\alpha$, and the surface area of $\cM$.
\end{corollary}

\vspace{-0.1in}
\subsection{Policy Learning with Continuous Actions}
\label{subseccontinuousaction}

Our analysis can be extended to the continuous-action scenario. For simplicity, we let the action space be a unit interval, i.e., $\cA = [0, 1]$. In such a continuous-action scenario, a policy $\pi(\bx, \cdot)$, either randomized or deterministic, is a probability distribution on $[0,1]$ for each covariate $\bx\in\cM$. The expected reward of the policy $\pi$ is defined as
\begin{align}
  Q(\pi) = \int_{\cM}\int_0^1 \mu(\bx,A)\pi(\bx,A)dAd\PP(\bx), \label{eq:continuousQ}
\end{align}
where $\PP$ is the marginal distribution of covariate $\bx$.
%

As mentioned in Section \ref{sec:cts-action}, we tackle the continuous-action scenario using a discretization technique on the action space. This is motivated by practical applications where continuous objects are often quantized.
The action space $\cA$ is uniformly partitioned into $V$ sub-intervals $I_j = [(j-1)/V, j/V]$ for $j = 1, \dots, V$, where $V$ is to be determined in Theorem \ref{thm:continuous}. The discretized version of the propensity score and the expected reward on $I_j$ are defined in \eqref{eq:discretized_score_reward}.

We also consider discretized policies on $\cA$. In particular, we identify all the actions belonging to a single sub-interval $I_j$ as its midpoint $A_j = \frac{2j-1}{2V}$. A discretized policy is defined as
\begin{align}\label{eq:discretizedpolicy}
\pi^{\rm (D)}(\bx, A) = \sum_{j=1}^V p_j(\bx) \delta_{A_j}(A),
\end{align}
where $\delta_{A_j}$ is the Dirac delta function at $A_j$ and $p_j(\bx)$ denotes the probability of choosing action $A_j$, which satisifes $\sum_{j=1}^V p_j(\bx) = 1$. In fact, $\pi^{\rm (D)}(\bx, \cdot)$ can be interpreted as a vector in the $V$-dimensional simplex, since it is only supported on $V$ discretized actions. For simplicity, we denote vector $\pi^{\rm (D)}(\bx) = [p_1(\bx), \dots, p_V(\bx)]^\top$ with $p_j(\bx)$ representing the probability of choosing the action $A_j$, as an equivalent notation of $\pi^{\rm (D)}(\bx, \cdot)$.

For the discretized policy in \eqref{eq:discretizedpolicy}, the discretized expected reward is defined as
\begin{align}\label{eq:discretizedQ}
Q^{\rm (D)}(\pi^{\rm (D)})= \int_{\cM} \inner{[\mu_{I_1}(\bx),\dots,\mu_{I_V}(\bx)]^{\top}}{\pi^{\rm (D)}(\bx)}d\PP(\bx).
\end{align}
We observe the analogy between \eqref{eq:discretizedQ} and \eqref{eq.Qmu} --- the discrete conditional reward $\mu_{A_j}$ is replaced by the discretized conditional reward $\mu_{I_j}$, and the number of discrete actions $|\cA|$  becomes the number of discretized actions $V$. On the other hand, the expected reward of a discretized policy is
\begin{align}
Q(\pi^{\rm (D)}) & = \int_{\cM} \int_0^1 \mu(\bx, A) \sum_{j=1}^V \pi_j(\bx) \delta_{A_j}(A) dA d\PP(\bx) \notag \\
& = \int_{\cM} \inner{[\mu(\bx, A_1), \dots, \mu(\bx, A_V)]^\top}{\pi^{\rm (D)}(\bx)} d\PP(\bx). \label{eq:piQreward}
\end{align}

The following lemma shows that if the $\mu(\bx,A)$ is Lipschitz in $A$ uniformly for any $\bx \in \cM$, $Q^{\rm (D)}(\pi^{\rm (D)})$ is close to $Q(\pi^{\rm (D)})$ when $V$ is large (see proof in Appendix \ref{proof.diserror}).
\begin{lemma}\label{lem.diserror}
  Assume there exists a constant $L_{\mu}>0$ such that
  $$
  \sup_{\bx\in \cM} |\mu(\bx,A)-\mu(\bx,\tilde{A})|\leq L_{\mu}|A-\tilde{A}| \quad \textrm{for any} \quad A,\tilde{A}\in [0,1].
  $$
  Then for any discretized policy $\pi^{\rm (D)}$ given in \eqref{eq:discretizedpolicy}, we have
  \begin{align}
    |Q(\pi^{\rm (D)})-Q^{\rm (D)}(\pi^{\rm (D)})|\leq L_{\mu}/V.\notag
  \end{align}
\end{lemma}
We remark a key difference between \eqref{eq:discretizedQ} and \eqref{eq:piQreward}. To evaluate \eqref{eq:piQreward}, one needs to accurately estimate the $\mu(\bx, A_j)$'s, which requires the action $A_j$ to be repeatedly observed. However, this is prohibitive in the continuous-action scenario, since an action $A_j$ is observed with probability $0$. In contrast, \eqref{eq:discretizedQ} relies on the average expected reward on a sub-interval, which can be estimated using standard nonparametric methods in the following Section \ref{sec:continuousmethod}. Moreover, thanks to Lemma \ref{lem.diserror}, we can well approximate $Q(\pi^{\rm (D)})$ by $Q^{\rm (D)}(\pi^{\rm (D)})$ up to a small discretization error. This is crucial to establish the regret bound in Theorem \ref{thm:continuous}.

\vspace{-0.05in}
\subsubsection{Doubly Robust Policy Learning with Continuous Actions}\label{sec:continuousmethod}
After discretization, we can apply the doubly robust framework to learn an optimal discretized policy.

In the first stage, we estimate the $\mu_{I_j}$'s and $e_{I_j}$'s. In the sequel, we use the plain font $a_i$ to denote the observed action of the $i$-th sample. The bold font $\ba_i \in \{0, 1\}^V$ denotes the one-hot vector with the $j$-th element being $1$, if $a_i \in I_j$. Similar to \eqref{eq:hat_mu} -- \eqref{eq:hat_e} in the finite-action case, we obtain estimators of the $\mu_{I_j}$'s and $e_{I_j}$'s by minimizing the following empirical risks:
\begin{align}\label{eq:hatmu.continuous0}
\hat{\mu}_{I_j}(\bx) & = \argmin_{f \in \cF_{\rm NN}}~ \frac{1}{n_{I_j}} \sum_{i=1}^{n_1} (y_i - f(\bx_i))^2 \mathds{1}\{a_i \in I_j\} \quad \textrm{with}\quad n_{I_j} = \sum_{i=1}^{n_1} \mathds{1}\{a_i \in I_j\},
\end{align}
and
\begin{align}
&\hat{g}(\bx)  = \argmin_{g \in \cG_{\rm NN}}~ \frac{1}{n_1} \sum_{i=1}^{n_1} - \inner{[g(\bx_i)^\top, 1]^\top}{\ba_i} + \log \Big(1 + \sum_{j=1}^{V-1} \exp([g(\bx_i)]_j)\Big), \label{eq.hatg.continuous0} \\
&\hat{e}_{I_j}(\bx)  = \frac{\exp([\hat{g}(\bx)]_j)}{1 + \sum_{j=1}^{V-1} \exp([\hat{g}(\bx)]_j)}~ \textrm{for~} j \leq V-1, ~ \textrm{and} ~~ \hat{e}_{I_{V}}(\bx) = \frac{1}{1 + \sum_{j=1}^{V-1} \exp([\hat{g}(\bx)]_j)}, \label{eq:hate.continuous0}
\end{align}
where $\cF_{\rm NN}:\cM\rightarrow \RR$ and $\cG_{\rm NN}:\cM\rightarrow \RR^{V-1}$ are neural networks.

In the second stage, we learn an optimal discretized policy using the $\hat{\mu}_{I_j}$'s and $\hat{e}_{I_j}$'s. Recall that we define a mapping $I(a_i):=I_j$ for $a_i\in I_j$ to index which sub-interval $a_i$ belongs to. We use neural networks to learn a discretized policy by maximizing the doubly robust empirical reward in \eqref{eq:doublyrobustreward} and \eqref{eq:picdr}, i.e.,
\begin{align}
\hat{\pi}_{\text{ C-DR}}  = \argmax_{\pi \in \Pi_{{\rm NN}(\tem)}^{V}}~ \hat{Q}^{\rm (D)}(\pi),
\label{eq.continu.Q}
\end{align}
where $\Pi^V_{{\rm NN}(\tem)}$ represents the proper network class in \eqref{eq:doublyrobustreward}, which is defined as \eqref{eq.PiH}.
We emphasize that $\hat{\pi}_{\textrm{C-DR}}$ is a discretized policy and the output $\hat{\pi}_{\textrm{C-DR}}(\bx)$ is a $V$-dimensional vector in the simplex.
\vspace{-0.05in}
\subsubsection{Regret Bound of Learned Discretized Policy}
We begin with several assumptions, which are the continuous counterparts of Assumptions \ref{assum.causal} -- \ref{assum.noise} in the finite-action scenario.
\renewcommand{\theassumption}{B.\arabic{assumption}}
\setcounter{assumption}{2}
\begin{assumption}\label{assum.causal.continuous}
The propensity score and random reward satisfy:

\begin{itemize}

\item[(i)]  \emph{Overlap:} $e(\bx, A) \geq \eta $ for any $A \in \cA$, where $\eta> 0$ is a constant.

\item[(ii)] \emph{Bounded Reward:} $|Y(\bx, A)| \leq M_1$ for any $(\bx, A) \in \cM \times [0, 1]$, where $M_1>0$ is a constant.
\end{itemize}
\end{assumption}
\begin{assumption}\label{assum.regularity.continuous}
Given a H\"{o}lder index $\alpha \geq 1$, we have both the expected reward $\mu(\cdot, A) \in \cH^{\alpha}(\cM)$ and the propensity score $e(\cdot, A) \in \cH^{\alpha}(\cM)$ for any fixed action $A$. Moreover, the H\"{o}lder norms of $\mu(\cdot, A)$ and $e(\cdot, A)$ are uniformly bounded for any $A \in [0, 1]$, i.e.,
$$
\sup_{A \in [0, 1]} \norm{\mu(\cdot, A)}_{\cH^\alpha} \leq M_2, \quad \textrm{and} \quad \sup_{A \in [0, 1]} \norm{e(\cdot, A)}_{\cH^\alpha} \leq M_2
$$
 for some constant $M_2 > 0$. Furthermore, there exists a constant $M_3 > 0$ such that 
 $$
 \sup_{\bx \in \cM} |\mu(\bx, A_1) - \mu(\bx, A_2)| \leq M_3 |A_1 - A_2| \quad \textrm{for any} ~ A_1, A_2 \in [0, 1].
 $$There also exists a constant $M_4 > 0$ such that
$$
\norm{\log \left(\int_{I_1} e(\bx, A)dA/ \int_{I_2} e(\bx, A)dA\right)}_{\cH^{\alpha}} \leq M_4 \quad \textrm{for any intervals} ~ I_1,I_2\subset [0,1] \textrm{ of the same length}.
$$
\end{assumption}
Let $p_I(\bx)=\int_{I} e(\bx,A) dA$ denote the probability of choosing actions in $I$ given $\bx$. In Assumption \ref{assum.regularity.continuous}, the condition $e(\cdot,A)\in \cH^{\alpha}(\cM)$ for any given $A$ implies $p_I \in \cH^{\alpha}(\cM)$ (see Lemma \ref{lem.holder.integral}). Combining this and Assumption \ref{assum.causal.continuous} {\it (ii)} of $e(\bx,A)\geq \eta>0$, one deduces that $\log (p_{I_1}(\bx)/p_{I_2}(\bx))$ belongs to $\cH^{\alpha}(\cM)$ with a bounded H\"{o}lder norm. See Lemmas \ref{lem.holder.frac} -- \ref{lem.holder.log} in Appendix \ref{appendix:helper} for a formal justification. For simplicity, we denote
\begin{align}
M=\max\{1, M_1,2M_2,M_3,M_4,-\log \eta\}.
\label{eq.MC}
\end{align}

\begin{assumption}[Continuous Noise Condition]\label{assum.noise.continuous}

The following two conditions hold:

\begin{itemize}
\item[(i)] For each fixed $\bx \in \cM$, $\mu(\bx, A)$ is unimodal with respect to $A$: there exists a unique optimal action $A^*(\bx) \in \cA$ such that $\mu(\bx, A^*(\bx)) = \max_{A \in \cA} \mu(\bx, A)$.

\item[(ii)]There exist constants $q \geq 1$ and $c>0$, such that 
\begin{align*}
\PP\left[\mu(\bx,A^*(\bx)) - \mu(\bx, A) \le Mt ~\textrm{given}~ |A - A^*(\bx)| \ge \gamma \right] \le ct^q(1-\gamma),
\end{align*}
holds for any $t \in (0, 1)$ and any $\gamma \in (0, 1)$,
where $\PP$ denotes the marginal distribution on $\bx$.
\end{itemize}
\end{assumption}


\cms{parse Assumption B.5 in plain language and compare with the finite-action scenario} Assumption \ref{assum.noise.continuous} generalizes the noise condition for finite actions in Assumption \ref{assum.noise}, to the continuous-action scenario. Assmption \ref{assum.noise.continuous} {\it (i)} assures the uniqueness of the optimal action given each covariate. Assumption \ref{assum.noise.continuous} {\it (ii)} means that, with high probability, there is a gap between the reward at the optimal action and the rewards in its neighbors.


We establish a regret bound of $\hat{\pi}_{\text{C-DR}}$ against the unconstrained optimal (deterministic) policy
$$
  \pi^*_{\rm C} = \argmax_{\pi} ~Q(\pi)
$$
where $Q(\pi)$ is defined in (\ref{eq:continuousQ}). Due to Assumption \ref{assum.noise.continuous} {\it (i)}, $\pi^*_{\rm C}$ is deterministic with $\pi^*_{\rm C} (\bx, \cdot) = \delta_{A^*(\bx)}(\cdot)$.
The following theorem establishes the regret bound of $\hat{\pi}_{\textrm{C-DR}}$ against $\pi^*_{\rm C}$.
%
%
\begin{theorem}\label{thm:continuous}
Suppose Assumptions \ref{assum.unconf} -- \ref{assum.cM} and \ref{assum.causal.continuous} -- \ref{assum.noise.continuous} hold. Set $\cF_{\rm NN}=\cF(L_1,p_1,K_1,\kappa_1,R_1)$ in \eqref{eq:hatmu.continuous0} and $\cG_{\rm NN}=\cF(L_2,p_2,K_2,\kappa_2,R_2)$ in \eqref{eq.hatg.continuous0} with
\begin{equation}
  \begin{aligned}
    &L_{1}=O(\log n),\  p_{1}=O\left(\eta^{\frac{d}{2\alpha+d}}n^{\frac{12\alpha d+7d^2}{7(2\alpha+d)^2}}\right),\  K_{1}=O\left(\eta^{\frac{d}{2\alpha+d}}n^{\frac{12\alpha d+7d^2}{7(2\alpha+d)^2}}\log n\right),\\
     &\hspace{1.1in} \kappa_{1}=\max\{B,M,\sqrt{d},\tau^2\},\  R_{1}=M,
  \end{aligned}
  \label{eq.mupara.continuous}
\end{equation}
and
\begin{equation}
\begin{aligned}
    L_{2}=O(\log n),\  p_{2}=O\left(n^{\frac{10\alpha d+7d^2}{7(2\alpha+d)^2}}\right),\  K_{2}=O\left(n^{\frac{10\alpha d+7d^2}{7(2\alpha+d)^2}}\log n\right),\ \kappa_{2}=\max\{B,M,\sqrt{d},\tau^2\},\  R_{2}=M.
\end{aligned}
\label{eq.epara.continuous}
\end{equation}
If the network parameters in $\Pi_{{\rm NN}(\tem)}^{V}$ are chosen as
\begin{align}
 L_{\Pi}=O(\log n),\  p_{\Pi}=O\left(n^{\frac{2\alpha+7d}{7(2\alpha+d)}}\right),\  K_{\Pi}=O\left(n^{\frac{2\alpha+7d}{7(2\alpha+d)}}\log n\right),\ \kappa_{\Pi}=\max\{B,M,\sqrt{d},\tau^2\},\  R_{\Pi}=M,
 \label{eq.thm3.para}
\end{align}
and we set $V = n^{4\alpha/7(2\alpha + d)}$, the following bound holds with probability at least $1-C_1n^{-\frac{2\alpha^2+4\alpha d}{7(2\alpha+d)^2}}\log^3 n$
\begin{align}
  R(\pi^*_{\rm C},\hat{\pi}_{\rm C-DR})&\leq Cn^{-\frac{2\alpha}{7(2\alpha+d)}}\log^{2} n \log^{1/2} 1/\tem \nonumber \\
  &\quad+\left[2cMt^q+Mn^{\frac{4\alpha}{7(2\alpha+d)}} \exp\left(-\left(Mt-4Mn^{-\frac{2\alpha}{7(2\alpha+d)}}\right)/\tem\right)\right]
\end{align}
for any $t\in \left(2(1+1/M)n^{-\frac{2\alpha}{7(2\alpha+d)}},1\right)$, where $C_1$ is an absolute constant and $C$ depends on $\log D$, $d$, $B$, $M,$ $\sigma$, $\tau$, $\eta$, $\alpha,$ and the surface area of $\cM$.
\end{theorem}

\cms{comment on the regret bound} Theorem \ref{thm:continuous} is proved in Section \ref{sec.proof.thm3}. Similar to Equation \eqref{eq.thm2.bound} in Theorem \ref{thm.optimal}, the bound in Theorem \ref{thm:continuous} contains a variance term (the first term) and a bias term (the second term). The variance converges in the rate of $n^{-\frac{2\alpha}{7(2\alpha+d)}}$. For a fixed $\tem$ , the bias term does not vanish as $n$ goes to infinity. If we set $\tem=n^{-\frac{2\alpha}{7(2\alpha+d)}}$ and $t=4(1+M+1/M)n^{-\frac{\alpha}{7(2\alpha+d)}}$, the bias term converges in the rate of $n^{-\frac{2q\alpha}{7(2\alpha+d)}}$. Under this choice, the behavior of $R(\pi^*_{\rm C},\hat{\pi}_{\rm C-DR})$ is summarized in the following corollary.
\begin{corollary}\label{coro.continuous}
  Suppose Assumptions \ref{assum.unconf} -- \ref{assum.cM} and \ref{assum.causal.continuous} -- \ref{assum.noise.continuous} hold. In the setup of Theorem \ref{thm:continuous}, setting $\tem=n^{-\frac{2\alpha}{7(2\alpha+d)}}$ and $t=4(1+M+1/M)n^{-\frac{\alpha}{7(2\alpha+d)}}$ gives rise to
  \begin{align}
  R(\pi^*_{\rm C},\hat{\pi}_{\rm C-DR})&\leq Cn^{-\frac{2\alpha}{7(2(q\wedge1)\alpha+d)}}\log^{5/2} n
\end{align}
with probability no less than $1-C_1n^{-\frac{2\alpha^2+4\alpha d}{7(2\alpha+d)^2}}\log^3 n$, where $C_1$ is an absolute constant and $C$ depends on $\log D\ ,\ d,\ B,\ M,\ \sigma,\ \tau,\ \eta,\ \alpha,$ and the surface area of $\cM$.
\end{corollary}

There are limited theoretical guarantees for causal inference with continuous actions. \citet{kennedy2017non} proposed a doubly robust method to estimate continuous treatment effects. The asymptotic behavior of the method was analyzed while the policy learning problem was not addressed.
To our knowledge, Theorem \ref{thm:continuous} and Corollary \ref{coro.continuous} is the first finite-sample performance guarantee of policy learning with continuous actions.

\section{Proof of Main Results}\label{sec:proofsketch}


We prove our main results in this section, and the lemmas used in this section are proved in the Appendix.
\subsection{Proof of Theorem \ref{thm.holder}}\label{sec.proof.thm1}
\begin{proof}[Proof of Theorem \ref{thm.holder}.]
We denote
\begin{equation}
\hat{\pi}^* = \argmax_{\pi\in\Pi_{\rm NN}^{|\cA|}}~ Q(\pi),
\label{eqhatpistar}
\end{equation} which is the optimal policy given by the neural network class $\Pi_{\rm NN}^{|\cA|}$ defined in (\ref{eq.PiNN}). The regret can be decomposed as
\begin{align}
   R(\pi^*_{\beta}, \hat{\pi}_{\rm DR})=\underbrace{Q(\pi^*_{\beta})-Q(\hat{\pi}^*)}_{\rm (I_1)} + \underbrace{Q(\hat{\pi}^*)-Q(\hat{\pi}_{\rm DR})}_{\rm (II_1)}.
   \label{eq.sketch1.decomposition}
\end{align}
In (\ref{eq.sketch1.decomposition}), ${\rm (I_1)}$ is the approximation error (bias) of the optimal H\"{o}lder policy $\pi^*_{\beta}$ by the neural network class $\Pi_{\rm NN}^{|\cA|}$, and ${\rm (II_1)}$ represents the variance of the estimated policy in $\Pi_{\rm NN}^{|\cA|}$. We next derive the bounds for both terms.

\noindent{\bf $\bullet$ Bounding ${\rm (I_1)}$.}
Recall that $\pi^*_{\beta}$ is the H\"{o}lder continuous optimal policy in $\Pi_{\cH^{\beta}}$.  By defnintion, we can write $\pi^*_{\beta}=\Softmax ([\mu^*_1,\dots,\mu^*_{|\cA|}]^{\top}) \in \RR^{|\cA|}$ where $\mu^*_j\in \cH^{\beta}(\cM)$, for $j=1,\dots,|\cA|$. According to \citet{chen2019efficient}, H\"older functions can be uniformly approximated by a neural network class, if the network parameters are properly chosen. For any $\varepsilon\in(0,1)$ there exists a network architecture $\cF(L,p,K,\kappa,R)$ with
\begin{align}
  L=O\left(\log \frac{1}{\varepsilon}\right), \ p=O\left(\varepsilon^{-\frac{d}{\beta}}\right), \  K=O\left(\varepsilon^{-\frac{d}{\beta}}\log \frac{1}{\varepsilon}\right), \ \kappa=\max\{B,M,\sqrt{d},\tau^2\}, \  R=M,
  \label{eq.sketch1.muPara}
\end{align}
such that for each $\mu^*_j\in \cH^{\beta}(\cM)$, there exists $\tilde{\mu}_j\in \cF(L,p,K,\kappa,R)$ satisfying
\begin{equation}
\|\tilde{\mu}_j-\mu^*_j\|_{\infty}\leq \varepsilon.
\label{eqapprox1}
\end{equation}
The constants hidden in $O(\cdot)$ depend on $\log D$, $d$, $B$, $M$, $\tau$, $\beta$, and the surface area of $\cM$. We denote $\tilde{\pi} = \Softmax([\tilde{\mu}_1,\dots,\tilde{\mu}_{|\cA|}]^\top)\in \RR^{|\cA|}$, which implies $\tilde{\pi}\in \Pi_{\rm NN}^{|\cA|}$ with
\begin{align}
L_{\Pi}=L,\quad p_{\Pi}=|\cA|p,\quad K_{\Pi}=|\cA|K,\quad \kappa_{\Pi}=\kappa,\quad R_{\Pi}=R
\label{eq.sketch1.PiPara}
\end{align}
for $L,p,K,\kappa,R$ defined in (\ref{eq.sketch1.muPara}). Based on \eqref{eqapprox1} and the Lipschitz continuity of the Softmax function, we have
\begin{align*}
  \|\tilde{\pi}-\pi_{\beta}^*\|_{\infty}\leq \varepsilon,
\end{align*}
where $
\|\tilde{\pi}-\pi_{\beta}^*\|_{\infty}=\sup_{\bx\in \cM} \max_j |[\tilde{\pi}(\bx)-\pi_{\beta}^*(\bx)]_j|
$
with $[\tilde{\pi}(\bx)-\pi_{\beta}^*(\bx)]_j$ denoting the $j$-th element of $\tilde{\pi}(\bx)-\pi_{\beta}^*(\bx)$.
Therefore we bound ${\rm I_1}$ as
\begin{align}
{\rm (I_1)}=Q(\pi^*_{\beta})-Q(\hat{\pi}^*)\leq Q(\pi^*_{\beta})-Q(\tilde{\pi})=\EE\left[\inner{[\mu_{A_1}(\bx),\dots,\mu_{A_{|\cA|}}(\bx)]^\top}{\pi^*(\bx)-\tilde{\pi}(\bx)}\right]\leq M|\cA|\varepsilon.
\label{eqboundi1}
\end{align}

\noindent {\bf $\bullet$  Bounding ${\rm (II_1)}$.} \cms{high level picture} We introduce an intermediate reward function $\tilde{Q}$ to decompose the variance term ${\rm (II_1)}$. Define
\begin{align}
\tilde{Q} (\pi) = \frac{1}{n_2} \sum_{i=n_1+1}^n \left\langle \tilde{\Gamma}_i, \pi(\bx_i) \right\rangle \quad \textrm{with} \quad
\tilde{\Gamma}_i = \frac{y_i-\mu_{\ba_i}(\bx_i)}{e_{\ba_i}(\bx_i)}\cdot \ba_i+[\mu_{A_1}(\bx_i),\dots,\mu_{A_{|\cA|}}(\bx_i)]^{\top}\in \RR^{|\cA|}.
\label{eq.sketch1.tGamma}
\end{align}
Note that $\tilde{Q}$ has the same form as $\hat{Q}$ while the estimated propensity score $\hat{e}_{\ba}$ and expected reward $\hat{\mu}_{\ba}$ are replaced by their ground truth $e_{\ba}$ and $\mu_{\ba}$, respectively.

We decompose ${\rm (II_1)}$ as
\begin{align}
  {\rm (II_1)}&=Q(\hat{\pi}^*)-Q(\hat{\pi}_{\rm DR})
  \nonumber\\
  & = \hat{Q}(\hat\pi^*) - \hat{Q}(\hat{\pi}_{\rm DR}) + Q(\hat\pi^*) - \hat{Q}(\hat\pi^*) + \hat{Q}(\hat{\pi}_{\rm DR}) - Q(\hat{\pi}_{\rm DR}) \nonumber\\
  & \leq Q(\hat{\pi}^*) - Q(\hat{\pi}_{\rm DR}) + \hat{Q}(\hat{\pi}_{\rm DR}) - \hat{Q}(\hat{\pi}^*) \nonumber\\
  & \leq \sup_{\pi_1, \pi_2 \in \Pi_{\rm NN}^{|\cA|}} Q(\pi_1) - Q(\pi_2) - \left( \tilde{Q}(\pi_1) - \tilde{Q}(\pi_2)\right)+ \sup_{\pi_1, \pi_2 \in \Pi_{\rm NN}^{|\cA|}} \tilde{Q}(\pi_1) - \tilde{Q}(\pi_2) - \left(\hat{Q}(\pi_1) - \hat{Q}(\pi_2) \right) \nonumber\\
& \leq \underbrace{\sup_{\pi_1, \pi_2 \in \Pi_{\rm NN}^{|\cA|}} \Delta(\pi_1,\pi_2) - \tilde{\Delta}(\pi_1,\pi_2)}_{\cE_1} + \underbrace{\sup_{\pi_1, \pi_2 \in \Pi_{\rm NN}^{|\cA|}} \tilde{\Delta}(\pi_1,\pi_2) - \hat{\Delta}(\pi_1,\pi_2)}_{\cE_2},
\label{eq.sketch1.E.decomposition}
\end{align}
where $\Delta(\pi_1, \pi_2) = Q(\pi_1) - Q(\pi_2),\ \tilde{\Delta}(\pi_1, \pi_2) = \tilde{Q}(\pi_1) - \tilde{Q}(\pi_2)$ and $\hat{\Delta}(\pi_1, \pi_2) = \hat{Q}(\pi_1) - \hat{Q}(\pi_2).$ The first inequality in (\ref{eq.sketch1.E.decomposition}) come from (\ref{eqhatpistar}) which implies $ \hat{Q}(\hat{\pi}_{\rm DR})\leq \hat{Q}(\hat\pi^*)$.
In this decomposition, $\cE_1$ corresponds to the difference between $Q$ and $\tilde{Q}$ which can be bounded using the metric entropy argument, since $\tilde{Q}$ is unbiased, i.e. $\EE [\tilde{Q}(\pi)] = Q(\pi)$. The second term $\cE_2$ corresponds to the error between $\tilde{Q}$ and $\hat{Q}$, which can be bounded in terms of the estimation errors of the $e_{A_j}$'s and the $\mu_{A_j}$'s.

\noindent {\bf Bounding $\cE_1$.} We first show that $\EE\left[\tilde{Q}(\pi)\right]  = Q(\pi)$:
\begin{align}
\EE\left[\tilde{Q}(\pi)\right] & = \EE\left[\left\langle \EE\left[\frac{y - \mu_{\ba}(\bx)}{e_{\ba}(\bx)} \ba\Big| \bx\right], \pi(\bx)\right\rangle + \left\langle [\mu_{A_1}(\bx), \dots, \mu_{A_{|\cA|}}(\bx)]^\top, \pi(\bx) \right\rangle\right]  \nonumber\\
& = \EE\left[\left\langle [\mu_{A_1}(\bx), \dots, \mu_{A_{|\cA|}}(\bx)]^\top, \pi(\bx) \right\rangle \right]  = \EE[Y(\pi(\bx))] = Q(\pi),
\label{eq.ETildeQ}
\end{align}
which further implies $\EE\left[\tilde{\Delta}(\pi_1,\pi_2)\right]=\Delta(\pi_1,\pi_2)$.
In (\ref{eq.ETildeQ}), the second equality holds since
\begin{align*}
  \EE\left[\frac{y - \mu_{\ba}(\bx)}{e_{\ba}(\bx)} \ba\Big| \bx\right]=\frac{\EE[y|\bx] - \mu_{\ba}(\bx)}{e_{\ba}(\bx)} \EE[\ba|\bx]=0
\end{align*}
by Assumption \ref{assum.unconf}. Therefore we can write
\begin{align}
  \cE_1=\sup_{\pi_1, \pi_2 \in \Pi_{\rm NN}^{|\cA|}} \tilde{\Delta}(\pi_1,\pi_2)-\EE[\tilde{\Delta}(\pi_1,\pi_2)]
  \label{eq.E1.D}
\end{align}
with
\begin{align*}
  \tilde{\Delta}=\frac{1}{n_2} \sum_{i=n_1+1}^n \left\langle \tilde{\Gamma}_i, \pi_1(\bx_i) \right\rangle- \frac{1}{n_2} \sum_{i=n_1+1}^n \left\langle \tilde{\Gamma}_i, \pi_2(\bx_i) \right\rangle.
\end{align*}
We derive a bound of $\cE_1$ using the following lemma which is be proved by symmetrization and Dudley's entropy integral \citep{wainwright2019high,dudley1967sizes} in Appendix \ref{proof.E1}:

\begin{lemma}\label{lem.E1}
Let $\Pi: \cM\rightarrow \RR^{|\cA|}$ be a policy space on $|\cA|$ actions such that any $\pi\in\Pi$ maps a covariate $\bx\in\cM$ to $\pi(\bx)$ in the simplex of $\RR^{|\cA|}$, and $\cS_n = \{(\bx_i, y_i)\}_{i=1}^n$ be a set of i.i.d. samples, where $\bx_i$ is sampled from a probability distribution $\mathbb{P}$ supported on $\cM$ and $y_i \in \RR$. For any $(\bx, y)$, we define $\mathring{\Gamma}(\bx, y) \in \RR^{|\cA|}$ as a function of the sample $(\bx, y)$. Assume that there exists a constant $J \ge 0$, such that
\begin{align}
\sup_{(\bx, y) \in \cM \times \RR} |\mathring{\Gamma}(\bx, y)| \le  J.\label{eq:gammabound}
\end{align}
For any policies $\pi_1,\pi_2\in \Pi$, define
\begin{align}
  &\mathring{\Delta}(\pi_1,\pi_2)=\frac{1}{n}\sum_{i=1}^n\inner{\mathring{\Gamma}_i}{\pi_1(\bx_i)}- \frac{1}{n}\sum_{i=1}^n\inner{\mathring{\Gamma}_i}{\pi_2(\bx_i)}\quad \textrm{and} \label{eq.breveDelta}\\
  &\cD(\Pi)=\sup_{\pi_1,\pi_2\in\Pi} \mathring{\Delta}(\pi_1,\pi_2)-\EE[\mathring{\Delta}(\pi_1,\pi_2)]\label{eq.cD}
\end{align}
with the shorthand $\mathring{\Gamma}_i = \mathring{\Gamma}(\bx_i, y_i)$.
Then the following bound holds
  \begin{align}
  \cD(\Pi) \leq \inf_{\lambda}~ 4 \lambda + \frac{96}{\sqrt{n}} \int_{\lambda}^{\max\limits_{\pi \in \Pi} \norm{\pi}_\Gamma} \sqrt{\log \cN(\theta, \Pi, \norm{\cdot}_{\Gamma})} d\theta +12J\sqrt{\frac{\log 1/\delta}{2n}}
  \label{eq.sketch1.E1.lem}
\end{align}
with probability no less than $1-2\delta$ over $\cS_n$, where $\|\pi\|_{\Gamma}=\sqrt{\frac{1}{n}\sum_{i=1}^n \langle \mathring{\Gamma}_i,\pi(\bx_i)\rangle^2}$.
\end{lemma}
A key observation is that when taking $\mathring{\Gamma} = \tilde{\Gamma}$ defined in \eqref{eq.sketch1.tGamma} and $\Pi = \Pi^{|\cA|}_{\rm NN}$, we have $\cE_1 = \cD(\Pi)$ in (\ref{eq.E1.D}). To apply Lemma \ref{lem.E1} for bounding $\cE_1$, we only need to verify the assertion \eqref{eq:gammabound}. In fact, due to Assumption \ref{assum.causal}, we see that $y$, $\mu_{A_j}(\bx)$, and $e_{A_j}(\bx)$ are all bounded. A simple calculation yields $\sup_{(\bx, y) \in \cM \times \RR} |\mathring{\Gamma}(\bx, y)| \leq J \coloneqq 2M/\eta+M$. Therefore, we bound $\cE_1$ as
\begin{align}
  \cE_1\leq \inf_{\lambda}~ 4 \lambda + \frac{96}{\sqrt{n_2}} \int_{\lambda}^{\max\limits_{\pi \in \Pi_{\rm NN}^{|\cA|}} \norm{\pi}_\Gamma} \sqrt{\log \cN(\theta, \Pi_{\rm NN}^{|\cA|}, \norm{\cdot}_{\Gamma})} d\theta +(24M/\eta+12M)\sqrt{\frac{\log 1/\delta}{2n_2}} \label{eq.proof1.boundE1}
\end{align}
with probability no less than $1-2\delta$.

\noindent{\bf Bounding  $\cE_2$.} The $\cE_2$ term depends on the difference between $\tilde{\Gamma}_i$ and $\hat{\Gamma}_i$, where $\tilde{\Gamma}_i$ and $\hat{\Gamma}_i$ are defined in (\ref{eq.sketch1.tGamma}) and (\ref{eq.hatGamma}), respectively. In $\cE_2$, we have
\begin{align*}
  \tilde{\Delta}(\pi_1,\pi_2)-\hat{\Delta}(\pi_1,\pi_2)&= \frac{1}{n_2}\sum_{i=n_1+1}^n \left\langle \pi_1(\bx_i)-\pi_2(\bx_i), \tilde{\Gamma}_i-\hat{\Gamma}_i\right\rangle \nonumber \\
  &= \sum_{j=1}^{|\cA|} \frac{1}{n_2} \sum_{i=n_1+1}^n \left(\pi_{1,j}(\bx_i)-\pi_{2,j}(\bx_i)\right) \left(\tilde{\Gamma}_{i,j}-\hat{\Gamma}_{i,j}\right),
\end{align*}
where $\pi_{k,j}$ and $\tilde{\Gamma}_{i,j}$ denote the $j$-th element of $\pi_k$ and $\tilde{\Gamma}_{i}$, respectively.
Define
\begin{align*}
  \Lambda_j(\pi_1,\pi_2) = \frac{1}{n_2} \sum_{i=n_1+1}^n \left(\pi_{1,j}(\bx_i)-\pi_{2,j}(\bx_i)\right) \left(\tilde{\Gamma}_{i,j}-\hat{\Gamma}_{i,j}\right)\in \RR, \quad \textrm{for}~  j=1,\dots,|\cA|.
\end{align*}
Then we can write $\tilde{\Delta}(\pi_1, \pi_2) - \hat{\Delta}(\pi_1, \pi_2)$ as
\begin{align}
  \tilde{\Delta}(\pi_1,\pi_2)-\hat{\Delta}(\pi_1,\pi_2)=\sum_{j=1}^{|\cA|}  \Lambda_j(\pi_1,\pi_2).
  \label{eq.sumLambda}
\end{align}
The error term $\tilde{\Gamma}_{i,j}-\hat{\Gamma}_{i,j}$ in $\Lambda_j(\pi_1,\pi_2)$ depends on the estimation error of $\hat{\mu}_{A_j}$ and $\hat{e}_{A_j}$. Based on the source of the error, we decompose each $ \Lambda_j(\pi_1,\pi_2)$ into three terms:
\begin{align}
  \Lambda_j(\pi_1,\pi_2)&=\frac{1}{n_2} \sum_{i=n_1+1}^n \left(\pi_{1,j}(\bx_i)-\pi_{2,j}(\bx_i)\right)
  \bigg[\left( \frac{y_i-\mu_{A_j}(\bx_i)}{e_{A_j}(\bx_i)}\ind_{\{\ba_i=A_j\}}+ \mu_{A_j}(\bx_i) \right)\nonumber\\
  &\quad - \left( \frac{y_i-\hat{\mu}_{A_j}(\bx_i)}{\hat{e}_{A_j}(\bx_i)}\ind_{\{\ba_i=A_j\}}+ \hat{\mu}_{A_j}(\bx_i) \right)\bigg]\nonumber \\
  &=\frac{1}{n_2} \sum_{i=n_1+1}^n \left(\pi_{1,j}(\bx_i)-\pi_{2,j}(\bx_i)\right)\left( \mu_{A_j}(\bx_i)- \hat{\mu}_{A_j}(\bx_i)\right) \nonumber \\
  &\quad +\frac{1}{n_2} \sum_{i=n_1+1}^n \left(\pi_{1,j}(\bx_i)-\pi_{2,j}(\bx_i)\right) \ind_{\{\ba_i=A_j\}} \left( \frac{y_i-\mu_{A_j}(\bx_i)}{e_{A_j}(\bx_i)}- \frac{y_i-\hat{\mu}_{A_j}(\bx_i)}{\hat{e}_{A_j}(\bx_i)}\right) \nonumber\\
  &= S_{j}^{(1)}(\pi_{1,j},\pi_{2,j})+S_{j}^{(2)}(\pi_{1,j},\pi_{2,j})+ S_{j}^{(3)}(\pi_{1,j},\pi_{2,j}),\label{eq.S.1}
\end{align}
where
\begin{align*}
  &S_{j}^{(1)}(\pi_{1,j},\pi_{2,j}) = \frac{1}{n_2} \sum_{i=n_1+1}^n \left( \pi_{1,j}(\bx_i)-\pi_{2,j}(\bx_i) \right) \left( \mu_{A_j}(\bx_i) - \hat{\mu}_{A_j}(\bx_i)\right) \left(1- \frac{\ind_{\{\ba_i=A_j\}}}{e_{A_j}(\bx_i)}\right),\\
  &S_{j}^{(2)}(\pi_{1,j},\pi_{2,j}) = \frac{1}{n_2} \sum_{\{n_1+1\leq i\leq n|\ba_i=A_j\}} \left( \pi_{1,j}(\bx_i)-\pi_{2,j}(\bx_i) \right) \left( y_i - \mu_{A_j}(\bx_i)\right) \left(\frac{1}{e_{A_j}(\bx_i)} -\frac{1}{\hat{e}_{A_j}(\bx_i)}  \right),\\
  &S_{j}^{(3)}(\pi_{1,j},\pi_{2,j}) = \frac{1}{n_2} \sum_{\{n_1+1\leq i\leq n|\ba_i=A_j\}} \left( \pi_{1,j}(\bx_i)-\pi_{2,j}(\bx_i) \right) \left( \hat{\mu}_{A_j}(\bx_i) - \mu_{A_j}(\bx_i)\right) \left(\frac{1}{\hat{e}_{A_j}(\bx_i)}- \frac{1}{e_{A_j}(\bx_i)} \right).
\end{align*}
Here $S_{j}^{(1)}(\pi_{1,j},\pi_{2,j})$ and $S_{j}^{(2)}(\pi_{1,j},\pi_{2,j})$ can be bounded using Lemma \ref{lem.E1}. $S_{j}^{(3)}(\pi_{1,j},\pi_{2,j})$ contains the product of the estimation error of $\hat{\mu}_{A_j}$ and $\hat{e}_{A_j}$, which gives the doubly robust property.
According to (\ref{eq.sumLambda}) and (\ref{eq.S.1}),
\begin{align}
  \cE_2&=\sup_{\pi_1,\pi_2\in \Pi_{\rm NN}^{|\cA|}} \left[\sum_{j=1}^{|\cA|} S_{j}^{(1)}(\pi_{1,j},\pi_{2,j})+S_{j}^{(2)}(\pi_{1,j},\pi_{2,j})+ S_{j}^{(3)}(\pi_{1,j},\pi_{2,j})\right]\nonumber\\
  &\leq \sum_{j=1}^{|\cA|} \sup_{\pi_1,\pi_2\in \Pi_{\rm NN}^{|\cA|}} S_{j}^{(1)}(\pi_{1,j},\pi_{2,j})+ \sup_{\pi_1,\pi_2\in \Pi_{\rm NN}^{|\cA|}}S_{j}^{(2)}(\pi_{1,j},\pi_{2,j})+ \sup_{\pi_1,\pi_2\in \Pi_{\rm NN}^{|\cA|}} S_{j}^{(3)}(\pi_{1,j},\pi_{2,j}).
  \label{eq.proof1.S2.1}
\end{align}
In the rest of the proof, when there is no ambiguity, we omit the dependency on $(\pi_{1,j},\pi_{2,j})$ and use the notations $S_j^{(1)}$, $S_j^{(2)}$ and $S_j^{(3)}$. We next derive the bounds for the $S_j^{(1)}$, $S_j^{(2)}$ and $S_j^{(3)}$ terms in the right hand side of \eqref{eq.proof1.S2.1} respectively.

\noindent \textbf{{Bounding $\sup_{\pi_1,\pi_2\in \Pi_{\rm NN}^{|\cA|}}S_j^{(1)}$}}: For $S_j^{(1)}$, one can show that $\EE[S_j^{(1)}]= 0$:
\begin{align*}
  \EE[S_j^{(1)}]&=\EE\left[ \frac{1}{n_2} \sum_{i=n_1+1}^n \EE\left[\left( \pi_{1,j}(\bx_i)-\pi_{2,j}(\bx_i) \right) \left( \mu_{A_j}(\bx_i) - \hat{\mu}_{A_j}(\bx_i)\right) \left(1- \frac{\ind_{\{\ba_i=A_j\}}}{e_{A_j}(\bx_i)}\right)\bigg| \bx_i\right] \right] \nonumber\\
  & =\EE\left[ \frac{1}{n_2} \sum_{i=n_1+1}^n \left( \pi_{1,j}(\bx_i)-\pi_{2,j}(\bx_i) \right) \left( \mu_{A_j}(\bx_i) - \hat{\mu}_{A_j}(\bx_i)\right)\EE \left[1- \frac{\ind_{\{\ba_i=A_j\}}}{e_{A_j}(\bx_i)}\bigg| \bx_i\right] \right]=0.
\end{align*}
Denote $$\bar{\Gamma}^{(1,j)}(\bx_i) = \left(\mu_{A_j}(\bx_i) - \hat{\mu}_{A_j}(\bx_i)\right) \left(1- \frac{\ind_{\{\ba_i=A_j\}}}{e_{A_j}(\bx_i)} \right)\in \RR,$$ then we have
\begin{align}
  &\sup_{\pi_1,\pi_2\in \Pi_{\rm NN}^{|\cA|}} S_j^{(1)}=  \sup_{\pi_1,\pi_2\in \Pi_{\rm NN}^{|\cA|}} S_j^{(1)}- \EE\left[S_j^{(1)}\right] \nonumber\\
=&\sup_{\pi_1,\pi_2\in \Pi_{\rm NN}^{|\cA|}} \frac{1}{n_2} \sum_{i=n_1+1}^n \left( \pi_{1,j}(\bx_i)-\pi_{2,j}(\bx_i) \right) \bar{\Gamma}^{(1,j)}(\bx_i)- \EE \left[\frac{1}{n_2} \sum_{i=n_1+1}^n \left( \pi_{1,j}(\bx_i)-\pi_{2,j}(\bx_i) \right) \bar{\Gamma}^{(1,j)}(\bx_i)\right].
  \label{eq.S1}
\end{align}

The expression in (\ref{eq.S1}) resembles the same form as $\cD$ in (\ref{eq.cD}) with $\mathring{\Gamma}=\bar{\Gamma}^{(1,j)}(\bx)$ and $\Pi = \Pi_{\rm NN}^{|\cA|}$. Therefore, we can estimate $\sup_{\pi_1,\pi_2\in \Pi_{\rm NN}^{|\cA|}}S_j^{(1)}$ using Lemma \ref{lem.E1}.
Due to Assumption \ref{assum.causal}, for any $\bx \in \cM$, we have $|\bar{\Gamma}^{(1,j)}(\bx)|\leq 2M/\eta$. After
substituting $J =2M/\eta$ in Lemma  \ref{lem.E1}, we have
\begin{align}
  \sup_{\pi_1,\pi_2\in \Pi_{\rm NN}^{|\cA|}}S_j^{(1)}\leq \inf_{\lambda}~ 4 \lambda + \frac{96}{\sqrt{n_2}} \int_{\lambda}^{\max\limits_{\pi \in \Pi_{\rm NN}^{|\cA|}} \norm{\pi}_\Gamma} \sqrt{\log \cN(\theta, \Pi, \norm{\cdot}_{\Gamma})} d\theta+ (24M/\eta) \sqrt{\frac{\log 1/\delta}{2n_2}}
  \label{eq.proof1.S1}
\end{align}
with probability no less than $1-2\delta$.

\noindent \textbf{Bounding $\sup_{\pi_1,\pi_2\in \Pi_{\rm NN}^{|\cA|}}S_j^{(2)}$:} Similarly, one can show $\EE \left[S_j^{(2)}\right]=0$.
 Denote
$$
\bar{\Gamma}^{(2,j)}(\bx_i, y_i) = \left( y_i - \mu_{A_j}(\bx_i)\right) \left(\frac{1}{e_{A_j}(\bx_i)} -\frac{1}{\hat{e}_{A_j}(\bx_i)}\right)\ind_{\{\ba_i=A_j\}} \in \RR.
$$
We follow the same calculation in \eqref{eq.S1} to express
$\sup_{\pi_1,\pi_2\in \Pi_{\rm NN}^{|\cA|}}S_j^{(2)}$ in the same form as $\cD$ in (\ref{eq.cD}) with $\mathring{\Gamma}=\bar{\Gamma}^{(2,j)}$ and $\Pi = \Pi_{\rm NN}^{|\cA|}$. 

An upper bound of $\sup_{(\bx, y)} |\bar{\Gamma}^{(2,j)}(\bx)|$ can be derived as follows.
With $\cG_{\rm NN}$ chosen in (\ref{eq.e.parameter}), its output is bounded by $M$, which implies
$\hat{e}_{A_j}\geq (|\cA|e^{2M})^{-1}.$ Thus
$$\left|\frac{1}{e_{A_j}(\bx_i)} -\frac{1}{\hat{e}_{A_j}(\bx_i)}\right|\leq |\cA|e^{2M},$$
since $M\geq -\log \eta$ by (\ref{eq.M}). By Assumption \ref{assum.causal} and (\ref{eq.M}), we have $\sup_{\bx, y}|y - \mu_{A_j}(\bx)| \leq 2M$ hold for any $j = 1, \dots, |\cA|$.
Therefore, we have
\begin{align}
  \sup_{(\bx, y) \in \cM \times \RR} |{\Gamma}^{(2,j)}(\bx)|\leq 2|\cA|e^{2M}M.
  \label{eq.barGamma2}
\end{align}
Using Lemma \ref{lem.E1} and substituting $J = 2|\cA|e^{2M}M$ give rise to
\begin{align}
\sup_{\pi_1,\pi_2\in \Pi_{\rm NN}^{|\cA|}}S_j^{(2)} &\leq \inf_{\lambda}~ 4 \lambda + \frac{96}{\sqrt{n_2}} \int_{\lambda}^{\max\limits_{\pi \in \Pi_{\rm NN}^{|\cA|}} \norm{\pi}_\Gamma} \sqrt{\log \cN(\theta, \Pi, \norm{\cdot}_{\Gamma})} d\theta+ (24|\cA|e^{2M}M )\sqrt{\frac{\log 1/\delta}{2n_2}}
  \label{eq.proof1.S2}
\end{align}
with probability no less than $1-2\delta$.

\noindent \textbf{Bounding $\sup_{\pi_1,\pi_2\in\Pi_{\rm NN}^{|\cA|}} S_j^{(3)}$:} We next derive an upper bound of $\sup_{\pi_1,\pi_2\in\Pi_{\rm NN}^{|\cA|}} S_j^{(3)}$ as the product of the estimation errors of the $\hat{\mu}_{A_j}$'s and the $\hat{e}_{A_j}$'s: 
\begin{align}
  &\sup_{\pi_1,\pi_2\in\Pi_{\rm NN}^{|\cA|}} S_j^{(3)} \nonumber\\
  &=\frac{1}{n_2} \sup_{\pi_1,\pi_2\in\Pi_{\rm NN}^{|\cA|}} \sum_{\{n_1+1\leq i\leq n|\ba_i=A_j\}}\left( \pi_{1,j}(\bx_i)-\pi_{2,j}(\bx_i) \right) \left( \hat{\mu}_{A_j}(\bx_i) - \mu_{A_j}(\bx_i)\right) \left(\frac{1}{\hat{e}_{A_j}(\bx_i)}- \frac{1}{e_{A_j}(\bx_i)} \right) \nonumber\\
  &\leq \frac{1}{n_2} \sum_{i=n_1+1}^n\left|  \hat{\mu}_{A_j}(\bx_i) - \mu_{A_j}(\bx_i)\right| \left|\frac{1}{\hat{e}_{A_j}(\bx_i)}- \frac{1}{e_{A_j}(\bx_i)}\right| \tag*{(since $ |\pi_{1,j}(\bx_i)-\pi_{2,j}(\bx_i)|\leq 1$)} \nonumber\\
  &\leq \sqrt{\frac{1}{n_2} \sum_{i=n_1+1}^n\left(  \hat{\mu}_{A_j}(\bx_i) - \mu_{A_j}(\bx_i)\right)^2} \sqrt{\frac{1}{n_2} \sum_{i=n_1+1}^n\left(\frac{1}{\hat{e}_{A_j}(\bx_i)}- \frac{1}{e_{A_j}(\bx_i)}\right)^2}\tag*{\mbox{(by Cauchy-Schwarz)}} \nonumber\\
  &\leq \eta^{-1}|\cA|e^{2M}\sqrt{\frac{1}{n_2} \sum_{i=n_1+1}^n\left(  \hat{\mu}_{A_j}(\bx_i) - \mu_{A_j}(\bx_i)\right)^2} \sqrt{\frac{1}{n_2} \sum_{i=n_1+1}^n\left(\hat{e}_{A_j}(\bx_i)- e_{A_j}(\bx_i)\right)^2}, \label{eq.proof1.S3}
\end{align}
where the last inequality holds since $e_{A_j}\geq \eta$ by Assumption \ref{assum.causal} and $\hat{e}_{A_j}\geq (|\cA|e^{2M})^{-1}$. We denote
\begin{align}
\omega_j = \eta^{-1}|\cA|e^{2M}\sqrt{\frac{1}{n_2} \sum_{i=n_1+1}^n\left(  \hat{\mu}_{A_j}(\bx_i) - \mu_{A_j}(\bx_i)\right)^2} \sqrt{\frac{1}{n_2} \sum_{i=n_1+1}^n\left(\hat{e}_{A_j}(\bx_i)- e_{A_j}(\bx_i)\right)^2},
\label{eq.omega}
\end{align}
and write $\sup_{\pi_1,\pi_2\in\Pi_{\rm NN}^{|\cA|}} S_j^{(3)} \le \omega_j$.

\noindent\textbf{Putting the $S_j^{(1)}$, $S_j^{(2)}$ and $S_j^{(3)}$ terms together:}
Combining (\ref{eq.proof1.S1}), (\ref{eq.proof1.S2}), (\ref{eq.proof1.S3}) gives rise to
\begin{align*}
  &\sup_{\pi_1,\pi_2\in\Pi_{\rm NN}^{|\cA|}} \Lambda_j(\pi_1,\pi_2)
\leq \sup_{\pi_1,\pi_2\in\Pi_{\rm NN}^{|\cA|}} S_j^{(1)} +\sup_{\pi_1,\pi_2\in\Pi_{\rm NN}^{|\cA|}} S_j^{(2)} + \sup_{\pi_1,\pi_2\in\Pi_{\rm NN}^{|\cA|}} S_j^{(3)} \nonumber\\
  &\leq \inf_{\lambda}~ 8 \lambda + \frac{192}{\sqrt{n_2}} \int_{\lambda}^{\max\limits_{\pi \in \Pi_{\rm NN}^{|\cA|}} \norm{\pi}_\Gamma} \sqrt{\log \cN(\theta, \Pi_{\rm NN}^{|\cA|}, \norm{\cdot}_{\Gamma})} d\theta  + 48|\cA|e^{2M}M \sqrt{\frac{\log 1/\delta}{2n_2}}+\omega_j
\end{align*}
with probability no less than $1-4\delta$ where we used $e^{2M}\geq \eta^{-1}$ according to (\ref{eq.M}).

According to (\ref{eq.proof1.S2.1}), we can apply the union probability bound for $j=1,\ldots,|\cA|$ and obtain
\begin{align}
  \cE_2&=\sup_{\pi_1,\pi_2\in \Pi_{\rm NN}^{|\cA|}} \tilde{\Delta}(\pi_1,\pi_2)-\hat{\Delta}(\pi_1,\pi_2)\nonumber\\
   &\leq \inf_{\lambda}~ 8|\cA| \lambda + \frac{192|\cA|}{\sqrt{n_2}} \int_{\lambda}^{\max\limits_{\pi \in \Pi_{\rm NN}^{|\cA|}} \norm{\pi}_\Gamma} \sqrt{\log \cN(\theta, \Pi_{\rm NN}^{|\cA|}, \norm{\cdot}_{\Gamma})} d\theta  + 48|\cA|^2e^{2M}M\sqrt{\frac{\log 1/\delta}{2n_2}}+\sum_{j=1}^{|\cA|}\omega_j \label{eq.proof1.boundE2}
\end{align}
with probability no less than $1-4|\cA|\delta$.

Combining (\ref{eq.proof1.boundE1}) and (\ref{eq.proof1.boundE2}), we have
\begin{align}
{\rm (II_1)} & \leq  \inf_{\lambda}~ (8|\cA|+4) \lambda + \frac{192|\cA|+96}{\sqrt{n_2}} \int_{\lambda}^{\max\limits_{\pi \in \Pi_{\rm NN}^{|\cA|}} \norm{\pi}_\Gamma} \sqrt{\log \cN(\theta, \Pi_{\rm NN}^{|\cA|}, \norm{\cdot}_{\Gamma})} d\theta \nonumber \\
& \quad + \sum_{j=1}^{|\cA|}\omega_j + \left(72|\cA|^2e^{2M}M + 12M\right)\sqrt{\frac{\log 1/\delta}{2n_2}}
\label{eqboundii1}
\end{align}
with probability at least $1-6|\cA|\delta$.

\noindent{\bf $\bullet$ Putting ${\rm (I_1),\ (II_1)}$ together.} Putting our estimates of ${\rm (I_1)}$ in \eqref{eqboundi1} and ${\rm (II_1)}$ in \eqref{eqboundii1} together, we get
\begin{align}
R(\pi^*_{\beta}, \hat{\pi}_{\rm DR}) & \leq  |\cA|M\varepsilon+\inf_{\lambda}~ (8|\cA|+4) \lambda + \frac{192|\cA|+96}{\sqrt{n_2}} \int_{\lambda}^{\max_{\pi \in \Pi_{\rm NN}^{|\cA|}} \norm{\pi}_\Gamma} \sqrt{\log \cN(\theta, \Pi_{\rm NN}^{|\cA|}, \norm{\cdot}_{\Gamma})} d\theta \nonumber \\
& \quad + \sum_{j=1}^{|\cA|}\omega_j + 84|\cA|^2e^{2M}M\sqrt{\frac{\log 1/\delta}{2n_2}}
\label{eq.proof1.regret1}
\end{align}
with probability at least $1-6|\cA|\delta$.
The upper bound in (\ref{eq.proof1.regret1}) depends on the covering number $\cN(\theta, \Pi_{\rm NN}^{|\cA|}, \norm{\cdot}_{\Gamma})$ and the integral upper limit $\max_{\pi \in \Pi_{\rm NN}^{|\cA|}} \norm{\pi}_\Gamma$ which can be estimated by the following lemmas (see the proofs in Appendix \ref{appendix.covering} and \ref{appendix.integral} respectively):
\begin{lemma}\label{lem.covering}
  Suppose Assumptions \ref{assum.causal} and \ref{assum.regularity} hold and define $\Pi_{\rm NN}^{|\cA|}$ according to (\ref{eq.PiNN}). Then
\begin{align}
\cN(\theta, \Pi_{\rm NN}, \|\cdot\|_{\Gamma}) \leq \left(\frac{2(|\cA|M + 2M/\eta)L_{\Pi}^2 (p_{\Pi}R/|\cA|+2) \kappa^L_{\Pi} (p_{\Pi}/|\cA|)^{L_{\Pi}+1}}{\theta}\right)^{K_{\Pi}}.
\label{eq.covering}
\end{align}
\end{lemma}
\begin{lemma}\label{lem.integral}
Suppose Assumptions \ref{assum.causal} and \ref{assum.regularity} hold. For any $\pi\in \Pi_{\rm NN}^{|\cA|}$, the following holds
\begin{align}
\|\pi\|_{\Gamma}^2 \leq (2M/\eta+|\cA|M)^2.
\label{eq.upperlimit}
\end{align}
\end{lemma}

Setting the network parameter as in (\ref{eq.sketch1.PiPara}) and using (\ref{eq.covering}), we have
\begin{align}
  \log\cN(\theta, \Pi_{\rm NN}, \|\cdot\|_{\Gamma})\leq C_1|\cA|\varepsilon^{-\frac{d}{\beta}}\log\frac{1}{\varepsilon} \left(\log^2\frac{1}{\varepsilon}+\log\frac{1}{\theta}\right)
  \label{eq.logCovering}
\end{align}
with $C_1$ depending on $\log D$, $d$, $B$, $\tau$, $\eta$, $\beta$, and the surface area of $\cM$.

Substituting (\ref{eq.logCovering}) and (\ref{eq.upperlimit}) into (\ref{eq.proof1.regret1}) gives
\begin{align}
R(\pi^*_{\beta}, \hat{\pi}_{\rm DR})
&\leq  |\cA|M\varepsilon + \sum_{j=1}^{|\cA|}\omega_j + 84|\cA|^2e^{2M}M\sqrt{\frac{\log 1/\delta}{2n_2}}\nonumber\\
&\quad+ \inf_{\lambda}~ 12|\cA| \lambda + \frac{288|\cA|}{\sqrt{n_2}} \int_{\lambda}^{|\cA|M + 2M/\eta} \sqrt{C_1|\cA|\varepsilon^{-\frac{d}{\beta}}\log\frac{1}{\varepsilon} \left(\log^2\frac{1}{\varepsilon}+\log\frac{1}{\theta}\right)} d\theta\nonumber\\
&\leq |\cA|M\varepsilon + \sum_{j=1}^{|\cA|}\omega_j + 84|\cA|^2e^{2M}M\sqrt{\frac{\log 1/\delta}{2n_2}}+ \inf_{\lambda}~ 12|\cA| \lambda \nonumber\\
&\quad+ C_2\frac{288|\cA|^{3/2}}{\sqrt{n_2}}M\eta^{-1} \varepsilon^{-\frac{d}{2\beta}} \sqrt{\log\frac{1}{\varepsilon} \left(\log^2\frac{1}{\varepsilon}+\log\frac{1}{\lambda}\right)}
\label{eq.thm1bound}
\end{align}
with probability no less than $1-6|\cA|\delta $ and $C_2$ depending on $\log D$, $d$, $B$, $\tau$, $\eta$, $\beta$, and the surface area of $\cM$.
Setting $\varepsilon=n_2^{-\frac{\beta}{2\beta+d}}, \delta=n_2^{-\frac{\beta}{2\beta+d}},\lambda=n_2^{-\frac{\beta}{2\beta+d}}$ implies (\ref{eq.thm1.para}) and (\ref{eq.thm1.holderbound}) in Theorem \ref{thm.holder}.
\end{proof}

\subsection{Proof of Corollary \ref{coral.holder}}\label{sec.proof.coral1}
\begin{proof}[Proof of Corollary \ref{coral.holder}.]
Corollary \ref{coral.holder} is proved based on Theorem \ref{thm.holder} and Lemma \ref{lemma:stage1rate}. We first derive an upper bound of the $\omega_j$'s using Lemma \ref{lemma:stage1rate}.
Taking an expectation on the both sides of (\ref{eq.omega}) gives rise to
\begin{align*}
\EE[\omega_j]&\leq \eta^{-1}|\cA|e^{2M}\EE\left[ \sqrt{\frac{1}{n_2} \sum_{i=n_1+1}^n\left(  \hat{\mu}_{A_j}(\bx_i) - \mu_{A_j}(\bx_i)\right)^2} \sqrt{\frac{1}{n_2} \sum_{i=n_1+1}^n\left(\hat{e}_{A_j}(\bx_i)- e_{A_j}(\bx_i)\right)^2}\right] \nonumber\\
& \leq \eta^{-1}|\cA|e^{2M}\left( \sqrt{\frac{1}{n_2} \sum_{i=n_1+1}^n\EE\left(  \hat{\mu}_{A_j}(\bx_i) - \mu_{A_j}(\bx_i)\right)^2} \sqrt{\frac{1}{n_2} \sum_{i=n_1+1}^n\EE \left(\hat{e}_{A_j}(\bx_i)- e_{A_j}(\bx_i)\right)^2}\right) \nonumber\\
&\leq \eta^{-1}|\cA|e^{2M}\sqrt{\EE\left[  \|\hat{\mu}_{A_j} - \mu_{A_j}\|_{L^2}^2\right]} \sqrt{\EE \left[\|e_{A_j} -\hat{e}_{A_j}\|_{L^2}^2\right]}\nonumber\\
&\leq C_1e^{2M}(M+\sigma) \eta^{-\frac{3\alpha+d}{2\alpha+d}}|\cA|^{\frac{4\alpha+d}{2\alpha+d}} n_1^{-\frac{2\alpha}{2\alpha+d}}\log^3 n_1,
\end{align*}
where the second inequality is due to Jensen's inequality and the unconfoundedness condition in Assumption \ref{assum.unconf}, the last inequality is due to Lemma \ref{lemma:stage1rate}, and $C_1$ is a constant depending on $\log D, d, B,\tau,\alpha$ and the surface area of $\cM$.

%

By  Markov's inequality, for any $\delta>0$,
\begin{align}
\PP\left(\omega_j>\delta\right)\leq \frac{\EE\left[\omega_j\right]}{ \delta}\leq \frac{1}{\delta}C_1G_1 n_1^{-\frac{2\alpha}{2\alpha+d}}\log^3 n_1,
\label{eq.coro1.markov}
\end{align}
where $G_1=e^{2M}(M+\sigma) \eta^{-\frac{3\alpha+d}{2\alpha+d}}|\cA|^{\frac{4\alpha+d}{2\alpha+d}}$.
Applying a union probability bound gives rise to
\begin{align}
  \PP\left(\sum_{j=1}^{|\cA|}\omega_j> |\cA|\delta\right)\leq \frac{C_1}{\delta}|\cA|G_1 n_1^{-\frac{2\alpha}{2\alpha+d}}\log^3 n_1.
  \label{eq.SumOmega}
\end{align}

Substituting (\ref{eq.SumOmega}) into (\ref{eq.thm1.holderbound}) and setting $\delta=C_1|\cA|G_1n_1^{-\frac{\alpha}{2\alpha+d}}$, we get
\begin{align*}
  R(\pi^*_{\beta}, \hat{\pi}_{\rm DR})& \leq C_2e^{2M}|\cA|^{2}(M+\sigma) n_2^{-\frac{\beta}{2\beta+d}}\log^{2}n_2+C_1|\cA|^2G_1n_1^{-\frac{\alpha}{2\alpha+d}}\nonumber\\
  &\leq C_3e^{2M}|\cA|^{\frac{8\alpha+2d}{2\alpha+d}} (M+\sigma)n^{-\frac{\alpha \wedge\beta}{2(\alpha \wedge\beta)+d}}\log^{2}n
\end{align*}
with probability no less than $1-C_4n^{-\frac{\alpha\wedge \beta}{2(\alpha\wedge \beta)+d}}\log^3 n$, where $C_4$ is an absolute constant, $C_2,C_3$ are constants depending on $\log D,d,B,\tau,\alpha,\beta,\eta$, and the surface area of $\cM$.
\end{proof}

\subsection{Proof of Theorem \ref{thm.optimal}}\label{sec.proof.thm2}
\begin{proof}[Proof of Theorem \ref{thm.optimal}.]
In Theorem \ref{thm.optimal}, $\pi^*$ is the unconstrained optimal policy. We prove Theorem \ref{thm.optimal} in a similar manner as  we prove Theorem \ref{thm.holder}. We first decompose the regret using an oracle inequality:
\begin{align}\label{eq.proof2.decomposition}
R(\pi^*, \hat{\pi}_{\rm DR}) = \underbrace{Q(\pi^*) - Q(\hat{\pi}^*)}_{\rm (I_2)} + \underbrace{Q(\hat{\pi}^*) - Q(\hat{\pi}_{\rm DR})}_{\rm (II_2)},
\end{align}
where $\hat{\pi}^*$ is the same as in (\ref{eqhatpistar}).
In (\ref{eq.proof2.decomposition}), ${\rm (I_2)}$ is the bias of approximating $\pi^*$ by the policy class $\Pi_{{\rm NN}(\tem)}^{|\cA|}$, and ${\rm (II_2)}$ is the same as ${\rm(II_1)}$ in (\ref{eq.sketch1.decomposition}) which can be bounded similarly.

Following the proof of Theorem \ref{thm.holder} and Corollary \ref{coral.holder}, we can derive that
  \begin{align*}
 {\rm (II_2)}\leq  C_1e^{2M}|\cA|^{\frac{8\alpha+2d}{2\alpha+d}}(M+\sigma)n^{-\frac{\alpha }{2\alpha +d}}\log^{2}n\log^{1/2} (1/\tem)
  \end{align*}
  with probability no less than $1-C_2n^{-\frac{\alpha}{2\alpha+d}}\log^3 n $ where $C_2$ is an absolute constant and $C_1$ depends on $\log D,d,B,\tau,\alpha$, and the surface area of $\cM$. In addition, $\hat{\pi}_{\rm DR}\in \Pi_{{\rm NN}(\tem)}^{|\cA|}$ with $L_{\Pi}, p_{\Pi}, K_{\Pi}, \kappa_{\Pi}$ and $R_{\Pi}$ given in (\ref{eq.thm1.para}).   It remains to show ${\rm (I_2)} \leq 2cMt^q+M|\cA|^2\exp\left[\left(-Mt+2n^{-\frac{2\alpha}{2\alpha+d}}\right)/H\right]$ for any $t\in(0,1)$.

\noindent {\bf  Bounding ${\rm (I_2)}$}.
We estimate $Q(\pi^*) - Q(\hat{\pi}^*)$ on two regions. The first region is, for any given $t\in(0,1)$,
\begin{align*}
  \chi_t=\left\{\bx~\big|~ \bx\in\cM, \mu_{A_{j^*(\bx)}}(\bx)-\max_{j\neq j^*(\bx)}\mu_{A_j}(\bx)\leq Mt\right\}
\end{align*}
with $j^*(\bx)=\argmax_j \mu_{A_j}(\bx)$. On $\chi_t$, the gap between $\mu_{A_{j^*(\bx)}}(\bx)$, the reward of the optimal action, and the reward of the second optimal action is smaller than $Mt$. Assumption \ref{assum.noise} yields $\PP(\chi_t)\leq ct^q$. The second region is
\begin{align*}
  \chi_t^{\complement}=\left\{\bx~\big|~ \bx\in\cM, \mu_{A_{j^*(\bx)}}(\bx)-\max_{j\neq j^*(\bx)}\mu_{A_j}(\bx)> Mt\right\}
\end{align*}
on which the gap between $\mu_{A_{j^*}}(\bx)$ and the reward of any other action is larger than $Mt$.

For any policy $\pi$, we have
  \begin{align*}
    Q(\pi)=\EE[Y(\pi(\bx))]=\int_{\cM}\left\langle\bmu(\bx),\pi(\bx)\right\rangle d\PP(\bx),
  \end{align*}
where $\bmu(\bx) = [\mu_{A_1}(\bx),\dots,\mu_{A_{|\cA|}}(\bx)]^{\top}$.
According to \citet[Theorem 2]{chen2019efficient}, for any $\varepsilon\in(0,1)$, there is a neural network architecture $\cF(L,p,K,\kappa,R)$ with
  \begin{align*}
    L=O\left(\log 1/\varepsilon\right), \ p=O\left(\varepsilon^{-\frac{d}{\alpha}}\right), \  K=O\left(\varepsilon^{-\frac{d}{\alpha}}\log 1/\varepsilon\right), \ \kappa=\max\{B,M,\sqrt{d},\tau^2\}, \ R=M,
  \end{align*}
  such that for each $\mu_{A_j}$, there exists $\tilde{\mu}_{A_j}\in \cF(L,p,K,\kappa,R)$ and $\|\tilde{\mu}_{A_j}-\mu_{A_j}\|_{\infty}\leq \varepsilon$.

  Define $\tilde{\pi}=\Softmax_{\tem}(\tilde{\mu}_{A_1},\dots,\tilde{\mu}_{A_{|\cA|}})$. Since $\hat{\pi}^*=\arg\max_{\pi\in\Pi_{{\rm NN}(\tem)}^{|\cA|}} Q(\pi)$, $\tilde{\pi}\in \Pi_{{\rm NN}(\tem)}^{|\cA|}$, and $Q(\hat{\pi}^*)\geq Q(\tilde{\pi})$, we have
  \begin{align}
    Q(\pi^*)-Q(\hat{\pi}^*)&\leq Q(\pi^*)-Q(\tilde{\pi})=\int_{\cM}\inner{\bmu(\bx)}{\pi^*(\bx)-\tilde{\pi}(\bx)}d\PP(\bx) \nonumber\\
    &= \int_{\chi_t}\inner{\bmu(\bx)}{\pi^*(\bx)-\tilde{\pi}(\bx)}d\PP(\bx) + \int_{\chi_t^{\complement}}\inner{\bmu(\bx)}{\pi^*(\bx)-\tilde{\pi}(\bx)}d\PP(\bx).
    \label{eq.proof2.bias}
  \end{align}
  The first integral in (\ref{eq.proof2.bias}) can be bounded as
  \begin{align}
    \int_{\chi_t}\inner{\bmu(\bx)}{\pi^*(\bx)-\tilde{\pi}(\bx)}d \PP(\bx)&\leq \int_{\chi_t}\|\bmu(\bx)\|_{\infty}\|\pi^*(\bx)-\tilde{\pi}(\bx)\|_{1}d\PP(\bx)\nonumber\\
    &\leq 2M\int_{\chi_t} 1 d\PP(\bx)\leq 2cMt^q \label{eq.proof2.chi}
  \end{align}
  where $\|\bmu(\bx)\|_{\infty}=\max_j |\mu_{A_j}(\bx)|$ and $\|\pi(\bx)\|_1=\sum_{j=1}^{|\cA|} |[\pi(\bx)]_j|$.
  For the second integral, we first derive an upper bound of $\|\pi^*(\bx)-\tilde{\pi}(\bx)\|_{\infty}$. Since $\pi^*$ is the unconstrained optimal policy, represented by a one-hot vector $\pi^*(\bx)=A_{j^*(\bx)}$, we deduce
  \begin{align*}
    [\pi^*(\bx)-\tilde{\pi}(\bx)]_j=\begin{cases}
      -\frac{\exp(\tilde{\mu}_{A_j}(\bx)/\tem)}{\sum_{k}\exp(\tilde{\mu}_{A_k}(\bx)/\tem)} &\mbox{ if } \pi^*(\bx)\neq A_j,\\
     \frac{\sum_{k\neq j}\exp(\tilde{\mu}_{A_k}(\bx)/\tem)}{\sum_{k}\exp(\tilde{\mu}_{A_k}(\bx)/\tem)} &\mbox{ if } \pi^*(\bx)=A_j,
    \end{cases}
  \end{align*}
  \begin{align*}
    \left|[\pi^*(\bx)-\tilde{\pi}(\bx)]_j\right|\leq\begin{cases}
      \frac{\max_{k\neq j^*}\exp((\mu_{A_k}(\bx)+\varepsilon)/\tem)}{\exp((\mu_{A_{j^*}}(\bx)-\varepsilon)/\tem)} &\mbox{ if } j\neq j^*(\bx),\\
     \frac{|\cA|\max_{k\neq j^*}\exp((\mu_{A_k}(\bx)+\varepsilon)/\tem)}{\exp((\mu_{A_{j^*}}(\bx)-\varepsilon)/\tem)} &\mbox{ if } j=j^*(\bx).
    \end{cases}
  \end{align*}
  Therefore $\|\pi^*(\bx)-\tilde{\pi}(\bx)\|_{\infty}\leq |\cA|\exp\left(\max_{k\neq j^*(\bx)}\left(\mu_{A_k}(\bx)-\mu_{A_{j^*}}(\bx)+2\varepsilon\right)/\tem\right)$.
  Thus
  \begin{align}
&\int_{\chi_t^{\complement}}\inner{\bmu(\bx)}{\pi^*(\bx)-\tilde{\pi}(\bx)}d\PP(\bx)\leq \int_{\chi_t^{\complement}} \|\bmu(\bx)\|_{1} \|(\pi^*(\bx)-\tilde{\pi}(\bx))\|_{\infty} d\PP(\bx) \nonumber\\
\leq& \int_{\chi_t^{\complement}} M|\cA|^2\exp\left((-Mt+2\varepsilon)/\tem\right)d\PP(\bx)\leq M|\cA|^2\exp\left((-Mt+2\varepsilon)/\tem\right). \label{eq.proof2.chic}
  \end{align}
  Combining (\ref{eq.proof2.chi}) and (\ref{eq.proof2.chic}),  and setting $\varepsilon=n^{-\frac{\alpha}{2\alpha+d}}$
  give rise to $$Q(\pi^*)-Q(\hat{\pi}^*)\leq 2cMt^q+M|\cA|^2\exp\left((-Mt+2n^{-\frac{\alpha}{2\alpha+d}})/\tem\right),$$ which completes the proof.
\end{proof}

\subsection{Proof of Theorem \ref{thm:continuous}}\label{sec.proof.thm3}
\begin{proof}[Proof of Theorem \ref{thm:continuous}.]
  We denote
  $$
  \hat{\pi}^*=\argmax_{\pi\in \Pi_{\textrm{ NN}(\tem)}^V} Q^{\rm (D)}(\pi),
  $$
  where $Q^{\rm (D)}(\pi)$ is defined in (\ref{eq:discretizedQ}).  
  The regret can be decomposed as
  \begin{align}
    R(\pi^*_{\rm C},\hat{\pi}_{\text{ C-DR}})=\underbrace{Q(\pi^*_{\rm C})-Q^{\rm (D)}(\hat{\pi}^*)}_{\rm (I_3)}+ \underbrace{Q^{\rm (D)}(\hat{\pi}^*) -Q^{\rm (D)}(\hat{\pi}_{\rm C-DR})}_{\rm (II_3)} +\underbrace{Q^{\rm (D)}(\hat{\pi}_{\rm C-DR})- Q(\hat{\pi}_{\rm C-DR})}_{\rm (III_3)}.
\label{eq.proof3.decompose}
  \end{align}
  In (\ref{eq.proof3.decompose}), ${\rm (I_3)}$ is the bias of approximating the optimal policy $\pi^*_{\rm C}$ using the neural network policy class $\Pi_{{\rm NN}(\tem)}^V$ in the discretized setting. ${\rm (II_3)}$ is the variance of the estimated policy in $\Pi_{{\rm NN}(\tem)}^V$. ${\rm (III_3)}$ characterizes the difference between the discretized policy reward and the continuous policy reward of $\hat{\pi}_{\rm C-DR}$. We next derive the bounds for each part.

  \noindent\textbf{Bounding ${\rm (I_3)}$.}
  By Assumption \ref{assum.regularity.continuous}, $\mu_{I_j} \in \cH^\alpha(\cM)$. According to \citet{chen2019efficient}, H\"older functions can be uniformly approximated by a neural network class if the network parameters are properly chosen. For any $\varepsilon\in(0,1)$ there exists a network architecture $\cF(L,p,K,\kappa,R)$ with
  \begin{align}
    L=O(\log 1/\varepsilon), p=O\left(\varepsilon^{-\frac{d}{\alpha}}\right), K=O\left(\varepsilon^{-\frac{d}{\alpha}}\log 1/\varepsilon\right), \kappa=\max\{B,M,\sqrt{d},\tau^2\}, R=M,
    \label{eq.proof3.para}
  \end{align}
  such that if the weight parameters are properly chosen, we have $\tilde{\mu}_{I_j}\in \cF(L,p,K,\kappa,R)$   satisfying $$\|\tilde{\mu}_{I_j}-\mu_{I_j}\|_{\infty}\leq \varepsilon.$$ We then define an intermediate policy $$\tilde{\pi}=\Softmax_\tem(\tilde{\mu}_{I_1},\dots,\tilde{\mu}_{I_V}).$$


Let $A^*(\bx)=\argmax_{A\in[0,1]} \mu(\bx, A)$. Then $\pi^*_{\rm C}(\bx)=A^*(\bx)$.
  After defining $\bmu(\bx)=[\mu_{I_1}(\bx),\dots,\mu_{I_V}(\bx)]^{\top}\in \RR^V$, we can bound ${\rm (I_3)}$ as
  \begin{align}
{\rm (I_3)}&=Q(\pi^*_{\rm C})-Q^{\rm (D)}(\hat{\pi}^*)\leq Q(\pi^*_{\rm C})-Q^{\rm (D)}(\tilde{\pi})=\int_{\cM} \mu(\bx,A^*(\bx))d\PP(\bx)- \int_{\cM} \langle \bmu(\bx) , \tilde{\pi}(\bx)\rangle d\PP(\bx) \nonumber\\
    &=\underbrace{\int_{\cM} \mu(\bx,A^*(\bx)) d\PP(\bx)- \int_{\cM} \langle \bmu(\bx) , [\ind_{\{A^*(\bx)\in I_1\}},\dots,\ind_{\{A^*(\bx)\in I_V\}}]^{\top}\rangle d\PP(\bx)}_{T_1} \nonumber \\
    &\quad+ \underbrace{\int_{\cM} \inner{\bmu(\bx)}{[\ind_{\{A^*(\bx)\in I_1\}},\dots,\ind_{\{A^*(\bx)\in I_V\}}]^{\top}}d\PP(\bx)-\int_{\cM} \langle \bmu(\bx) , \tilde{\pi}(\bx) \rangle d\PP(\bx)}_{T_2}.
    \label{eq.proof.continu.I.0}
  \end{align}
  If $A^*(\bx)\in I_j$, we denote $j^*(\bx)=j$ and $I_*(\bx)=I_j$. According to Assumption \ref{assum.regularity.continuous} and (\ref{eq.MC}), $M$ is a Lipschitz constant of the function $\mu(\bx,\cdot)$ for any fixed $\bx \in \cM$. Since $A^*(\bx)\in I_*(\bx)$, $|\mu(\bx,A^*(\bx))-\mu_{I_*(\bx)}(\bx)|\leq M/V$ for any $\bx\in\cM$. Hence $T_1$ can be bounded as
  \begin{align}
    T_1=\int_{\cM} \mu(\bx,A^*(\bx))-\mu_{I_*(\bx)}(\bx) d\PP(\bx)\leq M/V.
    \label{eq.proof.continu.T1}
  \end{align}
  We then derive the bound for $T_2$ on two regions. The first region is 
  \begin{align*}
  \chi_{t,\gamma}=\{\bx| \mu(\bx,A^*(\bx))-\mu(\bx,A)\leq Mt \mbox{ given } |A-A^*(\bx)|\geq \gamma\}
  \end{align*}
  and the second region is $ \chi_{t,\gamma}^{\complement}$. According to Assumption \ref{assum.noise.continuous}, $\PP(\chi_{t,\gamma})\leq ct^q(1-\gamma)$.

  $T_2$ is decomposed as
  \begin{align}
    T_2&=\int_{\chi_{t,\gamma}} \inner{\bmu(\bx)}{ [\ind_{\{A^*(\bx)\in I_1\}},\dots,\ind_{\{A^*(\bx)\in I_V\}}]^{\top}-\tilde{\pi}(\bx)}d\PP(\bx) \nonumber\\
    &\quad+\int_{\chi^{\complement}_{t,\gamma}} \inner{\bmu(\bx)}{[\ind_{\{A^*(\bx)\in I_1\}},\dots,\ind_{\{A^*(\bx)\in I_V\}}]^{\top}-\tilde{\pi}(\bx)}d\PP(\bx).
    \label{eq.proof3.T2}
  \end{align}
  The first integral in (\ref{eq.proof3.T2}) is bounded as
  \begin{align}
    \int_{\chi_{t,\gamma}} \inner{\bmu(\bx)}{[\ind_{\{A^*(\bx)\in I_1\}},\dots,\ind_{\{A^*(\bx)\in I_V\}}]^{\top}-\tilde{\pi}(\bx)}d\PP(\bx)\leq 2cMt^q(1-\gamma).
    \label{eq.proof.continu.T2.chi}
  \end{align}

  We then derive an upper bound of the second integral in (\ref{eq.proof3.T2}) in a way similar to the derivation of (\ref{eq.proof2.chic}). Denote
  \begin{align*}
    \Xi(\bx)= \left[-\frac{\exp(\tilde{\mu}_{I_1}(\bx)/\tem)}{\sum_{j=1}^V \exp(\tilde{\mu}_{I_j}(\bx)/\tem)},\dots,\frac{\sum_{j\neq j^*(\bx)}\exp(\tilde{\mu}_{I_j}(\bx)/\tem)}{\sum_{j=1}^V \exp(\tilde{\mu}_{I_j}(\bx)/\tem)},\dots,-\frac{\exp(\tilde{\mu}_{I_V}(\bx)/\tem)}{\sum_{j=1}^V \exp(\tilde{\mu}_{I_j}(\bx)/\tem)}\right]^{\top}.
  \end{align*}
  Similar to (\ref{eq.proof2.chic}), we have
  \begin{align}
    &\quad~ \int_{\chi^{\complement}_{t,\gamma}} \inner{\bmu(\bx)}{[\ind_{\{A^*(\bx)\in I_1\}},\dots,\ind_{\{A^*(\bx)\in I_V\}}]^{\top}-\tilde{\pi}(\bx)}d\PP(\bx) =\int_{\chi^{\complement}_{t,\gamma}} \inner{\bmu(\bx)}{\Xi(\bx)} d \PP(\bx)\nonumber\\
    &\leq\int_{\chi^{\complement}_{t,\gamma}}\|\bmu(\bx)\|_1 \|\Xi(\bx)\|_{\infty} d \PP(\bx)\leq VM\int_{\chi^{\complement}_{t,\gamma}} \|\Xi(\bx)\|_{\infty} d \PP(\bx)  \label{eq.proof.continu.T2.chic.0}.
  \end{align}

  To derive an upper bound of $\|\Xi\|_{\infty}$, we need a lower bound of $\mu_{I_*(\bx)}(\bx)-\mu_{I_j}(\bx)$ for any $1\leq j\leq V$ and $ j\neq j^*(\bx)$.
  By Assumption \ref{assum.regularity.continuous},
  For any $j$ and $\tilde{A} \in I_j$, one has
  \begin{align*}
    |\mu(\bx,\tilde{A})-\mu(\bx, A_j)|\leq M/V,
  \end{align*}
  and
  \begin{align*}
    |\mu_{I_j}(\bx)-\mu(\bx, A_j)|\leq \frac{1}{|I_j|}\int_{I_j} |\mu(\bx,A)-\mu(\bx,A_j)|dA\leq M/V
  \end{align*}
  where $|I_j|=1/V$ represents the length of $I_j$.

  As a result, on $\chi^{\complement}_{t,\gamma}$, for any $j\neq j^*(\bx)$, we have
  \begin{align*}
\mu_{I_*(\bx)}(\bx)-\mu_{I_j}(\bx)& \geq \mu(\bx, A_{j^*(\bx)})-\mu(\bx, A_j) -2M/V \nonumber\\
    &\geq \mu(\bx,A^*(\bx))-\mu(\bx, A_j)-|\mu(\bx,A^*(\bx))-\mu(\bx, A_{j^*(\bx)})|-2M/V\nonumber\\
    &\geq Mt-3M/V,
  \end{align*}
  where the last inequality holds for two reasons: (1) $A^*(\bx)\in I_*(\bx)$ and $A_{j^*(\bx)} \in I_*(\bx)$; (2) We set $V <1/(2\gamma) $, and then $j\neq j^*(\bx)$ implies $|A_j-A^*(\bx)|\geq 1/(2V)\geq \gamma$. We then deduce
  \begin{align}
    \|\Xi\|_{\infty}\leq (V-1)\exp(-(Mt-3M/V-2\varepsilon)/\tem).
    \label{eq.XiBound}
  \end{align}
  Plugging (\ref{eq.XiBound}) into (\ref{eq.proof.continu.T2.chic.0}), we have
  \begin{align}
    & \int_{\chi^{\complement}_{t,\gamma}} \inner{\bmu(\bx)}{[\ind_{\{A^*(\bx)\in I_1\}},\dots,\ind_{\{A^*(\bx)\in I_V\}}]^{\top}-\tilde{\pi}(\bx)}d\PP(\bx) \nonumber\\
\leq& VM\cdot (V-1)\exp(-(Mt-3M/V-2\varepsilon)/\tem)\nonumber\\
\leq& MV^2\exp(-(Mt-3M/V-2\varepsilon)/\tem) \label{eq.proof.continu.T2.chic}.
  \end{align}

  Substituting (\ref{eq.proof.continu.T1}), (\ref{eq.proof.continu.T2.chi}) and (\ref{eq.proof.continu.T2.chic}) into (\ref{eq.proof.continu.I.0}), if $V <1/(2\gamma)$, we have
  \begin{align}
    {\rm (I_3)}\leq \frac{M}{V}+2cMt^q(1-\gamma)+MV^2\exp\left(-(Mt-3M/V-2\varepsilon)/\tem\right).
    \label{eq.proof.continu.I}
  \end{align}
  \textbf{Bounding ${\rm (II_3)}$.} ${\rm (II_3)}$ has the same form as ${\rm (II_1)}$ in (\ref{eq.sketch1.decomposition}). We derive the upper bound by following the same procedure while $|\cA|$ is replaced $V$. Besides, we need to express the estimation error of $\hat{\mu}_{I_j}$'s and $\hat{e}_{I_j}$'s in terms of $V$.
  Note that $e_{I_j}\geq \eta/V$. By Lemma \ref{lemma:stage1rate}, we can find $\hat{\mu}_{I_j}\in \cF(L_1,p_1,K_1,\kappa_1,R_1)$ with
  \begin{align*}
  &L_1=O(\log (\eta n_1/V)), \quad p_1=O\left( (\eta n_1/V)^{\frac{d}{2\alpha+d}}\right), \quad K_1=O\left( (\eta n_1/V)^{\frac{d}{2\alpha+d}}\log (\eta n_1/V)\right), \\
  &\hspace{1.3in} \kappa_1=\max\{ B,M,\sqrt{d},\tau^2\}, \quad R_1=M,
  \end{align*}
  such that
  \begin{align}
  \EE\left[\|\hat{\mu}_{I_j}-\mu_{I_j}\|_{L^2}^2\right]\leq C_1(M^2+\sigma^2) (\eta n_1/V)^{-\frac{2\alpha}{2\alpha+d}}\log^3 (\eta n_1/V)
  \label{eq.proof3.hmu}
  \end{align}
  with $C_1$ being a constant depending on $\log D,B, \tau,\alpha$ and the surface area of $\cM$.
  Similarly, we can find $\hat{g}\in \cF(L_2,p_2,K_2,\kappa_2,R_2)$ with
  \begin{align*}
  &L_2=O(\log (n_1/V)),\ p_2=O\left( V^{-\frac{2d}{2\alpha+d}}n_1^{\frac{d}{2\alpha+d}}\right),\ K_2=O\left( V^{-\frac{2d}{2\alpha+d}}n_1^{\frac{d}{2\alpha+d}}\log (n_1/V)\right), \\
  &\hspace{1.3in}\kappa_2=\max\{ B,M,\sqrt{d},\tau^2\},\ R_2=M,
\end{align*}
such that
\begin{align}
  \EE[\|\hat{e}_{I_j}-e_{I_j}\|_{L^2}^2]\leq C_2 M^2V^{\frac{4\alpha}{2\alpha+d}} n_1^{-\frac{2\alpha}{2\alpha+d}}\log^3n_1
  \label{eq.proof3.he}
\end{align}
with $C_2$ depending on $\log D,B,M,\alpha$ and the surface area of $\cM$.

Following the proof of Corollary \ref{coral.holder} and using (\ref{eq.proof3.hmu}) and (\ref{eq.proof3.he}), we rewrite (\ref{eq.SumOmega}) as
\begin{align}
  \PP\left(\sum_{j=1}^{V}\omega_j\geq V\delta_1\right) \leq \frac{C_3G_2}{\delta_1}V^{\frac{9\alpha+3d}{2\alpha+d}} n_1^{-\frac{2\alpha}{2\alpha+d}}\log^3 n_1
  \label{eq.proof3.omega}
\end{align}
with $G_2=e^{2M}(M+\sigma)\eta^{-\frac{3\alpha+d}{2\alpha+d}}$ and $C_3$ being an constant depending on $\log D,B,\tau,\eta,\alpha$ and the surface area of $\cM$.

By replacing $|\cA|$ by $V$ and $\eta$ by $\eta/V$ in $\cE_1$ and $\cE_2$ in the proof of Theorem \ref{thm.holder}, and substituting (\ref{eq.proof3.omega}), one derives
\begin{align*}
Q^{\rm (D)}(\hat{\pi}^*)- Q^{\rm (D)}( \hat{\pi}_{\text{ C-DR}}) & \leq V \delta_1 + 84e^{2M}V^2M\sqrt{\frac{\log 1/\delta}{2n_2}}+\inf_{\lambda}~ 12V \lambda \nonumber\\
&\quad + \frac{288V}{\sqrt{n_2}} \int_{\lambda}^{VM + 2VM/\eta} \bigg[K_{\Pi} \log \big(\theta^{-1}(VM + 2VM/\eta)\times\nonumber\\
&\qquad L_{\Pi}^2(p_{\Pi}B/V+2)\max(\kappa_{\Pi},1/\tem)^{L_{\Pi}} (p_{\Pi}/V)^{L_{\Pi}+1}\big)\bigg]^{1/2} d\theta
\end{align*}
with probability no less than
$$1-6V\delta-\frac{C_3 G_2}{\delta_1} V^{\frac{9\alpha+3d}{2\alpha+d}} n_1^{-\frac{2\alpha}{2\alpha+d}}\log^3 n_1.$$
Here $\hat{\pi}_{\text{ C-DR}}\in \Pi_{{\rm NN}(\tem)}^V$ where the network class $\Pi_{{\rm NN}(\tem)}^V$ has the parameters
\begin{align*}
  L_{\Pi}=L,\ p_{\Pi}=O(Vp),\ K_{\Pi}=O(VK),\ \kappa_{\Pi}=\kappa,\ R_{\Pi}=R
\end{align*}
with $L,p,K,\kappa$ and $R$ defined in (\ref{eq.proof3.para}).

Setting $ \varepsilon=n^{-\frac{\alpha}{2\alpha+d}}, \delta=n^{-\frac{\alpha}{2\alpha+d}}, \delta_1=C_3G_2V^{\frac{3}{2}}n_1^{-\frac{\alpha}{2\alpha+d}}$ and $\lambda= V^{\frac{3}{2}}n_2^{-\frac{\alpha}{2\alpha+d}}$ gives rise to
\begin{align*}
  &L_{\Pi}=O(\log n), p_{\Pi}=O\left( Vn^{\frac{d}{2\alpha+d}}\right), K_{\Pi}=O\left( Vn^{\frac{d}{2\alpha+d}}\log n\right), \kappa_{\Pi}=\max\{B,M,\sqrt{d},  \tau^2\}, R_{\Pi}=M,
\end{align*}
and
\begin{align}
  {\rm (II_3)} \leq C_4e^{2M}(M+\sigma) V^{\frac{5}{2}} n^{-\frac{\alpha}{2\alpha+d}}\log^{2} n \log^{1/2} (1/\tem)
  \label{eq.proof.continu.II}
\end{align}
with probability no less than $1-C_5V^{\frac{12\alpha+3d}{2(2\alpha+d)}} n^{-\frac{\alpha}{2\alpha+d}}\log^3 n$, where $C_4$ is a constant depending on $\log D,B,\tau,\eta,\alpha$, and the surface area of $\cM$, $C_5$ is an absolute constant.

\textbf{Bounding ${\rm (III_3)}$.}
According to Lemma \ref{lem.diserror},
\begin{align}
  {\rm (III_3)}&=Q^{\rm (D)}(\hat{\pi}_{\text{ C-DR}})-Q(\hat{\pi}_{\text{ C-DR}})\leq M/V.
  \label{eq.proof.continu.III}
\end{align}


\textbf{Putting all ingredients together.}
Putting (\ref{eq.proof.continu.I}), (\ref{eq.proof.continu.II}) and (\ref{eq.proof.continu.III}) together and using $\varepsilon=n^{-\frac{\alpha}{2\alpha+d}}$ give rise to
\begin{align*}
  R(\pi^*_{\rm C},\hat{\pi}_{\text{ C-DR}})&\leq \frac{2M}{V} +C_4e^{2M}(M+\sigma) V^{\frac{5}{2}} n^{-\frac{\alpha}{2\alpha+d}}\log^{2} n \log^{1/2} 1/\tem \nonumber \\
  &\quad+2cMt^q(1-\gamma)+MV^2\exp\left(-\left(Mt-3M/V-2n^{-\frac{\alpha}{2\alpha+d}}\right)/\tem\right)
\end{align*}
with probability no less than $1-C_5V^{\frac{12\alpha+3d}{2(2\alpha+d)}} n^{-\frac{\alpha}{2\alpha+d}}\log^3 n $ for any $t$ and $\gamma<1/4V$.

Setting $V=n^{\frac{2\alpha}{7(2\alpha+d)}}, \gamma=\frac{1}{4V}$ and $t>\frac{2}{V}+2\varepsilon/M$, we get
\begin{align*}
  R(\pi^*_{\rm C},\hat{\pi}_{\text{ C-DR}})&\leq C_4e^{2M}(M+\sigma)n^{-\frac{2\alpha}{7(2\alpha+d)}}\log^{2} n \log^{1/2} 1/\tem \nonumber \\
  &\quad+2cMt^q+Mn^{\frac{4\alpha}{7(2\alpha+d)}} \exp\left(-\left(Mt-4Mn^{-\frac{2\alpha}{7(2\alpha+d)}}\right)/\tem\right)
\end{align*}
for any $t \in (2(1+1/M)n^{-\frac{2\alpha}{7(2\alpha+d)}},1)$ with probability no less than $1-C_6n^{-\frac{2\alpha^2+4\alpha d}{7(2\alpha+d)^2}}\log^3 n$,  where $C_6$ is an absolute constant. In addition, $\hat{\pi}_{\text{ C-DR}}\in \Pi_{{\rm NN}(\tem)}^V $ with $L_{\Pi},p_{\Pi},K_{\Pi},\kappa_{\Pi},R_{\Pi}$ defined in (\ref{eq.thm3.para}), $\hat{\mu}_{I_j}\in \cF(L_1,p_1,K_1,\kappa_1,R_1)$ for $j=1,\dots,V$ with the parameters defined in (\ref{eq.mupara.continuous}), $\hat{g}\in \cF(L_2,p_2,K_2,\kappa_2,R_2)$ with the parameters defined in (\ref{eq.epara.continuous}).
\end{proof}

\section{Conclusion and Discussion}\label{sec:discussion}
This paper establishes statistical guarantee for doubly robust off-policy learning by neural networks. The covariate is assumed to be on a low-dimensional manifold. Non-asymptotic regret bounds for the learned policy are proved in the finite-action scenario and in the continuous-action scenario.  Our results show that when the covariates exhibit low dimensional-structures, neural networks provide a fast convergence rate whose exponent depends on the intrinsic dimension of the manifold instead of the ambient dimension. Our results partially justify the success of neural networks in causal inference with high-dimensional covariates.


We finally provide some discussions in connection with the existing literature.

\vspace{0.1in}
\noindent {\bf $\bullet$ Sample Complexity Lower Bound without Low Dimensional Structures}. \citet{gao2020minimax} established a lower bound of the sample complexity for policy evaluation (or treatment effect estimation), when the covariates are in $\RR^D$ and do not have low-dimensional structures. Specifically, they assume that both the initial policy and reward functions belong to a H\"{o}lder space. The sample complexity needs to be at least exponential in the dimension $D$. This result shows that the rate can not be improved unless additional assumptions are made. By assuming that the covariates are on a $d$-dimensional manifold, our sample complexity only depends on the intrinsic dimension $d$.
We remark that \citet{gao2020minimax} studied the H\"{o}lder space with a H\"{o}lder index $\alpha\in(0,1]$, while we focus on the case of $\alpha\geq1$. In the case that $\alpha=1$, if we have $q\geq 1$ (in Assumption \ref{assum.noise}), 
Corollary \ref{corollaryd2} gives the convergence rate $O\left(n^{-\frac{1}{2+d}}\log^3n\right)$. This rate is better than the minimax rate $O\left(n^{-\frac{1}{2+D}}\right)$ from \citet{gao2020minimax} thanks to the low-dimensional structures of the covariates.

\vspace{0.1in}
\noindent {\bf $\bullet$ Nonconvex Optimization of Deep Neural Networks}. Our theoretical guarantees hold for the global optimum of \eqref{eq:hat_mu}-\eqref{eq:pidr}. However, solving these optimization can be difficult in practice. Some recent empirical and theoretical results have shown that large neural networks help to ease the optimization without sacrificing  statistical efficiency \citep{zhang2016understanding,arora2019fine,allen2019learning}. This is also referred to as an overparameterization phenomenon. We will leave it for future investigation.

\bibliographystyle{ims} 
\bibliography{ref} 

\newpage
 \begin{appendix}

 \section*{\begin{center}Supplementary Materials for Doubly Robust Off-Policy Learning on\\ Low-Dimensional Manifolds by Deep Neural Networks\end{center}}



\section{Proof of Lemma \ref{lemma:stage1rate}}\label{appendix.stage1rate}
\begin{proof}[Proof of Lemma \ref{lemma:stage1rate}.]
We first derive the error bound $\|\hat{\mu}_{A_j}-\mu_{A_j}\|_{L^2}$ for any $j = 1, \dots, |\cA|$. Note that $\hat{\mu}_{A_j}\in \cF(L_1,p_1,K_1,\kappa_1,R_1)$ is the minimizer of (\ref{eq:hat_mu}). If we choose
\begin{align}
  L_1=O(\log n_{A_j}),\ p_1=O\left( n_{A_j}^{\frac{d}{2\alpha+d}}\right),\ K_1=O\left( n_{A_j}^{\frac{d}{2\alpha+d}}\log n_{A_j}\right),\ \kappa_1=\max\{ B,M,\sqrt{d},\tau^2\},\ R_1=M,
  \label{eq.proof.mu.net0}
\end{align}
then according to \citet[Theorem 1]{chen2019efficient}, for each $j$, we have
\begin{align}
  \EE\left[\|\hat{\mu}_{A_j}-\mu_{A_j}\|_{L^2}^2\right]\leq C_1(M^2+\sigma^2) n_{A_j}^{-\frac{2\alpha}{2\alpha+d}}\log^3 n_{A_j},
  \label{eq.proof.mu.bound0}
\end{align}
where $n_{A_j}=\sum_{i=1}^{n_1} \mathds{1}_{\{\ba_i=A_j\}}$ and $C_1$ is a constant only depending on $\log D, B,\tau$ and the surface area of $\cM$. In (\ref{eq.proof.mu.bound0}) the expectation is taken with respect to the randomness of samples.

Next, we derive a high probability lower bound of $n_{A_j}$ for all $j$'s in terms of $n_1$. By Assumption \ref{assum.causal}(ii), $\EE(n_{A_j}/n_1)\geq \eta$. By \citet[Lemma 29]{liao2019adaptive}, we have
\begin{align*}
  \PP\left(\left| \frac{n_{A_j}}{n_1}-\EE\left(\frac{n_{A_j}}{n_1}\right)\right|\geq\frac{1}{2}\EE\left(\frac{n_{A_j}}{n_1}\right)\right) \leq 2\exp\left( -\frac{3}{28}n_1\EE\left(\frac{n_{A_j}}{n_1}\right)\right).
\end{align*}
Thus $n_{A_j}\geq \eta n_1/2$ holds with probability at least $1-2\exp\left( -\frac{3}{28}\eta n_1\right)$. Denote the event $E_1\coloneqq\{n_{A_j}\geq \eta n_1/2\}$ and its complement by $E_1^{\complement}$. When $n_1$ (so as $n$) is large enough, we have
\begin{align*}
\EE\left[\|\hat{\mu}_{A_j}-\mu_{A_j}\|_{L^2}^2\right]&=\EE\left[\|\hat{\mu}_{A_j}-\mu_{A_j}\|_{L^2}^2| E_1\right]\PP\left( E_1\right)+ \EE\left[\|\hat{\mu}_{A_j}-\mu_{A_j}\|_{L^2}^2| E_1^\complement\right]\PP\left( E_1^{\complement}\right)\nonumber\\
  &\leq C_1(M^2+\sigma^2) (\eta n_1)^{-\frac{2\alpha}{2\alpha+d}}\log^3 (\eta n_1) +2C_1(M^2+\sigma^2)\exp\left( -\frac{3}{28}\eta n_1\right)\nonumber\\
  &\leq C_2(M^2+\sigma^2) (\eta n_1)^{-\frac{2\alpha}{2\alpha+d}}\log^3 (\eta n_1),
\end{align*}
where $C_2$ is a constant depending on $\log D, B, \tau$ and the surface area of $\cM$.
Substituting $n_{A_j}=\eta n_1$ into (\ref{eq.proof.mu.net0})  gives rise to
$\hat{\mu}_{A_j}\in \cF(L_1,p_1,K_1,\kappa_1,R_1)$ with $L_1,p_1,K_1,\kappa_1,R_1$ in (\ref{eq.mu.parameter}).


To estimate $\EE \left[\|\hat{e}_{A_j} -e_{A_j}\|_{L^2}^2\right]$, we use $\cH^{\alpha}_{|\cA|-1}(\cM)$ to denote the space of the $|\cA|-1$ dimensional vectors whose elements are in $\cH^{\alpha}(\cM)$. We denote $g^*=\left[g_{A_1},\dots,g_{A_{|\cA|-1}}\right]^{\top}$ with $g_{A_j}=\log\frac{e_{A_j}}{e_{A_{|\cA|}}}$. According to Assumption \ref{assum.regularity}, $g_{A_j}=\log e_{A_j}-\log e_{A_{|\cA|}}\in \cH^{\alpha}$, $\|g_{A_j}\|_{\cH^{\alpha}}\leq M$ and $g^*\in \cH^{\alpha}_{|\cA|-1}(\cM)$. Let $\hat{g}$ be the minimizer of (\ref{eq.hatg}).
From \citet[Corollary 4]{maurer2016vector} and the proof of \citet[Theorem 2]{farrell2018deep}, setting $\cG_{\rm NN}=\cF(L,p,K,\kappa,R)$ gives rise to
\begin{align*}
  \|\hat{g}-g^*\|_{L^2}^2\leq C_3M^2\left( \frac{|\cA|LK\log K}{n_1}\log n_1+ \frac{|\cA|(\log\log n_1+\gamma)}{n_1}+\sup_{g'\in \cH^{\alpha}_{|\cA|-1}(\cM)}\inf_{g\in \cG_{NN}} \|g^*-g'\|_{\infty}^2\right)
\end{align*}
with probability at least $1-\exp(-\gamma)$, where $C_3$ is an absolute constant. 

According to \citet[Theorem 2]{chen2019efficient}, for any $\varepsilon_2\in(0,1)$, there exists a neural network architecture $\cF(L,p,K,\kappa,R)$ with
\begin{align*}
L=O\left(\log \frac{1}{\varepsilon_2}\right), p=O\left(|\cA|\varepsilon_2^{-\frac{d}{\alpha}}\right), K=O\left(|\cA|\varepsilon_2^{-\frac{d}{\alpha}}\log \frac{1}{\varepsilon_2}\right),\kappa=\max\{ B,M,\sqrt{d},\tau^2\},R=M
\end{align*}
such that for any $g\in \cH^{\alpha}_{|\cA|-1}(\cM)$, there exists $\tilde{g}\in \cF(L,p,K,\kappa,R)$ with
$\|\tilde{g}-g\|_{\infty}\leq \varepsilon_2$, where $\|g\|_{\infty}=\sup_{\xb\in \cM} \max_j |g_j(\bx)|$
. Setting $\varepsilon_2=|\cA|^{\frac{2\alpha}{2\alpha+d}}n_1^{-\frac{2\alpha}{2\alpha+d}}, \gamma=|\cA|^{-\frac{2d}{2\alpha+d}}n_1^{\frac{d}{2\alpha+d}}$ gives rise to $\cG_{\rm NN}=\cF(L_2,p_2,K_2,\kappa_2,R_2)$ with 
\begin{align*}
  &L_2=O(\log (n_1/|\cA|)),\ p_2=O\left( |\cA|^{-\frac{2d}{2\alpha+d}}n_1^{\frac{d}{2\alpha+d}}\right),\ K_2=O\left( |\cA|^{-\frac{2d}{2\alpha+d}}n_1^{\frac{d}{2\alpha+d}}\log (n_1/|\cA|)\right), \nonumber\\
  &\kappa_2=\max\{B,M,\sqrt{d},\tau^2\},\ R_2=M
\end{align*}
which implies (\ref{eq.e.parameter}).
Then with probability no less than $1-\exp\left(-|\cA|^{-\frac{2d}{2\alpha+d}}n_1^{\frac{d}{2\alpha+d}}\right)$, we deduce
\begin{align*}
  \|\hat{g}-g^*\|_{L^2}^2\leq C_4 M^2|\cA|^{\frac{4\alpha}{2\alpha+d}}n_1^{-\frac{2\alpha}{2\alpha+d}}\log^3n_1
\end{align*}
with $C_4$ depending on $\log D,B,\tau$ and the surface area of $\cM$. Denote the event
\begin{align*}
E_2=\left\{\|\hat{g}-g^*\|_{L^2}^2\leq C_4 M^2|\cA|^{\frac{4\alpha}{2\alpha+d}}n_1^{-\frac{2\alpha}{2\alpha+d}}\log^3n_1\right\}
.\end{align*}
When $n_1$ (so as $n$) is large enough, we obtain
\begin{align*}
\EE[\|\hat{g}-g^*\|_{L^2}^2]&=\EE[\|\hat{g}-g^*\|_{L^2}^2|E_2]\PP(E_2)+ \EE[\|\hat{g}-g^*\|_{L^2}^2|E_2^{\complement}]\PP(E_2^{\complement})\nonumber\\
  &\leq  C_4 M^2|\cA|^{\frac{4\alpha}{2\alpha+d}}n_1^{-\frac{2\alpha}{2\alpha+d}}\log^3n_1+4(M^2+\sigma^2) \exp\left(-|\cA|^{-\frac{2d}{2\alpha+d}}n_1^{\frac{d}{2\alpha+d}}\right)\nonumber\\
  &\leq C_5  M^2|\cA|^{\frac{4\alpha}{2\alpha+d}}n_1^{-\frac{2\alpha}{2\alpha+d}}\log^3n_1
\end{align*}
with $C_5$ depending on $\log D,B,\tau$ and the surface area of $\cM$.

Define $r_j(g)=\frac{\exp([g]_j)}{1+\sum_{k=1}^{|\cA|-1}\exp([g]_k)}$ for $j=1,\dots,|\cA|-1$. Since $\|\nabla r_j\|_{\infty}\leq 1$ for any $j$, we have
\begin{align*}
\EE\left[\|\hat{e}_{A_j}-e_{A_j}\|_{L^2}^2\right]&=\EE\left[\|r_j(\hat{g})-r_j(g^*)\|_{L^2}^2\right] \nonumber\\
  &\leq \EE\left[\left( \|\nabla r_j\|_{\infty} \|\hat{g}-g^*\|_{L^2} \right)^2\right]
 \leq C_5M^2|\cA|^{\frac{4\alpha}{2\alpha+d}}n_1^{-\frac{2\alpha}{2\alpha+d}}\log^3n_1.
\end{align*}
Similarly, one can show $\EE\left[\|\hat{e}_{A_{|\cA|}}-e_{A_{|\cA|}}\|_{L^2}^2\right]\leq C_5M^2|\cA|^{\frac{4\alpha}{2\alpha+d}}n_1^{-\frac{2\alpha}{2\alpha+d}}\log^3n_1.$
\end{proof}
\section{Proof of Lemma \ref{lem.diserror}}\label{proof.diserror}
\begin{proof}[Proof of Lemma \ref{lem.diserror}.]
Recall that
\begin{align*}
&Q^{\rm (D)}(\pi)= \int_{\cM} \inner{[\mu_{I_1}(\bx),\dots,\mu_{I_V}(\bx)]^{\top}}{\pi(\bx)}  d\PP(\bx),\\
&Q(\pi)=\int_{\cM} \inner{[\mu_{A_1}(\bx),\dots,\mu_{A_V}(\bx)]^{\top} }{ \pi(\bx)} d\PP(\bx).
\end{align*}
Since $L_{\mu}$ is a uniform Lipschitz constant of $\mu(\bx,\cdot)$ for any $\bx\in\cM$, we derive
\begin{align*}
Q^{\rm (D)}(\pi)-Q(\pi) = \int_{\cM} \inner{[\mu_{I_1}(\bx)-\mu(\bx, A_1),\dots,\mu_{I_V}(\bx)-\mu(\bx, A_V)]^{\top} }{ \pi(\bx) } d\PP(\bx)\leq L_{\mu}/V.
\end{align*}
\end{proof}
\section{Proof of Lemma \ref{lem.E1}}\label{proof.E1}
\begin{proof}[Proof of Lemma \ref{lem.E1}.]

We first use McDiarmid’s inequality (Lemma \ref{lem.McDiarmid}) to show $\cD(\Pi)$ concentrates around $\EE[\cD(\Pi)]$ and then derive a bound of $\EE[\cD(\Pi)]$. To simplify the notation, we omit the domain $\Pi$ in $\cD$.

We denote $\{\mathring{\Gamma}_i'\}_{i=1}^n$ as the counterpart of $\{\mathring{\Gamma}\}_{i=1}^n$ when one sample $(\bx_k,\Gamma_k)$ is replaced by $(\bx_k,\Gamma_k')$ for any $k$ with $1\leq k\leq n$. $\mathring{\Delta}'(\pi_1,\pi_2)$ and $\cD'$ are defined analogously. We have
\begin{align}
|\cD-\cD'| &\leq \sup_{\pi_1,\pi_2\in \Pi} \mathring{\Delta}(\pi_1,\pi_2)-\mathring{\Delta}'(\pi_1,\pi_2) \leq \sup_{\pi_1, \pi_2\in \Pi} \frac{1}{n} \left\langle \mathring{\Gamma}_k - \mathring{\Gamma}'_k, \pi_1(\bx_k) - \pi_2(\bx_k) \right\rangle \nonumber \\
& \leq \frac{1}{n} \norm{\mathring{\Gamma}_k - \mathring{\Gamma}'_k}_\infty \norm{\pi_1(\bx_k) - \pi_2(\bx_k)}_1 \leq \frac{2}{n} \norm{\mathring{\Gamma}_k - \mathring{\Gamma}'_k}_\infty \leq \frac{4}{n}J,
\label{eq.ReplaceSample.1}
\end{align}
where $\|\cdot\|_{\infty}$ and $\|\cdot\|_1$ stand for the $\ell^{\infty}$ and $\ell^1$ norm for vectors. Applying Lemma \ref{lem.McDiarmid} with $f = \cD$, we have
\begin{align}
&\PP\left(\cD - \EE[\cD]\geq t\right)\leq \exp\left(-2nt^2/\left(16J^2\right)\right).
\end{align}
Setting $t=4J\sqrt{\frac{\log 1/\delta}{2n}}$ gives rise to
\begin{align}
  \cD\leq \EE[\cD]+4J\sqrt{\frac{\log 1/\delta}{2n}}
  \label{eq.proof1.E1}
\end{align}
with probability no less than $1-\delta$.

We next derive a bound of $\EE[\cD]$
by symmetrization:
\begin{align*}
\EE[\cD] & = \EE\left[\sup_{\pi_1, \pi_2\in \Pi} \mathring{\Delta}(\pi_1, \pi_2) - \EE\left[\mathring{\Delta}(\pi_1, \pi_2)\right] \right] \leq \EE \left[\sup_{\pi_1, \pi_2\in \Pi} \mathring{\Delta}(\pi_1, \pi_2) - \mathring{\Delta}_{\rm copy}(\pi_1, \pi_2) \right] \nonumber\\
&= \EE \EE_{\bm{\xi}} \left[\sup_{\pi_1, \pi_2\in \Pi} \bm{\xi} \odot \left(\mathring{\Delta}(\pi_1, \pi_2) - \mathring{\Delta}_{\rm copy}(\pi_1, \pi_2)\right) \right] = 2 \EE \EE_{\bm{\xi}} \left[\sup_{\pi_1, \pi_2\in \Pi} \bm{\xi} \odot \mathring{\Delta}(\pi_1, \pi_2) \right],
\end{align*}
where $\mathring{\Delta}_{\rm copy}$ denotes $\mathring{\Delta}$ using independent copies of samples and $\bm{\xi}=[\xi_{1},\dots,\xi_{n} ]^\top$ with  $\xi_i$'s being i.i.d. Rademacher variables which take value $1$ or $-1$ with the same probability. Here $\bm{\xi} \odot \mathring{\Delta}$ denotes the entry-wise product of $\xi$ and $\mathring{\Delta}$, i.e.,
\begin{align*}
\bm{\xi} \odot \mathring{\Delta}(\pi_1, \pi_2): = \frac{1}{n} \sum_{i=1}^n \xi_i \left\langle \mathring{\Gamma}_i , \pi_1(\bx_i) - \pi_2(\bx_i) \right\rangle.
\end{align*}

We next apply Lemma \ref{lem.McDiarmid} with $f=\EE_{\bm{\xi}} \left[\sup_{\pi_1, \pi_2 \in \Pi} \bm{\xi} \odot \mathring{\Delta}(\pi_1, \pi_2) \right]$. Again, we denote $\{\mathring{\Gamma}_i'\}_{i=1}^n$ as the counterpart of $\{\mathring{\Gamma}\}_{i=1}^n$ when one sample $(\bx_k,\mathring{\Gamma}_k)$ is replaced by $(\bx_k',\mathring{\Gamma}_k')$ for any $k$ with $1\leq k\leq n$. $\mathring{\Delta}'(\pi_1,\pi_2)$ is defined analogously. We get
\begin{align}
&\EE_{\bm{\xi}} \left[\sup_{\pi_1, \pi_2\in \Pi} \bm{\xi} \odot \mathring{\Delta}(\pi_1, \pi_2) \right] - \EE_{\bm{\xi}} \left[\sup_{\pi_1, \pi_2\in \Pi} \bm{\xi} \odot \mathring{\Delta}'(\pi_1, \pi_2) \right]\nonumber \\
\leq& \EE_{\bm{\xi}} \left[\sup_{\pi_1, \pi_2\in \Pi} \frac{1}{n} \xi_k \left\langle \mathring{\Gamma}_k - \mathring{\Gamma}'_k, \pi_1(\bx_k) - \pi_2(\bx_k) \right\rangle \right] \nonumber \\
\leq &\frac{1}{n} \norm{\mathring{\Gamma}_k - \mathring{\Gamma}'_k}_\infty \norm{\pi_1(\bx_k) - \pi_2(\bx_k)}_1 \leq \frac{4}{n} J.
\label{eq.ReplaceSample.1}
\end{align}
Applying Lemma \ref{lem.McDiarmid} with $f=\EE_{\bm{\xi}} \left[\sup_{\pi_1, \pi_2 \in \Pi} \bm{\xi} \odot \tilde{\Delta}(\pi_1, \pi_2) \right]$ gives rise to
\begin{align}
&\PP\Bigg(\EE \EE_{\bm{\xi}} \bigg[\sup_{\pi_1, \pi_2 \in \Pi} \bm{\xi} \odot \mathring{\Delta}(\pi_1, \pi_2) \bigg] - \EE_{\bm{\xi}} \bigg[\sup_{\pi_1, \pi_2 \in \Pi} \bm{\xi} \odot \mathring{\Delta}(\pi_1, \pi_2) \bigg] \geq t\Bigg) \leq \exp\left(-2nt^2/\left(16J^2\right)\right). \label{eq.proof.symbound1}
\end{align}

Setting $t=4J\sqrt{\frac{\log 1/\delta}{2n_2}}$ gives rise to
\begin{align}
\PP\Bigg(\EE \EE_{\bm{\xi}} \bigg[\sup_{\pi_1, \pi_2 \in \Pi} \bm{\xi} \odot \mathring{\Delta}(\pi_1, \pi_2) \bigg] &- \EE_{\bm{\xi}} \bigg[\sup_{\pi_1, \pi_2 \in \Pi} \bm{\xi} \odot \mathring{\Delta}(\pi_1, \pi_2) \bigg]\geq 4J\sqrt{\frac{\log 1/\delta}{2n}}\Bigg)\leq \delta. \label{eq.proof.symbound}
\end{align}

The following lemma provides an upper bound of $\EE_{\bm{\xi}} \left[\sup\limits_{\pi_1, \pi_2 \in \Pi} \bm{\xi} \odot \mathring{\Delta}(\pi_1, \pi_2) \right]$ (see a proof in Appendix \ref{appendix.RademacherBound}):
\begin{lemma}\label{lem.RademacherBound}
   Let $\bm{\xi}$ be a set of Rademacher random variable and $\mathring{\Delta}(\pi_1, \pi_2) $ defined in (\ref{eq.breveDelta}). Then the following bound holds
  \begin{align}
\EE_{\bm{\xi}} \left[\sup_{\pi_1, \pi_2 \in \Pi} \bm{\xi} \odot \mathring{\Delta}(\pi_1, \pi_2) \right] \leq \inf_{\lambda}~ 2 \lambda + \frac{48}{\sqrt{n}} \int_{\lambda}^{\max\limits_{\pi \in \Pi} \norm{\pi}_\Gamma} \sqrt{\log \cN(\theta, \Pi, \norm{\cdot}_{\Gamma})} d\theta,
\label{eq.sketch1.RademacherBound}
\end{align}
where $\cN(\theta, \Pi_{\rm NN}^{|\cA|},\norm{\cdot}_{\Gamma})$ is the $\theta$-covering number (see Definition \ref{def.covering}) of $\Pi$ with respect to the measure $\|\pi\|_{\Gamma}=\sqrt{\frac{1}{n}\sum_{i=1}^n \langle \mathring{\Gamma}_i,\pi(\bx_i)\rangle^2}$.
\end{lemma}

Substituting (\ref{eq.sketch1.RademacherBound}) into (\ref{eq.proof.symbound}) yields
\begin{align}
  \EE[\cD]&\leq\inf_{\lambda}~ 4 \lambda + \frac{96}{\sqrt{n}} \int_{\lambda}^{\max\limits_{\pi \in \Pi} \norm{\pi}_\Gamma} \sqrt{\log \cN(\theta, \Pi, \norm{\cdot}_{\Gamma})} d\theta  +8J\sqrt{\frac{\log 1/\delta}{2n}} \label{eq.proof1.boundEE1}
\end{align}
with probability no less than $1-\delta$.

Substituting (\ref{eq.proof1.boundEE1}) into (\ref{eq.proof1.E1}) give rise to
\begin{align}
  \cD\leq \inf_{\lambda}~ 4 \lambda + \frac{96}{\sqrt{n}} \int_{\lambda}^{\max\limits_{\pi \in \Pi} \norm{\pi}_\Gamma} \sqrt{\log \cN(\theta, \Pi, \norm{\cdot}_{\Gamma})} d\theta +12J\sqrt{\frac{\log 1/\delta}{2n}}
\end{align}
with probability no less than $1-2\delta$.
\end{proof}

\section{Proof of Lemma \ref{lem.covering}}\label{appendix.covering}
\begin{proof}[Proof of Lemma \ref{lem.covering}.]
We derive the bound of the covering number $\cN(\theta, \Pi_{\rm NN}^{|\cA|}, \norm{\cdot}_{\Gamma})$ using the covering number of the neural network class $\cN(\theta, \cF(L,p,K,\kappa,R), \|\cdot\|_{\infty})$. Let $\pi^{(1)}=\Softmax(\mu_{A_1}^{(1)},\dots,\mu_{A_{|\cA|}}^{(1)})$ and $\pi^{(2)}=\Softmax(\mu_{A_1}^{(2)},\dots,\mu_{A_{|\cA|}}^{(2)})$ be two policies in $\Pi_{\rm NN}^{|\cA|}(L_{\Pi},p_{\Pi},K_{\Pi},\kappa_{\Pi},R_{\Pi})$ such that for each $j$, $\|\mu_{A_j}^{(1)}-\mu_{A_j}^{(2)}\|_{\infty}\leq \theta$. By Assumption \ref{assum.causal} and (\ref{eq.M}), $\|\tilde{\Gamma}_i\|_{1}\leq 2M/\eta+|\cA|M$ for any $n_1\leq i\leq n$. Therefore we have
\begin{align*}
&\norm{\pi^{(1)} - \pi^{(2)}}_{\Gamma}^2 = \frac{1}{n_2}\sum_{i=n_1+1}^n \left\langle \tilde\Gamma_i, (\pi^{(1)} - \pi^{(2)})(\bx_i) \right\rangle^2 \nonumber\\
\leq& \frac{1}{n_2} \sum_{i=n_1+1}^n \norm{\tilde\Gamma_i}_{1}^2 \norm{(\pi^{(1)} - \pi^{(2)})(\bx_i)}_{\infty}^2  \leq \left(|\cA|M + 2M/\eta)\right)^2 \theta^2.
\end{align*}
Thus we obtain
\begin{align}
\cN(\theta, \Pi_{\rm NN}^{|\cA|}, \norm{\cdot}_{\Gamma}) \leq \cN\left(\theta/(|\cA|M + 2M/\eta), \Pi_{\rm NN}^{|\cA|}, \norm{\cdot}_\infty\right).
\label{eq.cover.1}
\end{align}
Since for every $\pi\in\Pi_{\rm NN}^{|\cA|}$ with $L_{\Pi}=L,p_{\Pi}=|\cA|p,K_{\Pi}=|\cA|K,\kappa_{\Pi}=\kappa,R_{\Pi}=R$, it contains $|\cA|$ parallel ReLU networks in $\cF(L,p,K,\kappa,R)$ with an additional softmax layer, we have $\cN(\theta, \Pi_{\rm NN}^{|\cA|}, \norm{\cdot}_\infty)\leq\cN(\theta, \cF(L,p,K,\kappa,R), \norm{\cdot}_\infty)^{|\cA|}$. From \citet[Proof of Theorem 3.1]{chen2019efficient}, we have
\begin{align*}
  \cN(\theta, \cF(L,p,K,\kappa,R), \norm{\cdot}_\infty)\leq \left(\frac{2L^2 (pR+2) \kappa^L p^{L+1}}{\theta}\right)^K.
\end{align*}
We get
\begin{align}
\cN(\theta, \Pi_{\rm NN}, \|\cdot\|_{\infty}) \leq \left(\frac{2L^2 (pR+2) \kappa^L p^{L+1}}{\theta}\right)^{|\cA|K}.
\label{eq.cover.2}
\end{align}
Combining (\ref{eq.cover.1}) and (\ref{eq.cover.2}) proves Lemma \ref{lem.covering}.
\end{proof}

\section{Proof of Lemma \ref{lem.integral}}\label{appendix.integral}
\begin{proof}[Proof of Lemma \ref{lem.integral}.]
For any $\pi\in \Pi_{\rm NN}^{|\cA|}$,
\begin{align*}
\|\pi\|_{\Gamma}^2&\leq \frac{1}{n_2}\sum_{i=n_1+1}^n \langle \tilde{\Gamma}_i, \pi(\bx_i)\rangle^2 \leq \frac{1}{n_2}\sum_{i=n_1+1}^n \|\tilde{\Gamma}_i\|_{1}^2 \|\pi(\bx_i)\|_{\infty}^2 \leq \frac{1}{n_2}\sum_{i=n_1+1}^n \|\tilde{\Gamma}_i\|_{1}^2.
\end{align*}
By Assumption \ref{assum.causal} and (\ref{eq.M}), $\|\tilde{\Gamma}_i\|_{1}\leq 2M/\eta+|\cA|M$ for any $n_1\leq i \leq n$. Therefore we obtain
\begin{align*}
  \|\pi\|_{\Gamma}^2\leq (2M/\eta+|\cA|M)^2.
\end{align*}
\end{proof}

\section{Proof of Lemma \ref{lem.RademacherBound}}\label{appendix.RademacherBound}
We first define the covering number of a set.
\begin{definition}\label{def.covering}
  Let $\cF$ be a set equipped with metric $\rho$. For any $\delta>0$, a $\delta$-covering of $\cF$ is a set $\{f_1,\dots,f_N\}\subset\cF$ such that for any $f\in \cF$, there exists $f_k$ for $1\leq k\leq N$ with $\rho(f_k,f)\leq \theta$. The $\delta$-covering number of $\cF$ is defined as
  \begin{align}
    \cN(\delta,\cF,\rho)=&\inf \{N: \mbox{ there exists } \{f_1,\dots,f_N\} \mbox{ which is a $\theta$-covering of $\cF$}\}.
  \end{align}
\end{definition}
\begin{proof}[Proof of Lemma \ref{lem.RademacherBound}.]
To bound $\EE_{\bm{\xi}} \left[\sup\limits_{\pi_1, \pi_2 \in \Pi} \bm{\xi} \odot \mathring{\Delta}(\pi_1, \pi_2) \right]$  with respect to the measure $\|\cdot\|_{\Gamma}$, we construct a series of $I$ coverings of $\Pi$ with resolutions $\{\delta_i\}_{i=1}^I$ satisfying $\delta_{i+1} = \frac{1}{2} \delta_i$. The elements in the $(i)$-th covering are denoted as $\{\pi_i^{(i)}\}_{i=1}^{N^{(i)}}$, where the $N^{{(i)}}$'s are to be determined later. Thus for any $\pi \in \Pi$, there exists $\pi^{(i)}$ in the ${(i)}$-th covering such that
\begin{align*}
\sqrt{\frac{1}{n}\sum_{i=1}^n \left\langle \mathring \Gamma_i, \pi(\bx_i) - \pi^{(i)}(\bx_i) \right\rangle^2} \leq \delta_i.
\end{align*}

Let $\pi_1^{(i)}$ denote the closest element of $\pi_1$ in the ${(i)}$-th covering, and $\pi_2^{(i)}$ is defined analogously. We now expand $\pi_1 - \pi_2$ using a telescoping sum:
\begin{align}
\pi_1 - \pi_2 = \left(\pi_1 - \pi_1^{(I)} + \sum_{i=1}^{I-1} \pi_1^{(i+1)} - \pi_1^{(i)} + \pi_1^{(1)}\right) - \left(\pi_2 - \pi_2^{(I)} + \sum_{i=1}^{I-1} \pi_2^{(i+1)} - \pi_2^{(i)} + \pi_2^{(1)}\right).
\label{eq.proof.tele}
\end{align}
Substituting (\ref{eq.proof.tele}) into $\EE_{\bm{\xi}} \left[\sup\limits_{\pi_1, \pi_2 \in \Pi} \bm{\xi} \otimes \mathring{\Delta}(\pi_1, \pi_2) \right]$, due to the bi-linearity of $\mathring{\Delta}$, we have
\begin{align}
&\EE_{\bm{\xi}} \left[\sup_{\pi_1, \pi_2 \in \Pi} \bm{\xi} \otimes \mathring{\Delta}(\pi_1, \pi_2) \right] \nonumber\\
\leq& \EE_{\bm{\xi}} \left[\sup_{\pi_1 \in \Pi} \frac{1}{n} \sum_{i=1}^n \xi_i \left\langle \mathring{\Gamma}_i, \left(\pi_1 - \pi_1^{(I)} + \sum_{i=1}^{I-1} \pi_1^{(i+1)} - \pi_1^{(i)} + \pi_1^{(1)}\right)(\bx_i) \right\rangle \right] \nonumber \\
&  + \EE_{\bm{\xi}} \left[ \sup_{\pi_2 \in \Pi} \frac{1}{n} \sum_{i=1}^n \xi_i \left\langle \tilde{\Gamma}_i, \left(\pi_2 - \pi_2^{(I)} + \sum_{i=1}^{I-1} \pi_2^{(i+1)} - \pi_2^{(i)} + \pi_2^{(1)}\right)(\bx_i) \right\rangle \right]. \label{eq.proof.Exi}
\end{align}
By the construction of the coverings, we immediately have
\begin{align}
&\EE_{\bm{\xi}} \left[\sup_{\pi_1 \in \Pi} \frac{1}{n} \sum_{i=1}^n \xi_i \left\langle \mathring{\Gamma}_i, \left(\pi_1 - \pi_1^{(I)}\right) (\bx_i) \right\rangle \right] \nonumber \\
\leq& \EE_{\bm{\xi}} \left[\sup_{\pi_1 \in \Pi} \frac{1}{n} \norm{\bm{\xi}}_2 \sqrt{\sum_{i=1}^n \left\langle \mathring{\Gamma}_i, \left(\pi_1 - \pi_1^{(I)}\right) (\bx_i) \right\rangle^2} \right] \leq \delta_I. \label{eq.proof.coverI}
\end{align}
We can also check
\begin{align}
&\sqrt{\sum_{i=1}^n \left\langle \mathring\Gamma_i, \pi^{(i+1)}(\bx_i) - \pi^{(i)}(\bx_i) \right\rangle^2} = \sqrt{\sum_{i=1}^n \left\langle \mathring\Gamma_i, \pi^{(i+1)}(\bx_i) - \pi(\bx_i) + \pi(\bx_i) - \pi^{(i)}(\bx_i) \right\rangle^2} \nonumber \\
\leq& \sqrt{2\sum_{i=1}^n \left\langle \mathring{\Gamma}_i, \pi^{(i+1)}(\bx_i) - \pi(\bx_i) \right\rangle^2 + 2\sum_{i=1}^n \left\langle \mathring{\Gamma}_i, \pi(\bx_i) - \pi^{(i)}(\bx_i) \right\rangle^2} \nonumber\\
\leq &\sqrt{2n(\delta_{i+1}^2 + \delta_i^2)}\leq \sqrt{2n}(\delta_{i+1}+\delta_i). \label{eq.proof.coveri}
\end{align}
Using Lemma \ref{lem.Massart}, we have
\begin{align}
&\EE_{\bm{\xi}} \left[\sup_{\pi_1 \in \Pi} \frac{1}{n} \sum_{i=1}^n \xi_i \left\langle \mathring\Gamma_i, \pi_1^{(i+1)}(\bx_i) - \pi_1^{(i)}(\bx_i) \right\rangle \right] \nonumber \\
\leq& \frac{2(\delta_{i+1} + \delta_i)\sqrt{\log (\cN(\delta_{i}, \Pi, \norm{\cdot}_{\Gamma}) \cN(\delta_{i+1}, \Pi, \norm{\cdot}_{\Gamma}))}}{\sqrt{n}}  \leq \frac{4(\delta_{i+1} + \delta_i) \sqrt{\log \cN(\delta_{i+1}, \Pi, \norm{\cdot}_{\Gamma})}}{\sqrt{n}}\label{eq.proof.coveri1},
\end{align}
where the metric in the covering is $\norm{\pi}_{\Gamma} = \sqrt{\frac{1}{n}\sum_{i=1}^n \left\langle \mathring \Gamma_i, \pi(\bx_i) \right\rangle^2}$. Substituting (\ref{eq.proof.coverI}), (\ref{eq.proof.coveri1}) into (\ref{eq.proof.Exi}), and invoking the identity $\delta_{i+1} + \delta_i = 6(\delta_{i+1} - \delta_{i+2})$ yield
\begin{align*}
&\EE_{\bm{\xi}} \left[\sup_{\pi_1, \pi_2 \in \Pi} \bm{\xi} \otimes \mathring{\Delta}(\pi_1, \pi_2) \right] \leq 2\delta_I + \sum_{i=1}^{I-1} \frac{8(\delta_{i+1} + \delta_i) \sqrt{\log \cN(\delta_{i+1}, \Pi, \norm{\cdot}_{\Gamma})}}{\sqrt{n}} \nonumber \\
\leq& 2 \delta_I + \frac{48(\delta_{i+1} - \delta_{i+2}) \sqrt{\log \cN(\delta_{i+1}, \Pi, \norm{\cdot}_{\Gamma})}}{\sqrt{n}} \leq 2 \delta_I + \frac{48}{\sqrt{n}} \int_{\delta_I}^{\delta_1} \sqrt{\log \cN(\tau, \Pi, \norm{\cdot}_{\Gamma})} d\tau.
\end{align*}
Choosing $\delta_1 = \max_{\pi \in \Pi} \norm{\pi}_\Gamma$ so that the first covering only consists of one element, we derive
\begin{align*}
\EE_{\bm{\xi}} \left[\sup_{\pi_1, \pi_2 \in \Pi} \bm{\xi} \odot \mathring{\Delta}(\pi_1, \pi_2) \right] \leq \inf_{\lambda}~ 2 \lambda + \frac{48}{\sqrt{n}} \int_{\lambda}^{\max\limits_{\pi \in \Pi} \norm{\pi}_\Gamma} \sqrt{\log \cN(\theta, \Pi, \norm{\cdot}_{\Gamma})} d\theta.
\end{align*}
\end{proof}

\section{Some Useful Lemmas}\label{appendix:helper}

\begin{lemma}\label{lem.holder.integral}
  Let $f(\bx,A)$ be any function defined on $\cM\times[0,1]$. Assume there exists $M>0$ such that
  \begin{align}
    \sup_{A\in [0,1]} \|f(\cdot,A)\|_{\cH^{\alpha}(\cM)}\leq M \mbox{ and } \sup_{\bx\in\cM}|f(\bx,A)-f(\bx,\tilde{A})|\leq M|A-\tilde{A}|,\ \forall A,\tilde{A}\in[0,1].
  \end{align}
  Then $F^{(I)}=\int_I f(\bx,A)dA\in \cH^{\alpha}(\cM)$ satisfies $\|F^{(I)}\|_{\cH^{\alpha}(\cM)}\leq M|I|$ for any interval $I\subset [0,1]$ where $|I|$ is the length of $I$.
\end{lemma}
\begin{proof}[Proof of Lemma \ref{lem.holder.integral}.]
To show $F^{(I)}\in \cH^{\alpha}(\cM)$, it is sufficient to show $\|F^{(I)}\|_{\cH^{\alpha}(U)}<\infty$ for any chart $(U,\phi)$ of $\cM$. For simplicity, we denote
$$F^{(I)}_{\phi}(\bz)\coloneqq F^{(I)}\circ\phi^{-1}(\bz),\ f_{\phi}(\bz,A)\coloneqq f(\phi^{-1}(\bz),A)$$
for $\bz\in \phi(U)$. Then $F^{(I)}_{\phi}(\bz)=\int_I f_{\phi}(\bz,A)dA$. 

We first consider $0<\alpha\leq1$. In this case, we have
\begin{align}
\|F^{(I)}_{\phi}\|_{\cH^{\alpha}(\phi(U))}&=
\sup\limits_{ \bz\neq \by \in \phi(U)}\frac{|F^{(I)}_{\phi}(\bz)- F^{(I)}_{\phi}(\by)|}{\|\bz-\by\|_2^{\alpha}}\nonumber\\
&\leq
\sup\limits_{ \bz\neq \by \in \phi(U)}\int_I\frac{ |f_{\phi}(\bz,A)- f_{\phi}(\by,A)|}{\|\bz-\by\|_2^{\alpha}}dA\leq M|I|< \infty,
\end{align}
which implies $F^{(I)}\in \cH^{\alpha}(\cM)$.

Next we consider $\alpha>1$.
We first show that $\partial^{\bs} F^{(I)}_{\phi}(\bz)=\int_I \partial^{\tilde{\bs}} f_{\phi}(\bz,A)dA$ for any $|\bs|\leq \lceil \alpha-1 \rceil$ where $\tilde{\bs}=[\bs^{\top},0]^{\top}$.
Let $\{h_n\}_{n=1}^{\infty}$ be any sequence converging to 0. When $|\bs|=1$, by definition, we have
\begin{align*}
  \partial^{\bs} F^{(I)}_{\phi}(\bz)&=\lim_{n\rightarrow\infty} \frac{F^{(I)}_{\phi}(\bz+h_n\bs)-F^{(I)}_{\phi}(\bz)}{h_n} =\lim_{n\rightarrow\infty} \int_I \frac{f_{\phi}(\bz+h_n\bs,A)-f_{\phi}(\bz,A)}{h_n}dA.
\end{align*}
Since $\|f_{\phi}(\bz,A)\|_{\cH^{\alpha}(\phi(U))}\leq M$ for any fixed $A\in[0,1]$, by the mean value theorem,
\begin{align*}
\left|\frac{f_{\phi}(\bz+h_n\bs,A)-f_{\phi}(\bz,A)}{h_n}\right|\leq \max_{\tilde{\bz}\in\phi(U)} |\partial^{\tilde{\bs}}f_{\phi}(\tilde{\bz},A)|\leq M.
\end{align*}
Since
\begin{align}
\lim_{n\rightarrow\infty} \frac{f_{\phi}(\bz+h_n\bs,A)-f_{\phi}(\bz,A)}{h_n}=\partial^{\tilde{\bs}}f_{\phi}(\bz,A)
\label{eq.derivativeDef}
\end{align}
and by the dominated convergence theorem, we obtain
\begin{align*}
  \partial^{\bs} F^{(I)}_{\phi}(\bz)&=\lim_{n\rightarrow\infty} \frac{F^{(I)}_{\phi}(\bz+h_n\bs)-F^{(I)}_{\phi}(\bz)}{h_n}\nonumber\\
  &= \int_I \lim_{n\rightarrow\infty}\frac{f_{\phi}(\bz+h_n\bs,A)-f_{\phi}(\bz,A)}{h_n}dA= \int_I \partial^{\tilde{\bs}}f_{\phi}(\bz,A)dA.
\end{align*}
Similarly, for any $|\bs|\leq \lceil \alpha-1 \rceil$, $\partial^{\tilde{\bs}}f(\bx,A)$ can be expressed in the form similar to (\ref{eq.derivativeDef}) using the Taylor series. Following the same procedure, one can show
\begin{align*}
  \partial^{\bs} F^{(I)}_{\phi}(\bz)= \int_I \partial^{\tilde{\bs}}f_{\phi}(\bz,A)dA
\end{align*}
for any $|\bs|\leq \lceil \alpha-1 \rceil$. Therefore we have
\begin{align}
  \max_{|\bs|\leq \lceil \alpha-1 \rceil}\sup_{\bz\in \phi(U)} |\partial^{\bs} F^{(I)}_{\phi}|\leq M|I|<\infty,
  \label{eq.holderint.1}
\end{align}
where $|I|$ represents the length of $I$.

On the other hand,
\begin{align}
&\max\limits_{|\bs|=\lceil \alpha-1\rceil}
\sup\limits_{ \bz\neq \by \in \phi(U)}\frac{|\partial^{\bs}F^{(I)}_{\phi}(\bz)- \partial^{\bs}F^{(I)}_{\phi}(\by)|}{\|\bz-\by\|_2^{\alpha - \lceil \alpha-1\rceil}}\nonumber\\
\leq& \max\limits_{|\bs|=\lceil \alpha-1\rceil}
\sup\limits_{ \bz\neq \by \in \phi(U)}\int_I\frac{ |\partial^{\tilde{\bs}}f_{\phi}(\bz,A)- \partial^{\tilde{\bs}}f_{\phi}(\by,A)|}{\|\bz-\by\|_2^{\alpha - \lceil \alpha-1\rceil}}dA\leq M|I|< \infty.
\label{eq.holderint.2}
\end{align}
Combining (\ref{eq.holderint.1}) and (\ref{eq.holderint.2}) gives $\|F^{(I)}\|_{\cH^{\alpha}(U)}<\infty$ for any chart $(U,\phi)$ which implies $F^{(I)}\in \cH^{\alpha}(\cM)$.

\end{proof}

\begin{lemma}\label{lem.holder.frac}
  Assume Assumption \ref{assum.cM}. Let $f,g\in \cH^{\alpha}(\cM)$ with $\inf_{\bx\in\cM}g(\bx)\geq \eta>0$. Let $M>0$ be a constant such that  $\|f\|_{\cH^{\alpha}(\cM)}\leq M$ and $\|g\|_{\cH^{\alpha}(\cM)}\leq M$. Then we have $f/g\in \cH^{\alpha}(\cM)$ with $\|f/g\|_{\cH^{\alpha}(\cM)}\leq 2^{\frac{5+\lceil \alpha-1\rceil}{2}\lceil \alpha-1\rceil}(M/\eta)^{2^{\lceil \alpha\rceil}}(2B+1)$.
\end{lemma}
\begin{proof}[Proof of Lemma \ref{lem.holder.frac}.]
To prove Lemma \ref{lem.holder.frac}, it is sufficiently to show $\|f/g\|_{\cH^{\alpha}(U)}<\infty$ for any chart $(U,\phi)$ of $\cM$. For simplicity, denote
$$f_{\phi}(\bz)\coloneqq f\circ \phi^{-1}(\bz),\ g_{\phi}(\bz)\coloneqq g\circ \phi^{-1}(\bz)$$
for any $\bz\in\phi(U)$. 

We first consider $0<\alpha\leq1$. In this case,
\begin{align}
  &\|f/g\|_{\cH^{\alpha}(U)}=\sup_{\bz\neq \by\in \phi(U)} \frac{\left|\left(f_{\phi}(\bz)/g_{\phi}(\bz)\right)- \left(f_{\phi}(\by)/g_{\phi}(\by)\right)\right|}{\|\bz-\by\|_2^{\alpha}}\nonumber\\
\leq& \frac{|f_{\phi}(\bz)g_{\phi}(\by)-f_{\phi}(\bz)g_{\phi}(\bz) +f_{\phi}(\bz)g_{\phi}(\bz) -f_{\phi}(\by)g_{\phi}(\bz)|}{g_{\phi}(\by)g_{\phi}(\bz)\|\bz-\by\|_2^{\alpha}}\nonumber\\
\leq& \frac{1}{\eta^2}\left( M\frac{|g_{\phi}(\by)-g_{\phi}(\bz)|}{\|\bz-\by\|_2^{\alpha}} +M\frac{|f_{\phi}(\bz)-f_{\phi}(\by)|}{\|\bz-\by\|_2^{\alpha}}\right)\leq 2M^2/\eta^2<\infty
\end{align}
which implies $f/g\in\cH^{\alpha}(\cM)$.

We next consider the case $\alpha>1$. We first show $|\partial^{\bs} (f_{\phi}(\bz)/g_{\phi}(\bz))|<\infty$ for $|\bs|\leq \lceil \alpha-1\rceil$. When $|\bs|=1$, we have
\begin{align*}
  \left|\partial^{\bs}\left(\frac{f_{\phi}(\bz)}{g_{\phi}(\bz)}\right)\right|= \left|\frac{\partial^{\bs} f_{\phi}(\bz)g_{\phi}(\bz) -f_{\phi}(\bz)\partial^{\bs}g_{\phi}(\bz)}{ g_{\phi}^2(\bz)}\right|\leq 2M^2/\eta^2.
\end{align*}
For any $|\bs|\leq \lceil \alpha-1\rceil$,
following this process, one can show
\begin{align}
  \left|\partial^{\bs}\left(\frac{f_{\phi}(\bz)}{g_{\phi}(\bz)}\right)\right| =\frac{\sum_{i=1}^{2^{\frac{1+|\bs|}{2}|\bs|}}G_i}{g_{\phi}^{2^{|\bs|}}}
  \leq 2^{\frac{1+|\bs|}{2}|\bs|}(M/\eta)^{2^{|\bs|}}<\infty,
  \label{eq.holder.frac.1}
\end{align}
where each $G_i$ is the product of $2^{|\bs|}$ terms from $\{\partial^{\bar{\bs}}f_{\phi}, \partial^{\bar{\bs}}f_{\phi}| |\bar{\bs}|\leq |\bs|\}$.

On the other hand, note that for any $\bs$ with $|\bs|=1$, we have
\begin{align*}
&\left|\partial^{\bs}\left(\frac{f_{\phi}(\bz)}{g_{\phi}(\bz)}\right)- \partial^{\bs}\left(\frac{f_{\phi}(\by)}{g_{\phi}(\by)}\right)\right|\\
=& \left|\frac{\partial^{\bs}f_{\phi}(\bz)g_{\phi}(\bz) -f_{\phi}(\bz)\partial^{\bs}g_{\phi}(\bz)}{g^2_{\phi}(\bz)}- \frac{\partial^{\bs}f_{\phi}(\by)g_{\phi}(\by) -f_{\phi}(\by)\partial^{\bs}g_{\phi}(\by)}{g^2_{\phi}(\by)}\right|\\
=&\left|\frac{g^2_{\phi}(\by)\partial^{\bs}f_{\phi}(\bz)g_{\phi}(\bz) -g^2_{\phi}(\bz)\partial^{\bs}f_{\phi}(\by)g_{\phi}(\by) -\left(g^2_{\phi}(\by)f_{\phi}(\bz)\partial^{\bs}g_{\phi}(\bz) -g^2_{\phi}(\bz)f(\by)\partial^{\bs}g_{\phi}(\by)\right)}{g^2_{\phi}(\bz)g^2_{\phi}(\by)}\right|\\
\leq& \frac{1}{\eta^4}\Big|g_{\phi}(\by)g_{\phi}(\bz)\left[\partial^{\bs}f_{\phi}(\bz)g_{\phi}(\by) -\partial^{\bs}f_{\phi}(\bz)g_{\phi}(\bz) + \partial^{\bs}f_{\phi}(\bz)g_{\phi}(\bz)-\partial^{\bs}f_{\phi}(\by)g_{\phi}(\bz)\right]\\
&+g_{\phi}^2(\by)f_{\phi}(\bz)\partial^{\bs}g_{\phi}(\bz)- g_{\phi}^2(\by)f_{\phi}(\bz)\partial^{\bs}g_{\phi}(\by) +g_{\phi}^2(\by)f_{\phi}(\bz)\partial^{\bs}g_{\phi}(\by) -g_{\phi}^2(\by)f_{\phi}(\by)\partial^{\bs}g_{\phi}(\by) \\ &+g_{\phi}^2(\by)f_{\phi}(\by)\partial^{\bs}g_{\phi}(\by) -g_{\phi}(\by)g_{\phi}(\bz)f_{\phi}(\by)\partial^{\bs}g_{\phi}(\by) +g_{\phi}(\by)g_{\phi}(\bz)f_{\phi}(\by)\partial^{\bs}g_{\phi}(\by)- g^2_{\phi}(\bz)f_{\phi}(\by)\partial^{\bs}g_{\phi}(\by)\Big|\\
\leq & \frac{ M^3}{\eta^4}\left[ 3|g_{\phi}(\bz)-g_{\phi}(\by)| +|f_{\phi}(\bz)-f_{\phi}(\by)| +|\partial^{\bs}g_{\phi}(\bz) -\partial^{\bs}g_{\phi}(\by)| +|\partial^{\bs}f_{\phi}(\bz) -\partial^{\bs}f_{\phi}(\by)|\right]\\
\leq & \frac{M^3}{\eta^4}\left[4M\|\bz-\by\|+ |\partial^{\bs}g_{\phi}(\bz) -\partial^{\bs}g_{\phi}(\by)| +|\partial^{\bs}f_{\phi}(\bz) -\partial^{\bs}f_{\phi}(\by)|\right].
\end{align*}
Analogously, for any $|\bs|\leq \lceil \alpha-1\rceil$, one can show
\begin{align*}
  &\left|\partial^{\bs}\left(\frac{f_{\phi}(\bz)}{g_{\phi}(\bz)}\right)- \partial^{\bs}\left(\frac{f_{\phi}(\by)}{g_{\phi}(\by)}\right)\right|\\
\leq& (M/\eta)^{2^{|\bs|+1}-1}\left(C_1M\|\bz-\by\|+C_2|\partial^{\bs}g_{\phi}(\bz) -\partial^{\bs}g_{\phi}(\by)| +C_3|\partial^{\bs}f_{\phi}(\bz) -\partial^{\bs}f_{\phi}(\by)|\right)
\end{align*}
for some absolute constants $C_1,C_2,C_3$ such that $C_1+C_2+C_3=2^{\frac{5+|\bs|}{2}|\bs|}$. Thus we deduce
\begin{align}
  &\max_{|\bs|\leq \lceil \alpha-1\rceil} \sup_{\bz\neq \by\in \phi(U)} \frac{\left|\partial^{\bs}\left(f_{\phi}(\bz)/g_{\phi}(\bz)\right)- \partial^{\bs}\left(f_{\phi}(\by)/g_{\phi}(\by)\right)\right|}{\|\bz-\by\|_2^{\alpha-\lceil \alpha-1\rceil}}\nonumber\\
\leq& (M/\eta)^{2^{\lceil \alpha-1\rceil+1}-1}(2C_1MB+(C_2+C_3)M)
<2^{\frac{5+\lceil \alpha-1\rceil}{2}\lceil \alpha-1\rceil}(M/\eta)^{2^{\lceil \alpha\rceil}}(2B+1)< \infty.
  \label{eq.holder.frac.2}
\end{align}
Combining (\ref{eq.holder.frac.1}) and (\ref{eq.holder.frac.2}) yields
\begin{align}
\|f/g\|_{\cH^{\alpha}(U)}<2^{\frac{5+\lceil \alpha-1\rceil}{2}\lceil \alpha-1\rceil}(M/\eta)^{2^{\lceil \alpha\rceil}}(2B+1)<\infty
\end{align}
for any chart $(U,\phi)$ of $\cM$ which implies $f/g\in\cH^{\alpha}(\cM)$.

\end{proof}

\begin{lemma}\label{lem.holder.log}
  Assume Assumption \ref{assum.cM}. Let $f\in \cH^{\alpha}(\cM)$ with $ \alpha>1$ and $f(\bx)\geq \eta>0$. Let $M>0$ be a constant such that $\|f\|_{\cH^{\alpha}(\cM)}\leq M$. Then we have $\log f\in \cH^{\alpha}(\cM)$ with $\|\log f\|_{\cH^{\alpha}(\cM)}\leq 2^{\frac{5+\lceil \alpha-2\rceil}{2}\lceil \alpha-2\rceil}(M/\eta)^{2^{\lceil \alpha-1\rceil}}(2B+1)$.
\end{lemma}
\begin{proof}[Proof of Lemma \ref{lem.holder.log}.]
It is sufficiently to show $\|\log f\|_{\cH^{\alpha}(U)}<\infty$ for any chart $(U,\phi)$ of $\cM$. For simplicity, denote
$f_{\phi}(\bz)\coloneqq f\circ \phi^{-1}(\bz)$
for any $\bz\in\phi(U)$. We further denote $M>0$ such that $\|f\|_{\cH^{\alpha}(U)}\leq M$.

For any $|\bs|=1$,
$
\partial^{\bs} \log f_{\phi}(\bz)=\frac{\partial^{\bs} f_{\phi}(\bz)}{f_{\phi}(\bz)}.
$
Note that $\partial^{\bs} f_{\phi}(\bz)\in \cH^{\alpha-1}(\phi(U))$. According to Lemma \ref{lem.holder.frac},
\begin{align}
  \frac{\partial^{\bs} f_{\phi}}{f_{\phi}}\in \cH^{\alpha-1}(\phi(U))
  \label{eq.holder.log.1}
\end{align}
for any $|\bs|=1$. Combining (\ref{eq.holder.log.1}) and
\begin{align*}
  \max_{|\bs|=1,\bz\in \phi(U)}\left|\frac{\partial^{\bs} f_{\phi}(\bz)}{f_{\phi}(\bz)}\right|\leq M/\eta,
\end{align*}
we have
$\|\log f\|_{\cH^{\alpha}(U)}<\infty$ with $\|\log f\|_{\cH^{\alpha}(U)}\leq 2^{\frac{5+\lceil \alpha-2\rceil}{2}\lceil \alpha-2\rceil}(M/\eta)^{2^{\lceil \alpha-1\rceil}}(2B+1)$ which proves Lemma \ref{lem.holder.log}.
\end{proof}

The following two lemmas are extensively used in the previous proofs.
\begin{lemma}[McDiarmid’s inequality \citep{mcdiarmid1989method}] \label{lem.McDiarmid}
  Let $\bx_1,\dots,\bx_n\in \cX$ be independent random variables and $f:\cX^n\rightarrow \RR$ be a map. If for any $i$ and $\bx_1,\dots,\bx_n,\bx'_i\in \cX$, the following holds
  \begin{align*}
    |f(\bx_1,\dots,\bx_{i-1},\bx_i,\bx_{i+1},\dots,\bx_n)-f(\bx_1,\dots,\bx_{i-1},\bx'_i,\bx_{i+1},\dots,\bx_n)|\leq c_i,
  \end{align*}
  then for any $t>0$,
  \begin{align*}
    \PP(|f(\bx_1,\dots,\bx_n)-\EE[f]\geq t|)\leq \exp\left( \frac{-2t^2}{\sum_{i=1}^n c_i^2}\right).
  \end{align*}
\end{lemma}

\begin{lemma}[Massart’s lemma \citep{massart2000some}] \label{lem.Massart}
  Let $\cX$ be some finite set in $\RR^m$ and $\varepsilon_1,\dots,\varepsilon_m$ be independent Rademacher random variables. Then
  \begin{align*}
    \EE\left[ \sup_{\bx\in\cX} \frac{1}{m}\sum_{i=1}^m \varepsilon_i x_i\right] \leq \sup_{\bx\in\cX}\|\bx\| \frac{\sqrt{2\log |\cX|}}{m}.
  \end{align*}
\end{lemma}

 \end{appendix}
\end{document}